\documentclass[journal,10pt]{IEEEtran} 

\ifCLASSOPTIONcompsoc
   \usepackage[nocompress]{cite}
\else
   \usepackage{cite}
\fi

\usepackage{mdframed}
\usepackage{bbding}
\usepackage{times,amsmath,amssymb,amsfonts,algorithm,epsfig,subfigure,cite}
\usepackage{color} 
\usepackage[algo2e]{algorithm2e}
\usepackage{multirow}
\usepackage{makecell}
\usepackage{color}
\usepackage{bm}
\newtheorem{remark}{Remark}
\usepackage[colorlinks,
            linkcolor=blue,
            anchorcolor=blue,
            citecolor=blue,
            ]{hyperref}

\newcommand{\PreserveBackslash}[1]{\let\temp=\\#1\let\\=\temp}
\newcolumntype{C}[1]{>{\PreserveBackslash\centering}p{#1}}
\newcolumntype{R}[1]{>{\PreserveBackslash\raggedleft}p{#1}}
\newcolumntype{L}[1]{>{\PreserveBackslash\raggedright}p{#1}}

\usepackage{booktabs}
\def\X{{\mathbf{X}}}
\def\U{{\mathbf{U}}}
\def\V{{\mathbf{V}}}
\def\W{{\mathbf{W}}}
\def\B{{\mathbf{B}}}
\def\A{{\mathbf{A}}}
\def\D{{\mathbf{D}}}

\def\E{{\mathbf{E}}}
\def\Z{{\mathbf{Z}}}
\def\L{{\mathbf{L}}}
\def\S{{\mathbf{S}}}
\def\M{{\mathbf{M}}}
\def\Y{{\mathbf{Y}}}
\usepackage{graphicx}
\usepackage{algorithm}
\usepackage{algorithmic}
\renewcommand{\arraystretch}{1.12}
\usepackage{array}
\newcolumntype{C}[1]{>{\PreserveBackslash\centering}p{#1}}

\newtheorem{proof}{Proof}
\newtheorem{theorem}{Theorem}
\newtheorem{corollary}{Corollary}
\newtheorem{lemma}{Lemma}
\newtheorem{assumption}{Assumption}
\newtheorem{definition}{Definition}

\usepackage{pifont}

\begin{document}
	
	\title{Beyond Low-rankness: Guaranteed Matrix Recovery via Modified Nuclear Norm}

	\author{ 
	Jiangjun Peng,~\IEEEmembership{Member,~IEEE},
	Yisi Luo,
    Xiangyong Cao, ~\IEEEmembership{Member,~IEEE},
    Shuang Xu,
	and Deyu Meng, ~\IEEEmembership{Member,~IEEE},
\thanks{Jiangjun Peng and Xu Shuang are with the School of Mathematics and Statistics, Northwestern Polytechnical University, Xi’an 710021, China. Email: pengjj@nwpu.edu.cn, xs@nwpu.edu.cn}
\thanks{Yisi Luo and Deyu Meng are with the School of Mathematics and Statistics and Ministry of  Education Key Lab of  Intelligent Networks and Network Security, Xi’an Jiaotong University, Xi’an 710049, China. Email: yisiluo1221@foxmail.com, dymeng@mail.xjtu.edu.cn.}
\thanks{Xiangyong Cao is with the School of Computer Science and Technology, Xi’an Jiaotong University, Xi’an 710049, China. Email: caoxiangyong@mail.xjtu.edu.cn.}}

	
	\maketitle
	%
	\begin{abstract}
The nuclear norm (NN) has been widely explored in matrix recovery problems, such as Robust PCA and matrix completion, leveraging the inherent global low-rank structure of the data. In this study, we introduce a new modified nuclear norm (MNN) framework, where the MNN family norms are defined by adopting suitable transformations and performing the NN on the transformed matrix. The MNN framework offers two main advantages: (1) it jointly captures both local information and global low-rankness without requiring trade-off parameter tuning; (2) under mild assumptions on the transformation, we provide exact theoretical recovery guarantees for both Robust PCA and MC tasks—an achievement not shared by existing methods that combine local and global information. Thanks to its general and flexible design, MNN can accommodate various proven transformations, enabling a unified and effective approach to structured low-rank recovery. Extensive experiments demonstrate the effectiveness of our method. Code and supplementary material are available at \url{https://github.com/andrew-pengjj/modified_nuclear_norm}.
	\end{abstract}
	
	\IEEEpeerreviewmaketitle
	
\section{Introduction}
\label{intro}
The nuclear norm (NN), serving as a convex relaxation of the matrix rank, is widely applied to various matrix optimization problems due to its ability to effectively preserve the global low-rank property of data while possessing many desirable theoretical properties \cite{candes2012exact,candes2011robust,de2008tensor,recht2010guaranteed,gu2014weighted}. The data property is typically encoded by a regularizer term $ \mathcal{R}(\cdot) $ and incorporated into the following energy function (\ref{ener}):
\begin{equation}
\label{ener}
\min_{\X} E(\X,\Y) := L(\X,\Y) + \lambda  \mathcal{R}(\X),
\end{equation}
where $\X$ and $\Y$ represent the data to be recovered and the observed data, respectively. The function $ L(\X,\Y) $ is the loss function, which can be in the form of $ \ell_1 $ and $\ell_2$ norm. 
	\begin{figure}[!t]
		\centering
		\includegraphics[scale=0.45]{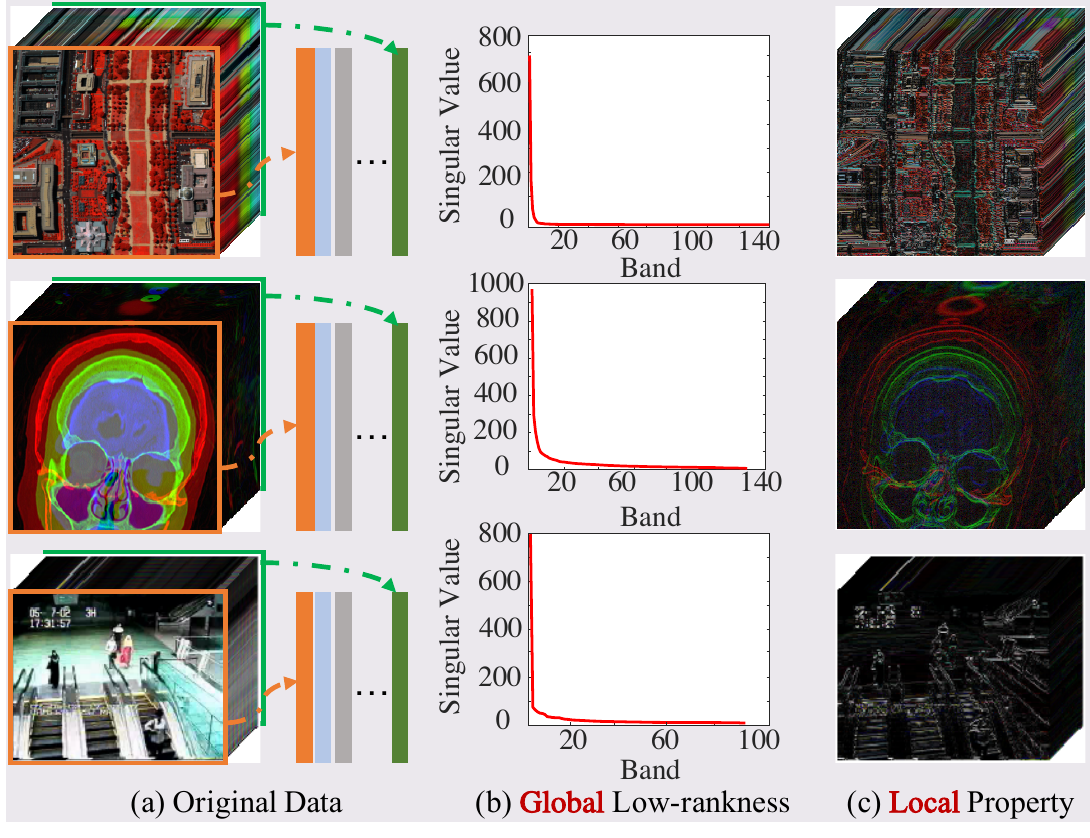}
		\caption{Demonstrations of the global low-rank property and local prior of data. (a) From top to bottom, showcasing the hyperspectral data, CT images and surveillance video data, as well as a matrix expanded along the third dimension; (b) Singular value curve of the mode-3 unfolding matrices of the data in (a); (c) Difference maps are obtained by applying difference operators. Difference maps have sparsity, which means that the original data has local smoothness.}
		\label{LRother}
	\end{figure}

However, in many cases, alongside the global low-rankness of data, local prior information also plays a crucial role, especially in data related to computer vision tasks such as hyperspectral image \cite{wang2017hyperspectral}, multispectral image \cite{xie2017kronecker}, CT images \cite{wang2023guaranteed}, and surveillance videos \cite{peng2022exact}, as shown in Fig. \ref{LRother}. Efficiently integrating global and local information to obtain a model with theoretical guarantees is an important issue for restoration tasks. Previous methods of integrating global and local information typically involve constructing multiple regularizer terms and combining them additively \cite{wang2017hyperspectral,he2015total,xue2021spatial,he2017total,shang2023hyperspectral,chen2023hyperspectral}. Taking global low-rankness and local smoothness properties as an example, the traditional fused regularization method is given by:
\begin{equation}
\label{nn_tv}
\mathcal{R}(\X):=\Vert \X \Vert_* + \beta \Vert \X \Vert_{\mbox{\scriptsize{TV}}},
\end{equation}
where $ \Vert \X \Vert_{\mbox{\scriptsize{TV}}} $ is the total variation (TV) regularizer to encode the sparsity of difference map of the original data \cite{rudin1992nonlinear}, and $\beta$ is the trade-off parameter need to be fine-tuned for each data. The TV regularizer often fails to capture the local smoothness of the data adequately. Therefore, in many cases, additional regularizers need to be incorporated in the manner (\ref{nn_tv}). This inevitably introduces the problem of selecting more balanced parameters. Besides, except for NN, many regularizers don't hold good theoretical properties. 

To tackle the above two drawbacks, we introduce a new modified nuclear norm (MNN) that can capture both the local information and global low-rankness of data. Specifically, the MNN is defined by adopting a suitable transformation and then performing the NN on the transformed matrix, i.e., 
\begin{equation}
\label{mnn_def}
\Vert \X \Vert_{\scriptsize{\mbox{MNN}}}:=\Vert \mathcal{D}(\X) \Vert_{*}.
\end{equation}
Compared to NN, MNN (\ref{mnn_def}) can concurrently capture local information and global low-rankness of data. Specifically, the transformation $\mathcal{D}(\cdot)$ exploits useful local correlation information of the matrix and performing the NN on the transformed matrix could finely utilize such local correlation based on the compatibility of norms. 

Specifically, the proposed MNN framework offers two key advantages over existing methods that integrate global and local information. First, MNN eliminates the need for parameter tuning required in additive manner (\ref{nn_tv}). Second, under mild transformation assumptions, MNN provides exact recovery guarantees for two typical matrix applications: Robust PCA and Matrix Completion (MC). Building on theoretical tools related to nuclear norms \cite{candes2012exact,candes2011robust,chen2015incoherence,candes2010matrix,negahban2012restricted,shahid2015robust}, the MNN framework leverages these tools to offer theoretical guarantee. It incorporates various classical transformations, such as the difference operator \cite{peng2022exact,wang2023guaranteed,liu2023tensor}, the Sobel operator \cite{kanopoulos1988design}, and the Laplacian operators \cite{wang2007laplacian}, enabling the integration of diverse local information. In summary, the contributions of this paper are:

\textbf{Modeling}: We propose the MNN framework, which can simultaneously exploit global and local information of data by performing the NN on a transformed matrix that encodes local correlations through a suitable transformation operator without trade-off parameters. 

\textbf{Theory}: Under mild conditions, we prove the exact recoverability theory of MNN on two types of problems: Robust PCA and MC. Although the theoretical bounds are not improved compared to NN, from the perspective of embedding local smoothness, MNN provides a unified recoverability theory framework for integrating low-rank and multi-layer local information priors. Extensive experiments validate the efficacy of the MNN framework. 

\textbf{Operator}: Classical operators are introduced into the MNN framework to characterize local smoothness. This paper demonstrates that for low-rank image data, the difference, Sobel, and Laplacian operators can be directly embedded into MNN to improve model's performance. In particular, Laplacian operators can provide richer local information than the widely used first-order differences operator.

\section{Modified Nuclear Norm}
\label{mnn}
In this part, we will give the forms and theoretical results about the MNN-induced Robust PCA and MC models.  
\subsection{Motivations}
\label{analysis_mnn}
The previous global and local information fusion model \cite{wang2017hyperspectral,he2015total,xue2021spatial,he2017total,peng2022fast,peng2020enhanced} was characterized by constructing multiple regularizers and summing them up, as follows:
\begin{equation}
\label{tv_nn}
\sum_{i=1}^n \tau_i \mathcal{R}_i(\X),
\end{equation}
where $ \mathcal{R}_i(\X) (i=1,\cdots,n) $ are regularizers, such as the nuclear norm (NN) and total variation (TV) regularizer in the form (\ref{nn_tv}), and $\tau_i \) are trade-off parameters, manner (\ref{tv_nn}) is intuitive and effective. However, it suffers from two key issues: the challenge of parameter selection among multiple regularizers and the lack of theoretical recoverability. To address these limitations, inspired by the norm compatibility and exact recovery guarantees of the nuclear norm, we propose the MNN defined in Eq. (\ref{mnn_def}). As demonstrated in Remark \ref{remark_mnn}, Eq. (\ref{mnn_def}) effectively models both global and local information, eliminating the need for hyperparameter tuning.

\begin{remark}
\label{remark_mnn} MNN (\ref{mnn_def}) can simultaneously encode global and local information based on two aspects.

\textbf{1) MNN could capture global low-rankness}. We assume that $\mathcal{D}(\cdot)$ can be reformulated as a full-rank matrix-induced linear operator, i.e., there exists a full-rank matrix $\A$ such that
\begin{equation}
\mbox{rank}(\A\X) = \mbox{min}(\mbox{rank}(\A),\mbox{rank}(\X)) = \mbox{rank}(\X).
\end{equation}
Based on the above derivation, it can be inferred that the low-rankness of $\mathcal{D}(\X)$ stems from the low-rankness of $\X$, 
i.e., minimizing the rank of $\mathcal{D}(\X)$ is equivalent to minimizing the rank of $\X$. Hence, MNN can characterize the global low-rankness of the data.

\textbf{2) MNN could capture local information}. According to the norm compatibility theorem, we have
\begin{equation}
\label{nn_norm}
\Vert \mathcal{D}(\X) \Vert_F \leq \Vert \mathcal{D}(\X) \Vert_*\leq \Vert \mathcal{D}(\X) \Vert_1,
\end{equation}
holds for all transformation operators. Thus, it can be inferred that minimizing $\Vert \mathcal{D}(\X) \Vert_*$ can simultaneously minimize both $\Vert \mathcal{D}(\X) \Vert_F$ and $\Vert \mathcal{D}(\X) \Vert_1$. If an appropriate transformation operator $\mathcal{D}(\cdot)$ is chosen so that $\mathcal{D}(\X)$ can capture 
useful local information of data (e.g., edge structures), then MNN can effectively utilize such local information based on Eq. (\ref{nn_norm}). For example, when $\mathcal{D}(\cdot)$ is chosen as the first-order difference operator $\nabla(\cdot)$, $\Vert \nabla(\X)\Vert_F$ and $\Vert \nabla(\X)\Vert_1$ correspond to isotropic \cite{rudin1992nonlinear} and anisotropic \cite{he2017total} TV regularizer, respectively. Therefore, our MNN can then characterize the local smoothness of data \footnote{If $\mathcal{D}(\cdot)$ is set as identity mapping, $\Vert \mathcal{D}(\X) \Vert_*$ degenerates to  $\Vert \X\Vert_*$. Since $\Vert \X \Vert_F$ and $\Vert \X \Vert_1$ cannot encode local structural information of the data, $\Vert \X \Vert_*$ can only characterize the global low rankness.}.
\end{remark}

Indeed, several works such as Correlated Total Variation (CTV \cite{wang2023guaranteed,peng2022exact,liu2023tensor}) can be seen as a special case of our MNN by setting the transformation operator as the first-order difference. 
\begin{remark}
Unlike the first-order difference in CTV, which only captures adjacent similarities, the proposed MNN framework accommodates operators with larger receptive fields, enabling more effective modeling of local smoothness across broader regions. What's more, MNN also removes the need for tuning the parameters between multiple regularization terms. While combining several regularizers improves interpretability, it often requires careful hyperparameter tuning—poor choices can yield suboptimal results.
\end{remark}

\subsection{Theoretical Guarantees}
Next, we establish the exact recovery guarantees of MNN-induced RPCA and MC models.

\subsubsection{Models}
Robust PCA \cite{candes2011robust} is a problem aiming at accurately separating a low-rank matrix $\X_0\in \mathbb{R}^{n_1\times n_2}$ and a sparse matrix $\S_0$ from the observed data matrix $\M=\X_0+\S_0$. By using MNN and $\ell_1$ norms to respectively encode the global low-rankness and local priors of low-rank matrices, as well as the sparsity of sparse matrices, we can derive the following MNN-RPCA model:
\begin{equation}
\label{mnn_rpca}
\min_{\X,\S} \Vert \mathcal{D}(\X) \Vert_* +\lambda \Vert \S \Vert_1, s.t. \ \M = \X+\S.
\end{equation}
MC \cite{candes2012exact} is a problem aiming to accurately infer the clean data $\X_0$ from a limited set of observed data $\M$ (the support set denoted as $\Omega$). Using MNN, we can obtain the following MNN-MC model:
\begin{equation}
\label{mnn_mc}
\min_{\X} \Vert \mathcal{D}(\X) \Vert_*, s.t. \ \mathcal{P}_{\Omega}(\M) = \mathcal{P}_{\Omega}(\X),
\end{equation}
where $ \mathcal{P}_{\Omega}(\cdot) $ is a mapping operator. If $(i,j)\in \Omega$, then $\mathcal{P}_{\Omega}(\X_{ij}) = \X_{ij}$; otherwise, it is set to 0. 

In our MNN-based model, the single regularization term simultaneously captures both global and local information, which avoids the need for carefully tuning the trade-off parameter $\tau_i$ required by additive model-based methods (\ref{tv_nn}).

\subsubsection{Assumptions}
Before giving the theorems, we need three mild assumptions.

\begin{assumption}[Incoherence Condition]
\label{assumption1}
For the low-rank matrix $\X_0 \in \mathbb{R}^{n_1\times n_2}$ with rank $r$, it follows the incoherence condition with parameter $ \mu $, i.e.,
\begin{equation}
\begin{split}
\max_{k} & \| \U^* e_k\| \leq \frac{\mu r}{n_1}, \max_{k} \| \V^* e_k\| \leq \frac{\mu r}{n_2},  \\
& \| \U\V^* \|_{\infty} \leq \sqrt{\frac{\mu r}{n_1n_2}},
\end{split}
\end{equation}
where $\U \in \mathbb{R}^{n_1\times r}$ and $\V\in \mathbb{R}^{n_2\times r}$ are obtained from the singular vector decomposition of $\X_0$ and $ e_k $ is the unit orthogonal vector.
\end{assumption}

Incoherence condition is a widely used assumption for low-rank recovery problems \cite{candes2011robust,chen2015incoherence,candes2007sparsity} to control the dispersion of the elements of the low-rank matrix. If the data is a rank-r identity matrix, since the observed values are likely to be zero, it is difficult to infer the non-zero elements, so we need this condition to keep the data away from the identity matrix.

The random distribution assumption for $\S_0$ and the normalization assumption of $\mathcal{D}(\cdot)$ are as follows:
\begin{assumption}[Random Distribution]
\label{assumption2}
For the sparse term $\S_0$, its support $ \Omega $ is chosen uniformly among all sets of cardinality $ m $, and the signs of supports are random, i.e.
\begin{equation}
\begin{split}
\mathbb{P}[(\S_0)_{i,j}>0 | (i,j)  \in \Omega] & = \mathbb{P}[(\S_0)_{i,j}\leq 0 | (i,j) \in \Omega] \\
& =0.5.
\end{split}
\end{equation}
\end{assumption}
\begin{assumption}[Normalization]
\label{assumption3}
The Frobenius norm of the linear transformation $\mathcal{D}(\cdot)$ in Eq. (\ref{nn_norm}) is one.
\end{assumption}
Assumption \ref{assumption3} ensures that the elements of the transformed data $\mathcal{D}(\X)$ remain bounded, allowing key inequalities in the dual verification to hold. If the Frobenius norm of $\mathcal{D}(\cdot)$ is not one, normalization can be used to satisfy Assumption \ref{assumption3}.

\subsubsection{Main Results}
Based on Assumptions \ref{assumption1}-\ref{assumption3}, we can derive that:
\begin{theorem}[MNN-RPCA Theorem]
\label{theorem_main3} 
Suppose that $\mathcal{D}(\X_0) \in \mathbb{R}^{n_1\times n_2}$, $\S_0$ and $\mathcal{D}(\cdot)$ obey Assumptions \ref{assumption1}-\ref{assumption3}, respectively. Without loss of generality, suppose $ n_1 \geq n_2 $. Then, there is a numerical constant $ c>0 $ such that with probability at least $ 1-cn_1^{-10} $(over the choice of support of $ \S_0 $), the MNN-RPCA model (\ref{mnn_rpca}) with $ \lambda = 1/(\sqrt{n_1})$ is exact, i.e., the solution $ (\hat{\X},\hat{\S})=(\X_0,\S_0)$, provided that
\begin{equation}
\label{pho_r_mnn}
\mbox{rank}(\X_0) \leq \rho_r n_{2} \mu^{-1}(\log n_1)^{-2}, ~ \mbox{and} ~ m \leq \rho_s n_1n_2,
\end{equation}
where $ \rho_r $ and $ \rho_s $ are some positive numerical constants, and $m$ is the number of the support set of $\S_0$.
\end{theorem}
\begin{theorem}[MNN-MC Theorem]
\label{theorem_main4}
Suppose that $\mathcal{D}(\X_0) \in \mathbb{R}^{n_1\times n_2}$ and $\mathcal{D}(\cdot)$ obey Assumptions \ref{assumption1} and \ref{assumption3}, $\Omega\sim\mbox{Ber}(p)$ and $ m$ is the number of $\Omega $, where $\mbox{Ber}(p)$ represents the Bernoulli distribution with $p$. Without loss of generality, suppose $ n_1 \geq n_2 $. Then, there exist universal constants $c_0, c_1>0$ such that $\X_0$ is the unique solution to MNN-MC model (\ref{mnn_mc}) with probability at least $1-c_1n_1^{-3}\log n_1$, provided that
\begin{equation}\label{eq.17_mnn}
m\geq c_0\mu r n_1^{5/4}\log(n_1).
\end{equation}
\end{theorem}
Theorems \ref{theorem_main3} and \ref{theorem_main4} establish the exact recoverable theory for the fusion-induced model combining low-rank and general local smoothness priors, which was unavailable in previous studies. Proofs are provided in Appendices \ref{rpca_proof} and \ref{mc_proof}.

\subsection{Optimizations}
\label{opti}
The exact recoverable theory, as outlined in Theorems \ref{theorem_main3} and \ref{theorem_main4}, asserts that the optimal solution $(\hat{\X},\hat{\S})$ of the model  accurately reflects the true value $(\X_0,\S_0)$. Thus, the following corollary can be inferred with high probability.
\begin{corollary}
\label{coro_mnn}
Suppose $\X_0$ and $\S_0$ satisfy Assumptions \ref{assumption1} and \ref{assumption2}, and transformation operator $\mathcal{D}(\cdot)$ satisfy Assumption \ref{assumption3}. Denote the objective functions of the RPCA and MC models as 
\begin{equation}
\begin{split}
\mathcal{J}_1^{\mathcal{D}}(\X) := & \Vert \mathcal{D}(\X) \Vert_* +\lambda \Vert \M-\X \Vert_1, \\
\mathcal{J}_2^{\mathcal{D}}(\X) := & \Vert \mathcal{D}(\X) \Vert_* +\mu \Vert \mathcal{P}_{\Omega}(\M-\X) \Vert_F^2,\\
\end{split}
\end{equation}
respectively, where $\lambda =1/\sqrt{n_1}$ and $\mu=(\sqrt{n_1}+\sqrt{n_2})\sqrt{p}\sigma$ according to \cite{candes2010matrix}, and $n_1, n_2, \sigma, p$ are the sizes of matrix, noise standard variance, and missing ratio. Then, for any $\X$, we have:
\begin{equation}
\label{obj}
\mathcal{J}_1^{\mathcal{D}}(\X) \geq \mathcal{J}_1^{\mathcal{D}}(\X_0), \mathcal{J}_2^{\mathcal{D}}(\X) \geq \mathcal{J}_2^{\mathcal{D}}(\X_0).
\end{equation}
\end{corollary} 
Then, according to Corollary \ref{coro_mnn}, we can directly perform a simple gradient descent on the objective function (\ref{obj}) to obtain the final solution. More details can be seen in Appendix \ref{opt_proof}. 

\subsection{Some Suitable Transformation Operators for MNN}
\label{operator}
In this section, we shift our attention to another important aspect, namely, the selection of the transformation operator $\mathcal{D}(\cdot)$. The local smoothness of images is typically achieved through convolution operators. The most common one is the first-order difference operator $ \nabla(\cdot)$ =[-1,1]/[-1;1]. Based on the first-order difference operator \cite{rudin1992nonlinear} and its fast Fourier transform-based solving algorithm \cite{chambolle2004algorithm,huang2008fast}, various types of total variation regularization have been developed in recent decades \cite{wang2017hyperspectral,he2015total,peng2020enhanced}. 
However, the difference operators may not be able to characterize local information of data with multiple directions and higher-order derivative information. Hence, we consider several other transformations by taking advantage of multiple directions and higher-order derivatives. 

Specifically, we examine four common convolution kernels: the first-order central difference operator, the Sobel operator, and two types of Laplacian operators. Their forms and edge extraction effects on the "Barbara" image are shown in Figure  \ref{lena_ind}. It can be observed that compared to the central difference operator, the latter three operators capture more abundant local information about the image.
 
	\begin{figure}[htp]
		\centering
		\includegraphics[scale=0.4]{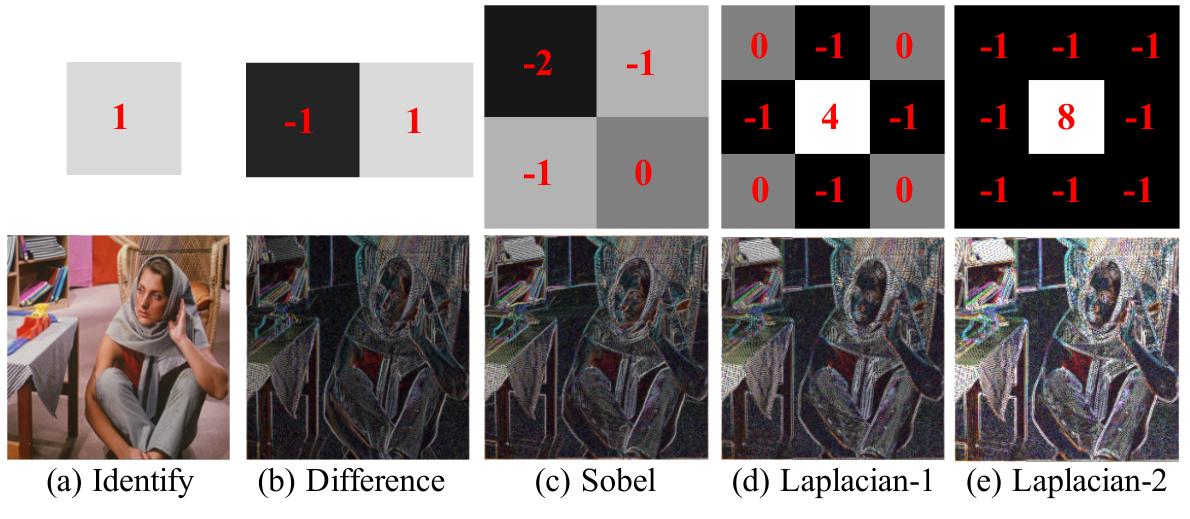}
		\caption{Demonstrates the forms of several typical convolution kernels and its edge extraction effect after acting on the image.}
		\label{lena_ind}
	\end{figure}
Next, we give an intuitive mathematical explanation of the superior local correlation excavation abilities of the Sobel and Laplacian operators by Remark \ref{remark_1}.

\begin{remark}
\label{remark_1}
For an image, it can be modeled as a bivariate function $z = f(x, y)$, where $x, y$ represent the coordinates, and $z$ represents the grayscale value. The first-order difference of $f(x, y)$ is:
\begin{equation}
\begin{split}
\frac{\partial{f}}{\partial{x}}&=f(x+1,y)-f(x,y), \\ \frac{\partial{f}}{\partial{y}}&=f(x,y+1)-f(x,y). 
\end{split}
\end{equation}
The second-order difference of $f(x, y)$ can be represented in the $ x $ and $ y $ directions as follows:
\begin{equation}
\begin{split}
\frac{\partial^2{f}}{\partial{x^2}} &=f(x+1,y)-2f(x,y)+f(x-1,y), \\
\frac{\partial^2{f}}{\partial{y^2}} &=f(x,y+1)-2f(x,y)+f(x,y-1). 
\end{split}
\end{equation}
The first-order difference operator $\frac{\partial{f}}{\partial{x}}/\frac{\partial{f}}{\partial{y}}$ corresponds to the convolution kernel [-1,1]/[-1;1]. The Sobel operator $\frac{\partial{f}}{\partial{x}}+\frac{\partial{f}}{\partial{y}}$ and Laplacian operator $\frac{\partial^2{f}}{\partial{x^2}}+\frac{\partial^2{f}}{\partial{y^2}}$ have effects equivalent to [-2,-1;-1,0] and [0,-1,0;-1,4,-1;0,-1,0], respectively. The first derivative detects edges, while the second derivative detects the rate of change of edges. Derivatives in multiple directions can comprehensively characterize edge information. Therefore, the Sobel operator and the two types of Laplacian operators provided in Figure \ref{lena_ind} can exploit richer local priors.
\end{remark}

\begin{figure}[htp]
		\centering
		\includegraphics[scale=0.4]{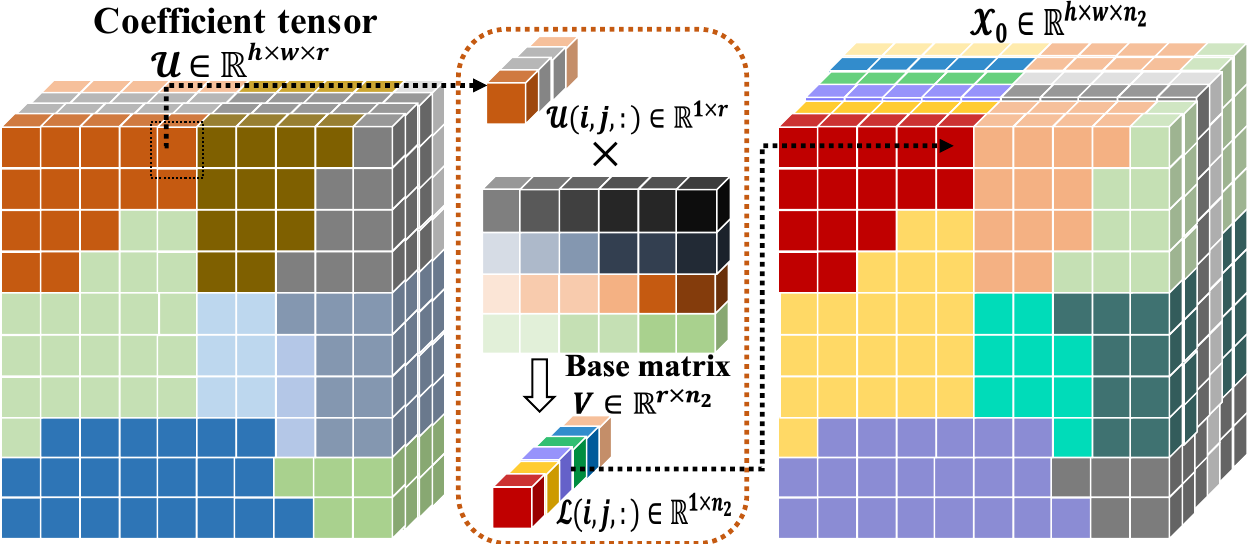}
		\caption{The simulated data generated mechanism of joint low rank and local smoothness data $\X_0$.}
		\label{simulation_generate}
	\end{figure}
	
	\begin{figure*}[!]
		\centering
		\includegraphics[scale=0.132]{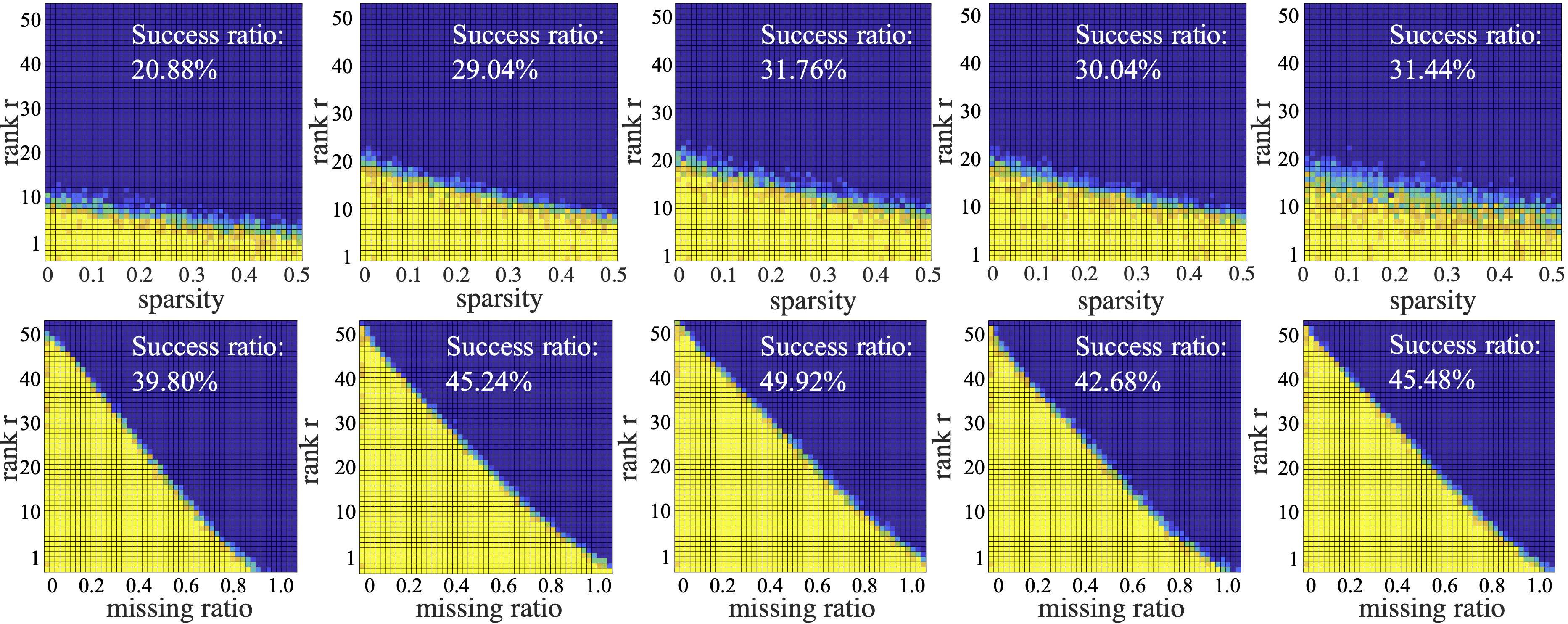}
		\caption{Phase transition diagrams of MNN variants for Robust PCA (top row) and MC (bottom row). Blue (0\%) and yellow (100\%) indicate success ratios. From left to right: NN, MNN-Diff, MNN-Sobel, MNN-L1, MNN-L2.
}
		\label{dsp_mc_fig}
	\end{figure*}

\section{Simulated Experiment}
\label{simu_exp}
The numerical experiments are conducted to validate Theorems \ref{theorem_main3} and \ref{theorem_main4}. As described in Section \ref{operator}, we use difference, Sobel, and Laplacian operators in MNN, denoted as MNN-Diff, MNN-Sobel, MNN-L1, and MNN-L2 in Figure \ref{lena_ind}. Note that MNN-Diff reduces to the CTV regularizer.  Following Corollary \ref{coro_mnn}, we set $ \lambda = 1/\sqrt{\max\{n_1, n_2\}} $ for the TRPCA task and \(\mu = (\sqrt{n_1} + \sqrt{n_2})\sqrt{p}\sigma \), \(\sigma = 1e^{-4}\) for the MC noiseless task. Further performance gains can be achieved by fine-tuning \(\lambda\) and \(\mu\). All simulations are run on a PC with an Intel Core i5-10600KF  CPU (4.10 GHz), 32 GB RAM, and a GeForce RTX 3080 GPU.

\subsection{Data Generation}
We create two factor matrices $\U \in \mathbb{R}^{n_1\times r} $ and $ \mathbf{V} \in \mathbb{R}^{r \times n_2}$, yielding $\X_0 = \U\V^T$ as a low-rank matrix, where $n_1=hw$. To introduce image-like properties, each column of $\U$ (i.e., $\U(:,i)$) is reshaped into an $h\times w$ matrix, thus we can get the tensor $\mathcal{U}\in \mathbb{R}^{h\times w\times r}$ satisfies $\U=\mbox{unfold}_3(\mathcal{U})$ where $\mbox{unfold}_3(\cdot)$ is the unfold operator to unfold the tensor to matrix along the third-mode. Each slice $\mathcal{U}(:,:,i)$ is randomly divided into $c$ regions, with elements within each region being consistent and drawn from $\mathcal{N}(0,1)$ distribution. The elements of $\mathbf{V}$ are selected from $\mathcal{N}(0,1)$ distribution. The entire generation process is shown in Figure  \ref{simulation_generate}. Besides, the support set $\Omega$ is chosen randomly. For RPCA task, $ \S_0 $ is set as $ \S_0 =\mathcal{P}_{\Omega} (\E) $ and the observed matrix $\M$ is set as $ \M=\X_0+\S_0 $, where $\E$ is a matrix with independent Bernoulli $ \pm 1 $ entries. For MC task, $\M$ is set as $\M =\mathcal{P}_{\Omega}(\X_0)$.

\begin{figure}[!]
		\centering
		\includegraphics[scale=0.52]{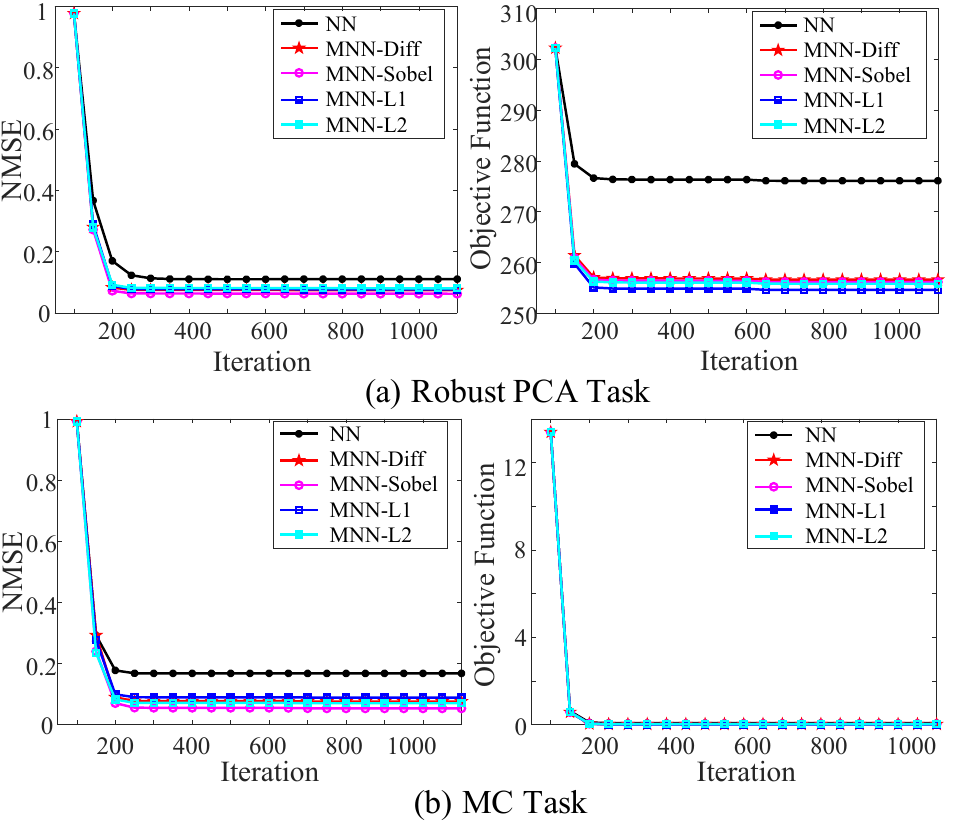}
		\caption{Objective function and relative recovery error over iterations: (a) RPCA with $r=10, \rho_s=0.1$; (b) MC with $r=10$, sampling ratio 0.2.}
		\label{mnn_ite}
	\end{figure}
	
\begin{table}[htbp]
\centering
\fontsize{8.5}{8.5}\selectfont
\caption{Comparison of methods for fusing global and local priors.}
\begin{tabular}{l|c}
\hline
\textbf{Type} & \textbf{Methods} \\
\hline
Non-convex Decomposition & KBR, LRTDTV,\\
Multi-regularization Fusion & LRTV, LRTDTV,\\
Matrix-based Methods & NN, LRTV, CTV\\
Tensor-based Methods &TCTV,SNN,TNN, Qrank, Framelet\\
\hline
\end{tabular}
\label{tab:methods_comparison}
\end{table}

\subsection{Experiment Settings and Result Analysis}

\textbf{Experiment Settings}. In all experiments, we set $h=w=50$, $n1=hw=2500$, and $n2=100$. We evaluate how the rank $r$, sparsity $\rho_s$ (RPCA), and missing ratio $\rho$ (MC) affect performance by varying $\rho_s$ in $(0.01, 0.5)$ (step 0.01), $\rho$ in $(0.01, 0.99)$ (step 0.02), and $r$ in $(1, 50)$ (step 1).

\textbf{Phase Transition Diagram}. For each $(r,\rho_s)$ and $(r,\rho)$, we run 10 trials and consider recovery successful if $\Vert \hat{\X} - \X_0 \Vert_F/ \Vert \X_0 \Vert_F \leq 0.05$. Figure \ref{dsp_mc_fig} shows the phase diagrams with recovery rates. Compared to NN, MNN variants (Diff, Sobel, L1, L2) significantly expand the recovery region by leveraging local structures. Notably, MNN-Sobel, MNN-L1, and MNN-L2 outperform MNN-Diff, highlighting the advantage of high-order operators as noted in Remark \ref{remark_1}.

\begin{table*}[htp]
\setlength\tabcolsep{4.5pt}
\fontsize{10}{10}\selectfont
  \centering
  \caption{Restoration comparison of all competing methods under different salt and pepper noise noise variance (SR). The best and second results are highlighted in bold italics and \underline{underlined}.}
\begin{tabular}{l|c|c|c|c|c|c|c|c|c|c|c|c|c}
     \Xhline{2\arrayrulewidth}
     SR & Metrics & NN & SNN & KBR & TNN &LRTV &  \makecell[c]{LRTD\\TV} & CTV & TCTV &  \makecell[c]{MNN-\\Diff}   &   \makecell[c]{MNN-\\Sobel}  &  \makecell[c]{MNN-\\L1}  &  \makecell[c]{MNN-\\L2}  \\
     \Xhline{2\arrayrulewidth}
		\multicolumn{13}{c}{Hyperspectral Image Denoising: Five Datasets} \cr
	 \hline
\multirow{2}{*}{0.1} & PSNR$\uparrow$ & 46.16& 29.58& 35.97& 45.43& 32.57& 31.21& 48.53& 47.30& 47.57& \textbf{49.52}& \underline{49.33}& 47.99 \\
& SSIM$\uparrow$ & 0.998& 0.897& 0.974& 0.989& 0.928& 0.894& 0.998& 0.992& 0.997& \textbf{0.999}& \textbf{0.999}& 0.998 \\
\hline
\multirow{2}{*}{0.3} & PSNR$\uparrow$ & 39.93& 23.21& 33.31& 40.34& 30.26& 29.30& 45.69& 43.55& 44.99& \textbf{47.08}& \underline{46.70}& 45.52 \\
& SSIM$\uparrow$ & 0.992& 0.608& 0.959& 0.980& 0.883& 0.845& 0.997& 0.988& 0.997& \textbf{0.999}& \textbf{0.998}& 0.997 \\
\hline
& SSIM$\uparrow$ & 0.690& 0.138& 0.519& 0.051& 0.644& 0.642& 0.962& 0.366& 0.781& \textbf{0.991}& 0.910& \textbf{0.991} \\
\hline
\multicolumn{2}{c|}{Average Time/s} & 8.62& 116.8& 58.39& 136.4& 33.79& 102.9& 96.63& 277.1& 77.61& 75.04& 74.12& 77.69  \\
	  \hline
	  \hline
	  \multicolumn{13}{c}{Multispectral Image Denoising: Eleven Datasets} \cr
	 \hline
\multirow{2}{*}{0.1} & PSNR$\uparrow$ & 29.85& 33.39& 35.76& 41.70& 33.56& 36.32& 40.64& \textbf{44.64}& 40.15& 42.14& \underline{42.92}& 40.74 \\
& SSIM$\uparrow$ & 0.968& 0.967& 0.965& 0.993& 0.957& 0.977& 0.995& \textbf{0.996}& 0.995& \textbf{0.996}& \textbf{0.996}& 0.995 \\
\hline
\multirow{2}{*}{0.3} & PSNR$\uparrow$ & 24.96& 31.14& 34.32& 39.14& 31.90& 34.48& 38.57& \textbf{42.95}& 38.34& 40.47& \underline{41.39}& 39.43 \\
& SSIM$\uparrow$ & 0.920& 0.945& 0.956& 0.988& 0.943& 0.965& 0.992& \textbf{0.995}& 0.992& 0.994& \textbf{0.995}& 0.993 \\
\hline
\multicolumn{2}{c|}{Average Time/s} & 17.01& 294.5& 106.8& 275.6& 52.91& 220.9& 233.5& 513.9& 68.62& 70.24& 67.24& 67.35  \\
	  \hline
	  \hline
	  \multicolumn{13}{c}{RGB Video Denoising: Ten Datasets} \cr
	 \hline
\multirow{2}{*}{0.1} & PSNR$\uparrow$ & 32.26& 27.44& 26.48& 38.01& 29.97& 23.51& 35.73& \underline{38.69}& 37.14& 38.61& \textbf{39.77}& 37.62  \\
& SSIM$\uparrow$ & 0.958& 0.774& 0.827& 0.980& 0.892& 0.809& 0.976& \underline{0.983}& 0.980& 0.983& \textbf{0.984}& 0.981   \\
\hline
\multirow{2}{*}{0.3} & PSNR$\uparrow$ & 30.16& 10.58& 25.89& 35.47& 28.52& 22.43& 33.91& \underline{36.63}& 34.90& 36.57& \textbf{37.45}& 35.84 \\
& SSIM$\uparrow$ & 0.933& 0.097& 0.808& 0.963& 0.870& 0.777& 0.963& 0.967& 0.970& \underline{0.977}& \textbf{0.979}& 0.975  \\
\hline
\multicolumn{2}{c|}{Average Time/s} & 28.17& 104.1& 72.61& 142.9& 52.67& 119.9& 129.2& 293.1& 157.6& 159.8& 157.8& 159.9  \\
	  \hline
	  \hline
	  \Xhline{1\arrayrulewidth}
\end{tabular}
\label{rpca_real}
\end{table*}

\begin{figure*}[!h]
\renewcommand{\arraystretch}{0.5}
\setlength\tabcolsep{0.5pt}
\centering
\begin{tabular}{ccccccccc}
\centering
\scriptsize{Original} & \scriptsize{Noisy} & \scriptsize{NN} & \scriptsize{LRTV} & \scriptsize{LRTDTV} & \scriptsize{CTV} & \scriptsize{TCTV} & \scriptsize{MNN-Sobel} & \scriptsize{MNN-L2} \\
\includegraphics[width=19mm, height = 19mm]{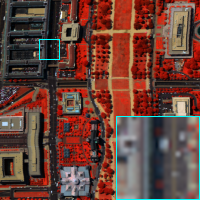}&
\includegraphics[width=19mm, height = 19mm]{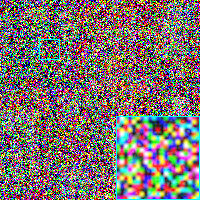}&
\includegraphics[width=19mm, height = 19mm]{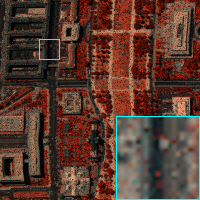}&
\includegraphics[width=19mm, height = 19mm]{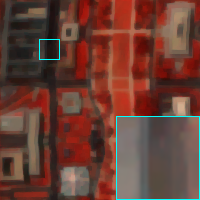}&
\includegraphics[width=19mm, height = 19mm]{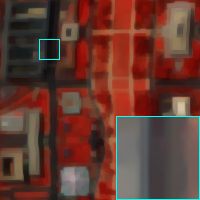}&
\includegraphics[width=19mm, height = 19mm]{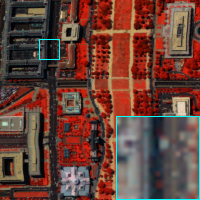}&
\includegraphics[width=19mm, height = 19mm]{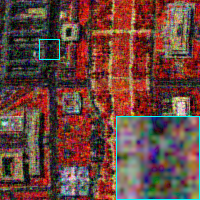}&
\includegraphics[width=19mm, height = 19mm]{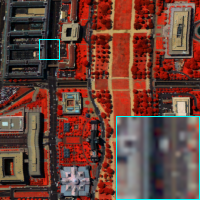}&
\includegraphics[width=19mm, height = 19mm]{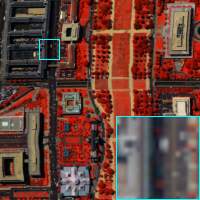} \\
\scriptsize{DCmall} & \scriptsize{6.02/0.026} & \scriptsize{24.76/0.842} & \scriptsize{22.72/0.563} & \scriptsize{22.31/0.543} & \scriptsize{35.14/0.978} & \scriptsize{19.19/0.452} & \scriptsize{\textbf{41.57/0.996}} & \scriptsize{\underline{40.79/0.995}} \\
\includegraphics[width=19mm, height = 16mm]{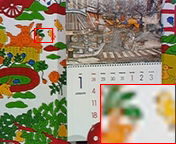}&
\includegraphics[width=19mm, height = 16mm]{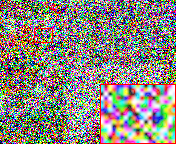}&
\includegraphics[width=19mm, height = 16mm]{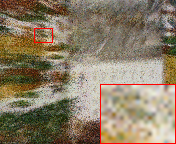}&
\includegraphics[width=19mm, height = 16mm]{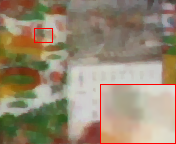}&
\includegraphics[width=19mm, height = 16mm]{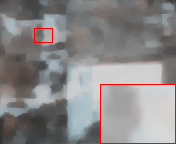}&
\includegraphics[width=19mm, height = 16mm]{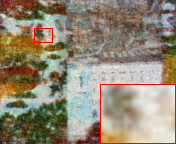}&
\includegraphics[width=19mm, height = 16mm]{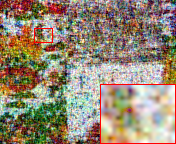}&
\includegraphics[width=19mm, height = 16mm]{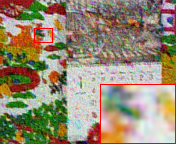}&
\includegraphics[width=19mm, height = 16mm]{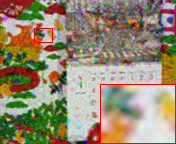} \\
\scriptsize{Mobile} & \scriptsize{6.04/0.054} & \scriptsize{15.70/0.382} & \scriptsize{17.27/0.393} & \scriptsize{15.30/0.327} & \scriptsize{19.35/0.663} & \scriptsize{16.25/0.427} & \scriptsize{\underline{21.91/0.807}} & \scriptsize{\textbf{22.01/0.822}} \\
\end{tabular}
\caption{Recovered pseudo-colored images under 70\% sparse noise. The PSNR and SSIM values are placed below the images, and the best and second-best results are bolded and underlined respectively.}
\label{rpca_case}
\end{figure*}

\begin{figure*}[!h]
\renewcommand{\arraystretch}{0.5}
\setlength\tabcolsep{0.5pt}
\centering
\begin{tabular}{ccccccccc}
\centering
\scriptsize{Original} & \scriptsize{Noisy} & \scriptsize{NN} & \scriptsize{KBR} & \scriptsize{TNN} & \scriptsize{CTV} & \scriptsize{TCTV} & \scriptsize{MNN-Sobel} & \scriptsize{MNN-L2} \\
\includegraphics[width=19mm, height = 19mm]{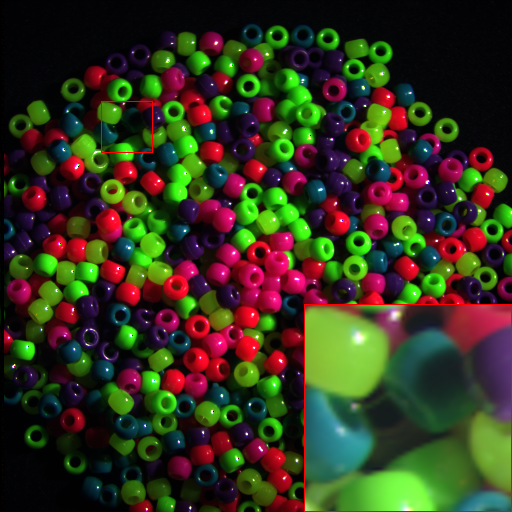}&
\includegraphics[width=19mm, height = 19mm]{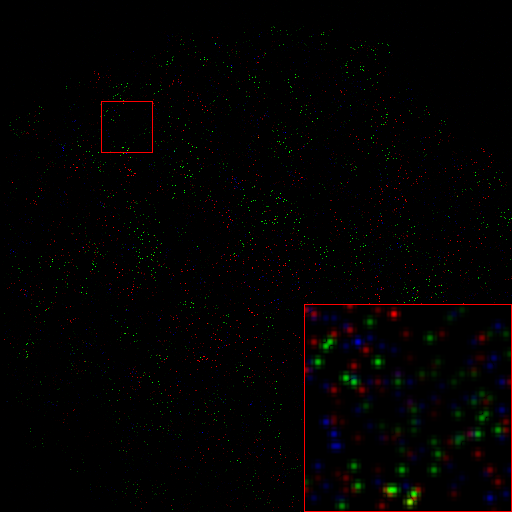}&
\includegraphics[width=19mm, height = 19mm]{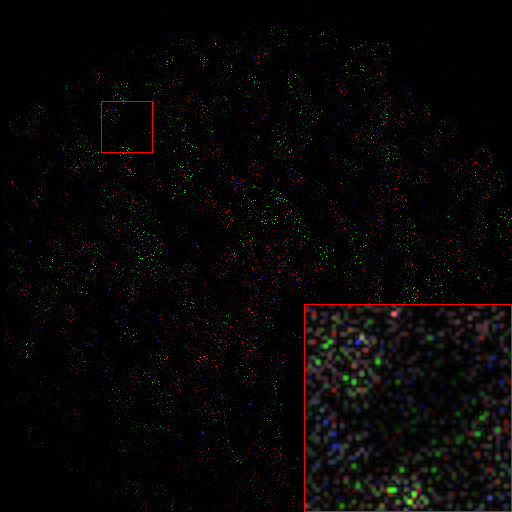}&
\includegraphics[width=19mm, height = 19mm]{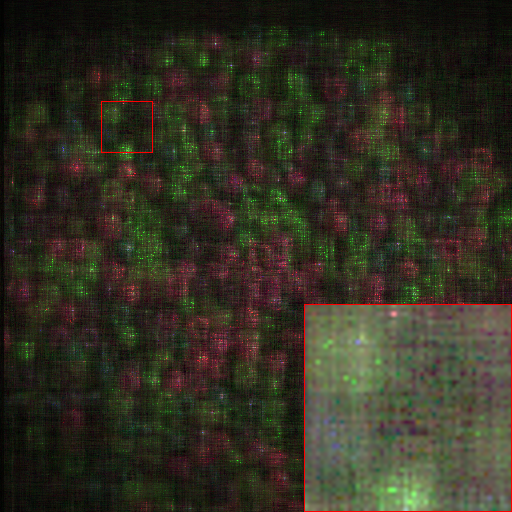}&
\includegraphics[width=19mm, height = 19mm]{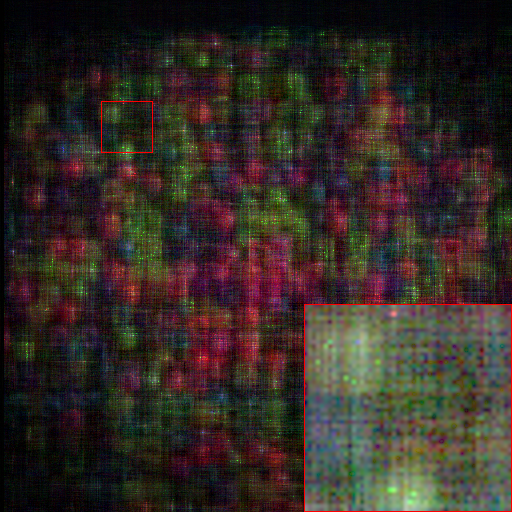}&
\includegraphics[width=19mm, height = 19mm]{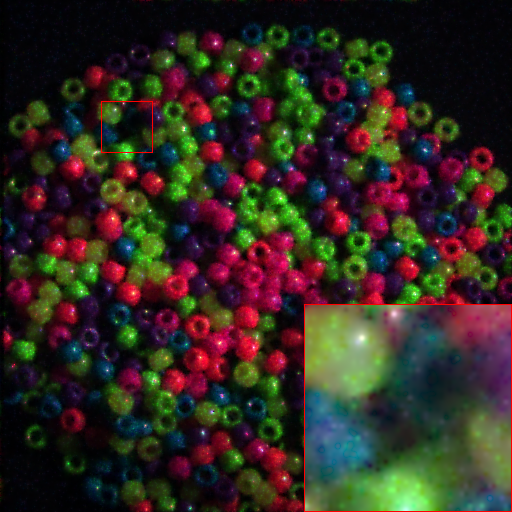}&
\includegraphics[width=19mm, height = 19mm]{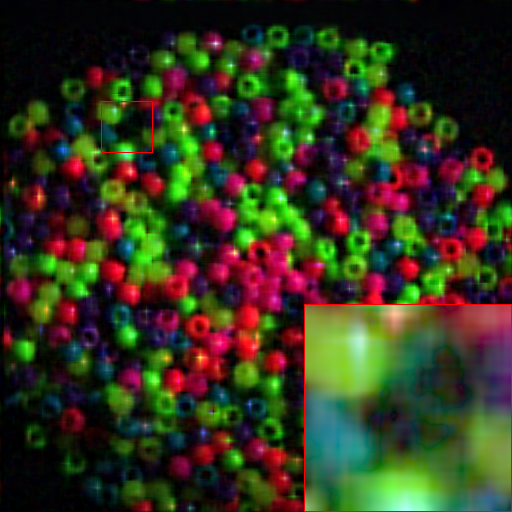}&
\includegraphics[width=19mm, height = 19mm]{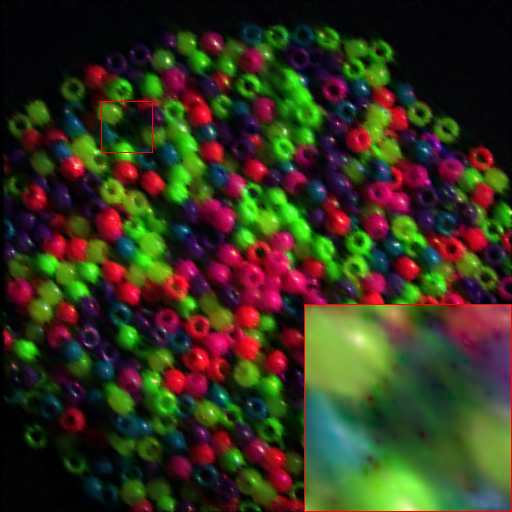}&
\includegraphics[width=19mm, height = 19mm]{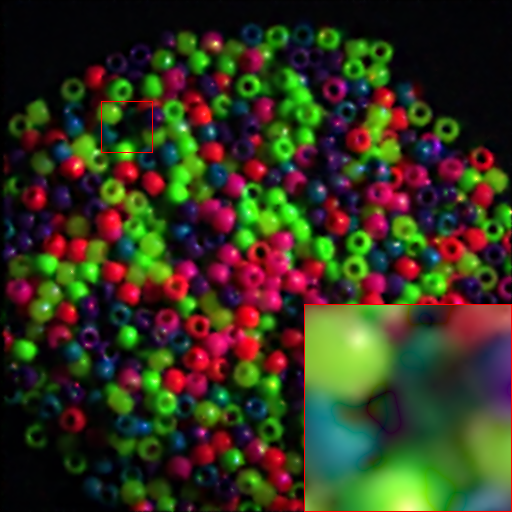} \\
\scriptsize{beads} & \scriptsize{13.92/0.081} & \scriptsize{14.35/0.167} & \scriptsize{18.17/0.403} & \scriptsize{18.39/0.426} & \scriptsize{24.21/0.829} & \scriptsize{23.82/0.766} &
\scriptsize{\underline{24.62/0.851}}& \scriptsize{\textbf{25.13/0.867}} \\
\includegraphics[width=19mm, height = 19mm]{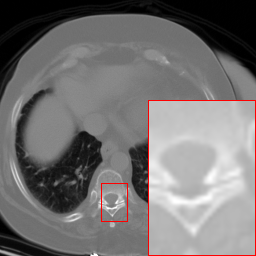}&
\includegraphics[width=19mm, height = 19mm]{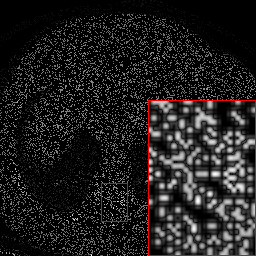}&
\includegraphics[width=19mm, height = 19mm]{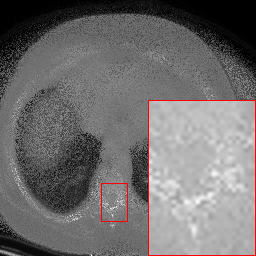}&
\includegraphics[width=19mm, height = 19mm]{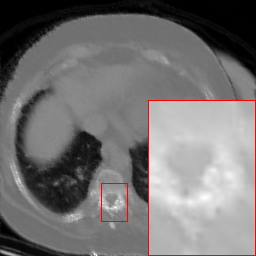}&
\includegraphics[width=19mm, height = 19mm]{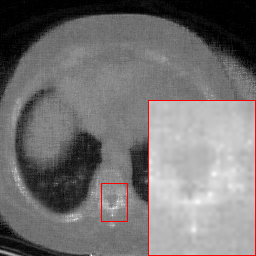}&
\includegraphics[width=19mm, height = 19mm]{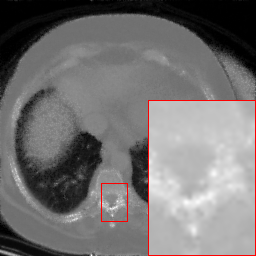}&
\includegraphics[width=19mm, height = 19mm]{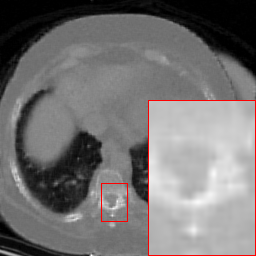}&
\includegraphics[width=19mm, height = 19mm]{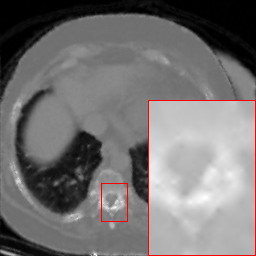}&
\includegraphics[width=19mm, height = 19mm]{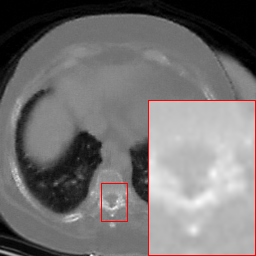} \\ \\
\scriptsize{chest-pet} & \scriptsize{9.28/0.196} & \scriptsize{23.13/0.520} & \scriptsize{33.86/0.934} & \scriptsize{28.62/0.690} & \scriptsize{31.93/0.898} & \scriptsize{34.02/0.903} & \scriptsize{\underline{34.25/0.957}} & \scriptsize{\textbf{34.57/0.962}} \\
\end{tabular}
\caption{Restored pseudo-colored images under sampling rates of 0.02 (first row) for beans data (R-G-B:23-13-4) and 0.2 (second row) for chest-pet dataset at band 80.}\label{mc_case}
\end{figure*}

\begin{table*}[htp]
\setlength\tabcolsep{4.5pt}
\fontsize{10}{10}\selectfont
  \centering
  \caption{Restoration comparison of all competing methods under different sampling ratio (SR). The best and second results are highlighted in bold italics and \underline{underlined}.}
\begin{tabular}{l|c|c|c|c|c|c|c|c|c|c|c|c|c}
     \Xhline{2\arrayrulewidth}
     SR & Metrics & NN & SNN & KBR & TNN &Framlet & Qrank & CTV & TCTV & \makecell[c]{MNN-\\Diff} &  \makecell[c]{MNN-\\Sobel} & \makecell[c]{MNN-\\L1} & \makecell[c]{MNN-\\L2}  \\
     \Xhline{2\arrayrulewidth}
		\multicolumn{13}{c}{Hyperspectral Image Completion: Five Datasets} \cr
\hline
\multirow{2}{*}{5\%} & PSNR$\uparrow$ & 21.85& 20.20& 30.28& 27.41& 29.26& 28.38& \underline{35.29}& 29.48& 29.29& \textbf{36.03}& 32.04& 32.57 \\
& SSIM$\uparrow$ & 0.625& 0.428& 0.871& 0.771& 0.825& 0.803& \underline{0.972}& 0.862& 0.913& \textbf{0.981}& 0.957& 0.961 \\
\hline
\multirow{2}{*}{10\%} & PSNR$\uparrow$ & 33.98& 21.91& \underline{42.60}& 31.01& 32.28& 28.06& 41.90& 33.54& 37.10& \textbf{43.09}& 39.34& 40.46  \\
& SSIM$\uparrow$ & 0.909& 0.528& 0.989& 0.876& 0.884& 0.861& \underline{0.992}& 0.935& 0.982& \textbf{0.995}& 0.990& 0.992 \\
\hline
\multicolumn{2}{c|}{Average Time/s} & 11.30& 35.23& 258.6& 143.3& 222.3& 112.1& 213.2& 474.5& 41.30& 42.98& 41.97& 41.85  \\
	  \hline
	  \hline
	  \multicolumn{13}{c}{Multispectral Image Completion: Eleven Datasets} \cr
\hline
\multirow{2}{*}{5\%} & PSNR$\uparrow$ & 16.71& 27.56& 26.49& 32.58& 31.61& 30.23& 35.71& \underline{36.12}& 30.64& \textbf{36.98}& 36.10& 36.05  \\
& SSIM$\uparrow$  & 0.583& 0.820& 0.888& 0.920& 0.912& 0.889& 0.963& 0.973& 0.921& \textbf{0.978}& \underline{0.975}& \underline{0.975} \\
\hline
\multirow{2}{*}{10\%} & PSNR$\uparrow$ & 22.41& 30.84& 35.45& 36.29& 35.48& 34.72& 39.89& \underline{40.09}& 35.42& \textbf{41.11}& 40.32& 40.05 \\
& SSIM$\uparrow$ & 0.793& 0.892& 0.980& 0.963& 0.958& 0.952& 0.983& 0.988& 0.971& \textbf{0.992}& \underline{0.991}& \underline{0.991} \\
\hline
\multicolumn{2}{c|}{Average Time/s} &15.4& 77.73& 431.5& 275.4& 512.2& 332.1& 303.1& 800.6& 31.70& 32.36& 31.99& 31.49  \\
	  \hline
	  \hline
	  \multicolumn{13}{c}{RGB Video Completion: Ten Datasets} \cr
\hline
\multirow{2}{*}{5\%} & PSNR$\uparrow$ &  23.01& 17.72& 26.90& 28.14& 26.53& 25.35& \underline{29.43}& 28.70& 25.76& \textbf{29.53}& 27.95& 28.10 \\
& SSIM$\uparrow$ & 0.605& 0.484& 0.802& 0.801& 0.790& 0.751& 0.890& 0.856& 0.826& \textbf{0.909}& 0.892& \underline{0.894} \\
\hline
\multirow{2}{*}{10\%} & PSNR$\uparrow$ & 29.06& 20.48& \underline{32.81}& 30.54& 31.55& 29.36& \textbf{33.18}& 31.89& 29.32& 32.75& 31.61& 31.72 \\
& SSIM$\uparrow$ &  0.846& 0.589& 0.917& 0.861& 0.893& 0.879& \underline{0.937}& 0.910& 0.898& \textbf{0.945}& 0.935& 0.936\\
\hline
\multicolumn{2}{c|}{Average Time/s} & 20.51& 58.22& 401.8& 214.8& 449.4& 168.6& 426.2& 685.1& 67.44& 68.16& 64.92& 64.83   \\
	  \hline
	  \hline
	  \multicolumn{13}{c}{MRI-CT Completion: Four Datasets} \cr
\hline
\multirow{2}{*}{5\%} & PSNR$\uparrow$ & 14.14& 17.40& 21.58& 20.65& 19.82& 19.94& 23.70& 24.80& 20.48& 24.45& \underline{24.87}& \textbf{24.94} \\
& SSIM$\uparrow$ & 0.204& 0.448& 0.532& 0.368& 0.506& 0.380& 0.688& 0.647& 0.625& 0.801& \underline{0.810}& \textbf{0.813} \\
\hline
\multirow{2}{*}{10\%} & PSNR$\uparrow$ & 18.36& 19.89& 25.68& 22.37& 22.71& 23.94& 25.48& 26.73& 22.75& 26.34& \textbf{27.04}& \underline{26.98}  \\
& SSIM$\uparrow$ & 0.369& 0.535& 0.704& 0.457& 0.621& 0.699& 0.761& 0.721& 0.737& 0.858& \underline{0.869}& \textbf{0.870} \\
\hline
\multicolumn{2}{c|}{Average Time/s} &  20.07& 69.32& 484.0& 316.8& 537.8& 235.2& 424.7& 1009& 60.81& 62.13& 60.04& 60.01 \\
	  \hline
	  \hline
	  \Xhline{1\arrayrulewidth}
\end{tabular}
\label{mc_real}
\end{table*}
	
\textbf{Empirical analysis of convergence}. According to Corollary \ref{coro_mnn}, the objectives of models (\ref{mnn_rpca}) and (\ref{mnn_mc}) reach their minima at $\X_0$ and $\S_0$ under Assumptions \ref{assumption1}–\ref{assumption3}. Even when the assumptions are only partially met, the objective values still decrease, and NMSE sequence $\Vert \X_k -\X_0\Vert_F/\Vert\X_0\Vert_F$ steadily declines. Figure \ref{mnn_ite} shows that under a learning rate of $ 1e^{-4} $, NN and all MNN variants converge stably.

\section{Real Application}
\label{real_exp}
We then assess the restoration performance on real datasets.

\textbf{Datasets}. We selected four commonly used low-rank image datasets, which are HSI data used in \cite{wang2023guaranteed}, MSI \footnote{\url{https://www.cs.columbia.edu/CAVE/databases/multispectral/}}, color video sequences \footnote{\url{http://trace.eas.asu.edu/yuv/}}, and MRI and CT images \footnote{\url{https://www.cancerimagingarchive.net/}}. Among them, hyperspectral images, multispectral images, color video sequences, and MRI and CT images contain 5, 11, 10, and 4 images, respectively. We adopt the same noise addition mechanism and sampling scheme as in the simulation experiments section. 

\textbf{Comparison methods}. To validate the effectiveness of MNN in fusing global and local priors, we selected methods such as LRTV \cite{he2015total}, LRTDTV \cite{wang2017hyperspectral}, CTV \cite{peng2022exact}, Qrank \cite{}, KBR \cite{xie2017kronecker}, TCTV \cite{wang2023guaranteed}, SNN \cite{liu2012tensor}, Qrank \cite{kong2021tensor}, TNN \cite{lu2019tensor} and Framlet \cite{jiang2020framelet}. It is worth noting that CTV can actually be regarded as a special case of MNN since it utilizes multiple differential operators from different methods and employs the ADMM \cite{boyd2011distributed} algorithm to solve the model. 
To facilitate readers' understanding of these comparison methods, the description of these methods are placed in Table \ref{tab:methods_comparison}.

\subsection{Robust Principal Component Analysis Tasks} 
We evaluate denoising performance on HSI, MSI, and RGB video data with salt-and-pepper noise levels ranging from 10\% to 70\%, using PSNR and SSIM as metrics. Table \ref{rpca_real} summarizes the average results. The following observations can be observed:  1) In most cases, the MNN-Sobel model achieves the best performance; 2) Models such as MNN-Sobel, MNN-L1, and MNN-L2 generally outperform MNN-Diff, partially supporting Remark \ref{remark_1}, which suggests that expanding the receptive field of convolutional operators enhances performance; 3) In some scenarios, MNN-L1 and MNN-L2 outperform MNN-Sobel, indicating that different transformation operators are required for different types of data and noise. If suitable transformation operators can be adaptively selected, further improvements in restoration performance are achievable.  

Figure \ref{rpca_case} shows restored pseudo-color images for two datasets, where MNN-Sobel and MNN-L2 visibly enhance noise reduction and color fidelity. More results are provided in the supplementary materials.  

\subsection{Completion Tasks} 

We conduct completion experiments on four datasets using random sampling. Table \ref{mc_real} reports the average restoration metrics, leading to the same three conclusions as in the denoising results (Table \ref{rpca_real} ). Notably, despite being matrix-based, MNN consistently outperforms tensor-based methods like TNN, TCTV, and KBR, demonstrating its ability to effectively capture both global and local features.

Figure \ref{mc_case} shows visual results on two datasets, where MNN-based models excel in structure preservation, color fidelity, and overall restoration. Additional results are provided in the supplementary materials.

\section{Conclusion}
\label{cons}
This paper proposes a modified nuclear norm (MNN) regularizer that integrates global low-rankness and local priors via transformation operators. Under mild and general theoretical conditions, MNN accommodates a broad range of transformations, enabling flexible and effective fusion of local and global information within a unified framework.

\section*{Appendix}
{\appendices
\section{Verification of the MNN}
\label{norm_def}
Determining whether a definition constitutes a norm, mainly involves verifying its non-negativity, homogeneity, and triangle inequality properties.

It is known that the nuclear norm satisfies the aforementioned three properties of a norm, then we have 
\begin{enumerate}
  \item Non-negativity: $\Vert \X \Vert_* \geq 0$, and $\Vert \X \Vert_* = 0$ if and only if $\X = \mathbf{0}$.
  \item Homogeneity: $\Vert \alpha \X \Vert_* = |\alpha| \cdot \Vert \X \Vert_*$ for any $\alpha$.
  \item Triangle Inequality: $\Vert \X + \Y \Vert_* \leq \Vert \X \Vert_* + \Vert \Y \Vert_*$.
\end{enumerate}

Since MNN of $ \X $ is $ \Vert \mathcal{D}(\X) \Vert_* $ with the fixed normalized linear transformation operator, we can easily obtain the following facts:
\begin{enumerate}
  \item Non-negativity: $\Vert \mathcal{D}(\X) \Vert_* \geq 0$, and $\Vert \X \Vert_* = 0$ if and only if $\X = \mathbf{0}$.
  \item Homogeneity: $\Vert \alpha \mathcal{D} (\X)\Vert_* = |\alpha| \cdot \Vert \mathcal{D}(\X) \Vert_*$ for any $\alpha$.
  \item Triangle Inequality: $\Vert \mathcal{D}(\X + \Y) \Vert_* = \Vert \mathcal{D}(\X) + \mathcal{D}(\Y) \Vert_* \leq \Vert \mathcal{D}(\X) \Vert_* + \Vert \mathcal{D}(\Y) \Vert_*$.
\end{enumerate}
Therefore, the MNN definition is a well-defined norm.

\section{The Proof of MNN-RPCA Theorem}
\label{rpca_proof}
For ease of exposition, let us first provide the equivalent model of MNN-RPCA.
\subsection{The Equivalent Model of MNN-RPCA Model}
Since $\mathcal{D}(\cdot)$ is a linear transformation operator, $\mathcal{D}(\cdot)$ can be rewritten as $\A\X$, that is, the role of $\mathcal{D}(\cdot)$ can be expressed into a corresponding matrix $\A$. For example, the difference operator $\nabla(\cdot)$ on row of matrix $\X\in \mathbb{R}^{n_1\times n_2}$ can be rewritten as:
\begin{equation}
\nabla(\X)= \A \X, \A=\mbox{\textbf{circ}}(-1,1,\overbrace{0,\dots,0}^{n_1-2})
\end{equation}
where ``$ \mbox{\textbf{circ}} $"  denotes the circulant matrix. Thus, the MNN-RPCA model (\ref{mnn_rpca}) can be rewritten as
\begin{equation}
\label{rpca_v}
\min_{\X,\S} \quad \|\A\X \|_* + \lambda \|\S\|_1, s.t. \ \M = \X+\S.
\end{equation}

\subsection{Mathematical Prerequisites}
Before proving Theorem \ref{theorem_main3}, it is helpful to review some basic concepts and introduce certain symbols.

Assuming we have a large data matrix $ \M $ and know that it can be decomposed as:
\begin{equation}
\M = \X_0 + \S_0,
\end{equation}
where $ \X_0 $ and $ \S_0 $ are the low-rank and sparse components, respectively. Theorem \ref{theorem_main3} asserts that by solving a variant of the MNN-RPCA model (\ref{rpca_v}), we can obtain an exact decomposition ($ \X_0, \S_0 $). 

For a given scalar $ x $, we use $ \mbox{sgn}(x) $ to denote the sign of $ x $. Extended to the matrix case, $\mbox{sgn}(\mathbf{S}) $ is a matrix whose elements represent the signs of the elements of $ \mathbf{S} $. Recalling the subdifferential of the $ \ell_1 $ norm at $ \mathbf{S_0} $, it takes the form on the support set $ \Omega $:
\begin{equation*}
\mbox{sgn}(\mathbf{S}_0) + \mathbf{F},
\end{equation*}
where $ \mathbf{F} $ is zero on $ \Omega $, i.e., $ \mathcal{P}_{\Omega}\mathbf{F} = 0 $, and satisfies $ |\vert \mathbf{F}|\vert_{ \infty } \leq 1 $.

Next, we assume that $ \X_0 $ with rank $ r $ has a singular value decomposition $ \U \mathbf{\Sigma} \V^T $, where $ \U \in \mathbb{R}^{n_1\times r} $ and $ \V \in \mathbb{R}^{n_2 \times r} $. Then, according to the chain rule of derivatives, we obtain
\begin{equation}
\frac{\partial  \| \A \X_0  \|_*}{\partial \X_0} = \A^T\U\V^T + \A^T\W,
\end{equation}
where $ \U^T \W = \mathbf{0} $, $ \W\V= \mathbf{0} $, and $ \| \W \| \leq 1 $. We use $ T $ to denote the linear space of matrices, i.e.,
\begin{equation*}
\label{space}
T:= \lbrace \U\X^T + \Y\V^T, \X\in \mathbb{R}^{n_1\times r},\Y\in \mathbb{R}^{n_2 \times r} \rbrace,
\end{equation*}
and use $ T^\bot $ to represent its orthogonal complement. For any matrix $ \M $, the projections onto $ T $ and $ T^\bot $ are given by
\begin{equation}
\begin{split}
& \mathcal{P}_{T} \mathbf{M} = \U\U^T\M +\M\V\V^T- \U\U^T\M \V\V^T, \\
& \mathcal{P}_{T^{\bot}} \mathbf{M} = (\mathbf{I}-\U\U^T)\mathbf{M}(\mathbf{I}-\V\V^T),\\
\end{split}
\end{equation}
where $ \mathbf{I} $ denotes the identity matrix.

Therefore, for any matrix of the form $ \hat{\mathbf{e}}_k\hat{\mathbf{e}}_j^T $, it can be easily observed that
\begin{equation*}
\begin{split}
\|\mathcal{P}_{T^{\bot}}  \hat{\mathbf{e}}_k \hat{\mathbf{e}}_j^T\|_F^2 &= \|(\mathbf{I}-\U\U^T)\hat{\mathbf{e}}_k\|^2 \|(\mathbf{I}-\V\V^T)\hat{\mathbf{e}}_j\|^2 \\
& \geq \left(1-\frac{\mu r}{n_{1}} \right)\left(1-\frac{\mu r}{n_{2}} \right).
\end{split}
\end{equation*}

Assuming $\mu r / n_{(1)} \leq 1 $, where $  n_{(1)}=\max\{n_1,n_2\}$ and $n_{(2)} =\min\{n_1,n_2\} $. Because $ \|\mathcal{P}_{T^{\bot}} \hat{\mathbf{e}}_k \hat{\mathbf{e}}_j^T\|_F^2 + \|\mathcal{P}_{T}\hat{\mathbf{e}}_k \hat{\mathbf{e}}_j^T\|_F^2 = 1$, we can derive
\begin{equation}
\label{coherence_inequality}
\|\mathcal{P}_{T} \hat{\mathbf{e}}_k \hat{\mathbf{e}}_j^T\|_F \leq \sqrt{\frac{2\mu r}{ n_{(2)}}}.
\end{equation}
\subsection{Elimination Theorem}
We begin with a useful definition.
\begin{definition}
If for any $ \mathbf{S}_{ij}^{'} \neq 0 $, $ \mathbf{S}^{'} $ satisfies $ supp(\mathbf{S}^{'}) \subset supp({\mathbf{S}})$ and $ \mathbf{S}_{ij}^{'} = \mathbf{S}_{ij} $, we call $ \mathbf{S}^{'} $ a trimmed version of $ \mathbf{S} $.
\end{definition}
In other words, a trimmed version of $ \mathbf{S} $ is obtained by setting some elements of $ \mathbf{S} $ to zero. Next, the following elimination theorem asserts that if the solution in Eq. (\ref{rpca_v}) recovers the low-rank and sparse components of $ \mathbf{M}_0 = \mathbf{\X}_0 + \mathbf{S}_0 $, then it can also correctly recover the components of $ \mathbf{M}_0^{'}=\mathbf{L}_0 + \mathbf{S}_0^{'}$, where $ \mathbf{S}_0^{'} $ is a trimmed version of $ \mathbf{S}_0 $.
\begin{lemma}
\label{elimination_theroem_dip}
Assume that the solution of Eq. (\ref{rpca_v}) is exact when the input is $ \M_0=  \X_0 + \S_0 $, and consider $ \M_0^{'} =  \X_0 + \S_0^{'} $, where $ \S_0^{'} $ is a trimmed version of $ \S_0 $. Then, the solution of Eq. (\ref{rpca_v}) is also exact when the input is $ \M_0^{'} =  \X_0 + \S_0^{'} $.
\end{lemma}

\begin{proof}
First, for some $ \Omega_0 \subset [n] \times [n] $, we express $ \S_0^{'}$ as $ \S_0^{'} = \mathbf{P}_{\Omega_0}{\S_0} $, and let $ (\widehat{\X}, \widehat{\S}) $ be the solution of Equation (\ref{rpca_v}) when the input is $ \M_0^{'} = \X_0 + \S_0^{'} $. Then, we have
\begin{equation*}
\|\A\widehat{\X}\|_* + \lambda \|\widehat{\S}\|_1 
\leq \|\A\X_0\|_*  + \lambda \|\mathbf{P}_{\Omega_0}{\mathbf{S}_0}\|_1.
\end{equation*}
Further, we can obtain
\begin{equation*}
\|\A\widehat{\X}\|_* +  \lambda \left( \|\widehat{\S}\|_1+  \|\mathbf{P}_{\Omega_0^{\bot}}{\mathbf{S}_0}\|_1 \right) \leq \|\A\X_0\|_*  + \lambda \|\S_0\|_1.
\end{equation*}
Noting that $ ( \widehat{\X}, \widehat{\S} + \mathbf{P}_{\Omega_0^{\bot}}{\mathbf{S}_0}) $ is a feasible solution of Eq. (\ref{rpca_v}) when the measurement matrix satisfies the condition $ \M_0= \X_0 + \S_0 $, and $ \Vert \widehat{\S} + \mathbf{P}_{\Omega_0^{\bot}}{\mathbf{S}_0} \Vert_1 \leq \Vert \widehat{\S} \Vert_1 + \Vert \mathbf{P}_{\Omega_0^{\bot}}{\mathbf{S}_0} \Vert_1$, we obtain
\begin{equation*}
\begin{split}
& \|\A \widehat{\X}\|_* + \lambda \left( \|\widehat{\S}+\mathbf{P}_{\Omega_0^{\bot}}{\S_0}\|_1 \right) \\
& \quad \leq \|\A \widehat{\X}\|_* + \lambda \left( \|\widehat{\S}\|_1+  \|\mathbf{P}_{\Omega_0^{\bot}}{\S_0}\|_1 \right) \\
& \quad \leq  \|\A\X_0\|_*  + \lambda \|\S_0\|_1.
\end{split}
\end{equation*}
However, the above right-hand side is the optimal value, and by the uniqueness of the optimal solution, we must have $ \A\widehat{\X} = \A\X_0 $ and $ \widehat{\S} + \mathbf{P}_{\Omega_0^{\bot}}{\S_0} = \S_0 $. Therefore, we further obtain $ \widehat{\X} = \X_0 $ and $ \widehat{\S} =  \mathbf{P}_ {\Omega_0}{\S_0} = \S^{'}$. This proves the lemma.
\end{proof}

\subsection{Dual Verification}
\begin{lemma}
\label{Dual_certifications_dip}
Suppose $ \Vert \mathcal{P}_{\Omega} \mathcal{P}_{T} \Vert \leq 1$. If there exists a pair $ ( \W, \mathbf{F})$ satisfying the following conditions:
\begin{equation}
\label{kkt}
\A^T\U\V^T + \A^T\W =  \lambda (\mbox{sgn}(\S_0)+\mathbf{F})
\end{equation}
where $\mathcal{P}_{T}(\W)=0$, $\|\W\| <1$, and $\mathcal{P}_{\Omega} {(\mathbf{F})}=0$, $\|\mathbf{F}\|_\infty < 1$, then $ \left( \X_0, \S_0 \right) $ is the unique solution of Eq. (\ref{rpca_v}).
\end{lemma}

\begin{lemma}
\label{Dual_certifications_Relaxtion_dip}
Assuming $ \Vert \mathcal{P}_{\Omega} \mathcal{P}_{T}\Vert \leq 1/2$. If there exists $ (\W, \mathbf{F})$ satisfying
\begin{equation}
\label{kkt_2}
\A^T\U\V^T + \A^T\W =  \lambda (\mbox{sgn}(\S_0)+\mathbf{F}+\mathcal{P}_\Omega {\mathbf{D}})
\end{equation}
where $ \mathcal{P}_{T}(\W)=0$, $\|\W\|<\frac{1}{2}$, $\|\mathbf{F}\|_\infty < \frac{1}{2}$, and $ \|\mathcal{P}_\Omega {\mathbf{D}}\|_\infty \leq 1/4$, then $ \left( \X_0, \S_0 \right) $ is the optimal solution of Eq. (\ref{rpca_v}).
\end{lemma}
The above lemma indicates that, to prove our final exact recovery result, it suffices to generate dual variables $\W$ satisfying the following conditions:
\begin{equation}
\label{Certification}
\left\{
             \begin{array}{lc}
             \ \W \in T^{\bot}, &  \\
             \| \W\| <\frac{1}{2}, &  \\
             \| \mathcal{P}_\Omega \left( \A^T\U\V^T +\A^T\W - \lambda \mbox{sgn}(\S_0) \right)\|_F \leq \lambda \eta,&\\
             \|\mathcal{P}_{\Omega^{\bot}}\left( \A^T\U\V^T +\A^T\W  \right)\|_\infty < \frac{\lambda}{2}.
             \end{array}
\right.
\end{equation}

\subsection{Construct dual variables based on the Golfing Scheme}
\label{golf_scheme_dip}
The remaining task is to construct dual variables that satisfy the corresponding conditions mentioned above. Before introducing our construction method, we first assume that $ \Omega\sim \mbox{Ber}\left( \rho\right) $, or equivalently, $ \Omega^{c}\sim\mbox{Ber}\left( 1-\rho\right) $. Naturally, the distribution of $\Omega^{c} $ is identical to the distribution of $ \Omega^{c} = \Omega_{1} \cup \Omega_{2} \cup \dots \cup\Omega_{j_{0}} $, where each $ \Omega_{j} $ follows a Bernoulli model with parameter $ q $, i.e.,
\begin{equation*}
\mathbb{P}\left( (i,j) \in \Omega \right) = \mathbb{P}(\mbox{Bin}(j^0,q)=0)=(1-q)^{j_0}.
\end{equation*}
Thus, if the following condition is satisfied:
\begin{equation*}
\rho = \left( 1-q\right)^{j_0},
\end{equation*}
the two models are equivalent. We decompose $\W$ as
\begin{equation*}
\mathbf{W} = \mathbf{W}^L + \mathbf{W}^S,
\end{equation*}
where each component can be constructed as follows:

\textbf{Constructing $ \mathbf{W}^L $ based on the Golfing scheme.} Suppose $ j_0 \geqslant 1 $, and let $ \Omega_j, 1 \leq j \leq j_0 $ be such that $ \Omega^{c} = \cup_ {1 \leq j \leq j_0} \Omega_{j}$. We then define
\begin{equation}
\label{Low rank golf dip}
\mathbf{W}^L = \mathcal{P}_{T^{\bot}} \mathbf{Y}_{j_0},
\end{equation}
where
\begin{equation}
\label{Y_j}
\mathbf{Y}_j = \mathbf{Y}_{j-1} + q^{-1}\mathcal{P}_{\Omega_j}\mathcal{P}_{T} \left(\U_{1}\V_{1}^T - \mathbf{Y}_{j-1} \right), \quad \mathbf{Y}_0 = 0.
\end{equation}

\textbf{Constructing $ \mathbf{W}^S $ via Least Squares.} Suppose $ \|\mathcal{P}_\Omega \mathcal{P}_{T}\| \leq \frac{1}{2} $. Then, $ \|\mathcal{P}_\Omega \mathcal{P}_{T} \mathcal{P}_\Omega \| < \frac{1}{4} $, and thus, the operator $ \mathcal{P}_\Omega- \mathcal{P}_\Omega \mathcal{P}_{T} \mathcal{P}_\Omega $, which maps $ \Omega $ to itself, is invertible with the inverse operator denoted as $ \left( \mathcal{P}_\Omega- \mathcal{P}_\Omega \mathcal{P}_{T} \mathcal{P}_\Omega \right)^{-1} $. We then set
\begin{equation}
\label{sparse golf dip}
\mathbf{W}^S = \lambda \mathcal{P}_{T^{\bot}} \left( \mathcal{P}_\Omega- \mathcal{P}_\Omega \mathcal{P}_{T} \mathcal{P}_\Omega \right)^{-1} \left(\mbox{sgn}(\mathbf{S_0}) \right).
\end{equation}
Eq. (\ref{sparse golf dip}) is equivalent to
\begin{equation}
\mathbf{W}^S = \lambda \mathcal{P}_{T^{\bot}}\sum_{k\geqslant 0}{\left(\mathcal{P}_\Omega \mathcal{P}_{T} \mathcal{P}_\Omega \right)^k } \left(\mbox{sgn}(\mathbf{S})\right).
\end{equation}
Since both $ \mathbf{W}^L $ and $ \mathbf{W}^S $ belong to $ T^{\bot} $, and $  \mathcal{P}_{\Omega} \W^S = \lambda \mathcal{P}_{\Omega} (\mathcal{I}-\mathcal{P}_{T}) \left( \mathcal{P}_\Omega- \mathcal{P}_\Omega \mathcal{P}_{T} \mathcal{P}_\Omega \right)^{-1} \left(\mbox{sgn}(\mathbf{S_0}) \right)= \lambda \mbox{sgn}(\S_0) $, we only need to show that if $ \W^L + \W^S $ satisfies the following conditions:
\begin{equation}
\label{Final Dual Certificate dip}
\left\{
             \begin{array}{lc}
             \| \W^L + \W^S \| <\frac{1}{2}, &  \\
             \|\mathcal{P}_\Omega \left( \A^T\U \V^T + \A^T\W^L \right) \|_F \leq \lambda\eta,&\\
              \|\mathcal{P}_{\Omega^{\bot}} \left( \A^T(\U \V^T+\W^L + \W^S) \right)\|_{\infty} \leq \frac{\lambda}{2}.
             \end{array}
\right.
\end{equation}
Then it is an effective dual certificate. The proof of Eq. (\ref{Final Dual Certificate dip}) can be ensured through the following two lemmas.
\begin{lemma}
\label{low rank certificate dip}
Assuming $ \Omega \sim \mbox{Ber}(\rho) $, where $ \rho \leq \rho_s $ for some $ \rho_s >0 $. Set $ j_0 = 2 \lceil {\rm log} n\rceil $ (for rectangular matrices, use $ {\rm log} n_{(1)}$). Then, the $ \mathbf{W}^L $ in Eq. (\ref{Low rank golf dip}) satisfies the following conditions:
\begin{itemize}
\item[(a)] $ \| \mathbf{W}^L\| < 1/4, $
\item[(b)] $ \|\mathcal{P}_\Omega \left( \A^T\U \V^T +\A^T\mathbf{W}^L \right)\|_F \leq \lambda\eta,$
\item[(c)] $ \|\mathcal{P}_{\Omega^\bot}\left( \A^T\U \V^T +\A^T\mathbf{W}^L\right) \|_\infty <\lambda/4.$
\end{itemize}
\end{lemma}

\begin{lemma}
\label{sparse certificate dip}
Assume $ \Omega\sim \mbox{Ber}(\rho_s) $, and the signs of $ \S_0 $ are independent and identically distributed symmetric random variables (independent of $ \Omega $). Then, the matrix $ \mathbf{W}^S $ in Eq. (\ref{sparse golf dip}) satisfies the following conditions:
\begin{itemize}
\item[(a)] $ \|\mathbf{W}^S\| < 1/4$,
\item[(b)] $ \| \mathcal{P}_{\Omega^{\bot}} (\A^T\mathbf{W}^S) \|_\infty < \lambda/4$.
\end{itemize}
\end{lemma}

\subsection{Proof of Dual Verification}
Before proving Lemma \ref{low rank certificate dip}, we need to invoke the following three lemmas.

\begin{lemma}
\label{T_condition}
[Lemma 4.1 in \cite{candes2012exact}]:
Assuming $ \Omega_0 \sim \mbox{Ber}(\rho_0) $. For some $ C_{0} >0$, when $ \rho_0 \geq C_{0}\epsilon^{-2} \beta \mu r \log n_{(1)} /n_{(2)}$, the following equation 
\begin{equation}
\|\mathcal{P}_{T} -\rho_0^{-1}\mathcal{P}_{T}\mathcal{P}_{\Omega_0}\mathcal{P}_{T}\| \leq \epsilon,
\end{equation}
hold with a high probability.
\end{lemma}

\begin{lemma}
\label{Z_condition}[Lemma 3.1 in \cite{candes2011robust}]: Suppose $ \Z \in T$ is a fixed matrix, and $ \Omega_0 \sim \mbox{Ber}(\rho_0) $, when $ \rho_0 \geq C_{0}\epsilon^{-2} \mu r \log n_{(1)} /n_{(2)} $, 
\begin{equation}
\| \Z -\rho_0^{-1}\mathcal{P}_{T_{1}}\mathcal{P}_{\Omega_0}\Z \|_{\infty} \leq \epsilon \| \Z \|_{\infty},
\end{equation}
holds with a high probability.
\end{lemma}

\begin{lemma}
\label{Z_infty}
[Lemma 3.2 in \cite{candes2011robust} and Lemma 6.3 in \cite{candes2012exact}]: Suppose $ \Z $ is a fixed matrix, and $ \Omega_0 \sim \mbox{Ber}(\rho_0) $. For some constant $ C_{1} >0 $, when $ \rho_0 \geq C_1 \mu \log n_{(1)}/n_{(2)} $,
 \begin{equation}
\| \left( \mathbf{I}- \rho_0^{-1} \mathcal{P}_{\Omega_0}\right) \Z \| \leq C_{1} \sqrt{\frac{\beta n_{(1)}\log n_{(1)}}{\rho_0}} \| \Z \|_{\infty},
\end{equation}
holds with a high probability.
\end{lemma}

\begin{lemma}
\label{Z_infty}
[Lemma 6.3 in \cite{candes2012exact} and Lemma 3.2 in \cite{candes2011robust}]: Suppose $ \Z $ is a fixed matrix, and $ \Omega_0 \sim \mbox{Ber}(\rho_0) $. For some small constant $ C_{1} >0 $, when $ \rho_0 \geq C_1 \mu \log n_{(1)}/n_{(2)} $,
 \begin{equation}
\| \left( \mathbf{I}- \rho_0^{-1} \mathcal{P}_{\Omega_0}\right) \Z \| \leq C_{1} \sqrt{\frac{\beta n_{(1)}\log n_{(1)}}{\rho_0}} \| \Z \|_{\infty},
\end{equation}
holds with a high probability.
\end{lemma}

\subsubsection{Proof of Lemma \ref{low rank certificate dip}}
\begin{proof}
We introduce some symbols first. Setting
\begin{equation*}
\Z_j = \U \V^T-\mathcal{P}_{T} \Y_j,
\end{equation*}
therefore, for all $j \geq 0$, we have $\Z_j \in T$. Considering the definition of $\Y_j$ in Eq. (\ref{Y_j}), it follows that $\Y_j \in \Omega^{\bot}$. Then, we have 
 \begin{equation*}
 \begin{split}
 \Z_j &= (\mathcal{P}_{T} - q^{-1}\mathcal{P}_{T}\mathcal{P}_{\Omega_{j}}\mathcal{P}_{T}) \Z_{j-1}, \\
 \Y_j &= \Y_{j-1} + q^{-1}\mathcal{P}_{\Omega_{j}} \Z_{j-1}.
 \end{split}
 \end{equation*}
Therefore, when
 \begin{equation}
 \label{q_range}
 q \geq C_{0}\epsilon^{-2} \mu r \log n_{(1)} /n_{(2)},
 \end{equation}
by employing Lemma \ref{Z_condition}, we obtain
 \begin{equation}
 \label{epsilon_z}
 \| \Z_j \|_{\infty} \leq \epsilon  \| \Z_{j-1} \|_{\infty}.
 \end{equation}
Specifically, this ensures that the following equation holds with high probability, i.e.,
\begin{equation}
\| \Z_j \|_{\infty} \leq \epsilon^{j} \|\U \V^T \|_{\infty}.
\end{equation}

When $ q $ satisfies Eq. (\ref{q_range}), by Lemma \ref{T_condition}, we have
 \begin{equation}
\| \Z_j \|_{F} \leq \epsilon \| \Z_{j-1} \|_{F}
\end{equation}
This further ensures the high-probability satisfaction of the following inequality,
 \begin{equation}
 \label{Z_inequality}
 \| \Z_j \|_F \leq \epsilon^{j} \|\U\V^T \|_F
  \leq  \epsilon^j \sqrt{r}.
 \end{equation}
We assume $ \epsilon \leq e^{-1} $.

\textbf{Proof of (a)}. 
Since $ \Y_{j_{0}} = \sum_{j} {q^{-1}\mathcal{P}_{\Omega_{j}} \Z_{j-1}}$, we have 
 \begin{equation}
 \label{y_j0}
 \begin{split}
 \| \mathbf{W}^L \| &= \| \mathbf{P}_{T^{\bot}} \Y_{j_0} \|_{\infty}   \leq \sum_{j} \| q^{-1} \mathcal{P}_{T^{\bot}} \mathcal{P}_{\Omega_{j}} \Z_{j-1}\| \\
 &\leq \sum_{j} \| \mathcal{P}_{T_{1}^{\bot}}( q^{-1}\mathcal{P}_{\Omega_{j}}\Z_{j-1}- \Z_{j-1})\| \\
& \leq \sum_{j} \| q^{-1} \mathcal{P}_{\Omega_{j}} \Z_{j-1} - \Z_{j-1}\| \\
 & \leq C_{1} \sqrt{\frac{\beta n_{(1)}\log n_{(1)}}{q}} \sum_j \| \Z_{j-1} \|_{\infty} \\
& \leq C_{1} \sqrt{\frac{\beta n_{(1)}\log n_{(1)}}{q}} \sum_j \epsilon^j \| \U\V^T \|_{\infty} \\
  & \leq \frac{C_{1}}{(1-\epsilon)} \sqrt{\frac{\beta n_{(1)}\log n_{(1)}}{q}} \|\U\V^T \|_{\infty}. \\
 \end{split}
 \end{equation}
where the fourth step follows from Lemma \ref{Z_infty}, and the fifth step can be directly obtained from Eq. (\ref{epsilon_z}). By using the conditions in Eq. (\ref{q_range}) and Eq. (\ref{new_UV}), we can derive, for some constant $ C_2 $, we have:
\begin{equation*}
\| \W^L \| \leq C_2 \epsilon.
\end{equation*}

\textbf{Proof of (b)}. Since $ \mathcal{P}_{\Omega}  \Y_{j_0} = 0 $,
\begin{equation*}
\begin{split}
\mathcal{P}_{\Omega}( \U \V^T+ \W^L) & =\mathcal{P}_{\Omega}( \U \V^T+ \mathcal{P}_{T^{\bot}}\Y_{j_0})\\
&= \mathcal{P}_{\Omega} (\U \V^T - \mathcal{P}_{T}\Y_{j_0}) = \mathcal{P}_{\Omega}(\Z_{j_0}).
\end{split}
\end{equation*}
By using Eq. (\ref{q_range}) and Eq. (\ref{Z_inequality}), we can obtain:
 \begin{equation*}
 \| \mathcal{P}_{\Omega}(\Z_{j_0}) \|_F \leq  \|\Z_{j_0}\|_F \leq  \epsilon^{j_0} \sqrt{r}.
 \end{equation*}
Since$ \epsilon\leq e^{-1} $, $ j_0 \geq 2\log n_{(1)} $, we have $ \epsilon^{j_0}\leq 1/n_{(1)}^2 $. Since
\begin{equation*}
\begin{split}
\|\mathcal{P}_{\Omega} ( \A^T( \U\V^T + \W^L))\|_F & \leq \| \A^T( \U\V^T + \W^L) \|_F \\
&\leq \| \A^T\| \| ( \U\V^T + \W^L) \|_F \\
& \leq \alpha \| ( \U\V^T + \W^L) \|_F \\
& \leq \alpha \|\Z_{j_0}\|_F \leq \frac{\sqrt{r}\alpha}{n_{(1)}^2}.
\end{split}
\end{equation*}
holds with at least $ 1-n_{(1)}^{-\beta} $ probability, where $ \beta >2 $, and $ \lambda \gamma \leq 1/4$. Therefore, when $ n_{(1)} \geq 2r^{1/4}\alpha^{1/2} $, we can easily establish the validity of condition (b). This completes the proof of this condition.

\textbf{Proof of (c)}. Since $\A^T$ satisfies Assumption \ref{assumption3}, we have
 \begin{equation}
 \| \A^T (\U\V^T + \W^L) \|_{\infty} \leq  \| \U\V^T + \W^L \|_{\infty}.
 \end{equation}

Therefore, we only need to prove $ \| \U\V^T + \W^L \|_{\infty} \leq \lambda/4$. We have $\U\V^T + \W^L = \Z_{j_0} + \Y_{j_0} $ and know that the support set of $ \Y_{j_0} $ is $ \Omega^c $. Thus, since $ \| \Z_{j_0}\|_{\infty} \leq \| \Z_{j_0}\|_F \leq \lambda/8 $, proving condition (c) is satisfied requires demonstrating $ \|\Y_{j_0} \|_{\infty} \leq \frac{\lambda}{8} $. To achieve this, we derive
\begin{equation*}
\begin{split}
 \| \Y_{j_0} \|_{\infty} & \leq q^{-1} \sum_{j} \| \mathcal{P}_{\Omega_{j}} \Z_{j-1}  \|_{\infty} \\
& \leq  q^{-1} \sum_{j} \| \Z_{j-1}  \|_{\infty}\leq  q^{-1} \sum_{j} \epsilon^j \|\U\V^T \|_{\infty}.
\end{split}
\end{equation*}

Since $ \| \U\V^T\|_{\infty} \leq \sqrt{\frac{\mu r}{n_1 n_2 }}$ for some $ C' $, we can deduce that
\begin{equation}
 \| \Y_{j_0} \|_{\infty} \leq \frac{C^{'}\epsilon^2}{\sqrt{ \mu r {n_{(2)}^{-1}} n_{(1)}(\log n_{(1)})^2 }},
\end{equation}
holds when $ q $ satisfies the Eq. (\ref{q_range}). Setting $ \lambda = 1/\sqrt{n_{(1)}} $, we can obtain that when
\begin{equation*}
\epsilon \leq C \left(  \frac{\mu r (\log n_{(1)})^2}{n_{(2)}}\right)^{\frac{1}{4}},
\end{equation*}
holds, then $ \| \Y_{j_0} \|_{\infty} \leq \lambda/8 $.

We have seen that if $\epsilon$ is small enough and $j_0 \geq 2\log n_{(1)}$, then (a)-(b) are satisfied. For (c), we can choose $\epsilon$ to be $\mathcal{O}\left(  (\mu r (\log n_{(1)})^2/n_{(2)} )^{1/4}\right)$, as long as $\rho_r$ in Eq. (\ref{pho_r_mnn}) is small enough, ensuring that $\epsilon$ is also small enough. Note that everything is consistent, as $C_{0}\epsilon^{-2} \mu r \log n_{(1)} /n_{(2)} < 1$.
\end{proof}

\subsubsection{Proof of Lemma \ref{sparse certificate dip}}
\begin{proof}
Following the proof in Lemma 2.9 in \cite{candes2011robust}, we have
\begin{enumerate}
\item[(a)] $ \| \mathbf{W}^S \|_F < 1/4 $,
\item[(b)] $ \| \mathbf{W}^S \|_\infty < \lambda/4 $.
\end{enumerate}
{\color{red}According to Assumption \ref{assumption3} about the transformation operator in the main text}, we can further obtain
 \begin{equation}
 \begin{split}
 \| \mathcal{P}_{\Omega^{\bot}} (\A^T \mathbf{W}^S) \|_\infty & \leq  \| \A^T \mathbf{W}^S \|_\infty \\
 & \leq \| \mathbf{W}^S \|_\infty < \lambda/4.
 \end{split}
 \end{equation}
Thus, this lemma holds. The proof is completed.
\end{proof}

\subsection{Proofs of Some Key Lemmas}

\subsubsection{Proof of Lemma \ref{Dual_certifications_dip}}

\begin{proof}
It is easy to see that for any non-zero matrix $\mathbf{H}$, the pair $\left(\X_0 + \mathbf{H}, \S_0 - \mathbf{H}\right)$ is also a feasible solution. Next, we show that the objective function at $\left(\X_0, \S_0\right)$ is greater than the objective function at $\left(\X_0 + \mathbf{H}, \S_0 - \mathbf{H}\right),$ thus proving that $\left(\X_0, \S_0\right)$ is the unique solution. To achieve this, let $\U\V^T + \W^0$ be an arbitrary subgradient of the nuclear norm at $\X_0,$ and $\mbox{sgn}(\S_0) + \mathbf{F}_0$ be an arbitrary subgradient of the $\ell_1$ norm at $\S_0$. Then we have
\begin{equation*}
\begin{split}
 \|\A\X_0 &+ \A\mathbf{H}\|_*  + \lambda \|\S_0 - \mathbf{H}\|_1  \geq \|\A\X_0\|_*  + \lambda \|\S_0\|_1 \\
 & + \left\langle \A^T\U\V^T + \A^T\W^0, \mathbf{H}\right\rangle  -\lambda \left\langle \mbox{sgn}(\S_0) + \mathbf{F}_0, \mathbf{H} \right\rangle.
 \end{split}
\end{equation*}
Through the duality relationship between the nuclear norm and the operator norm, we can choose $\W^0$ such that $\left\langle \W^0, \A\mathbf{H}\right\rangle  = \|\mathcal{P}_{T^{\bot}} (\A\mathbf{H}) \|_*$. Similarly, through the duality relationship between the $\ell_1$ norm and the operator norm, we can choose $\mathbf{F}_0$ such that $\left\langle\mathbf{F}_0, \mathbf{H}\right\rangle  = \|\mathcal{P}_{\Omega^{\bot}} (\mathbf{H}) \|_1$\footnote{For example, $\mathbf{F}_0 = -\mbox{sgn}(\mathcal{P}_{\Omega^{\bot}}\mathbf{H})$ is such a matrix. Also, through the duality relationship between the nuclear norm and the operator norm, there exists a matrix $\W$ with $\| \W\| = 1$ such that $\left\langle \W,\mathcal{P}_{T^{\bot}} (\A\mathbf{H}) \right\rangle  = \|\mathcal{P}_{T^{\bot}} (\A\mathbf{H}) \|_*$. We take $\W^0 = \mathcal{P}_{T^{\bot}} (\W)$.}. Therefore, we further have:
\begin{equation*}
\begin{split}
 \|\A\X_0 & + \A\mathbf{H}\|_* +  \lambda \|\S_0 - \mathbf{H}\|_1 \geqslant  \|\A\X_0\|_* + \lambda \|\S_0\|_1   \\
 & + \|\mathcal{P}_{T^{\bot}} (\A\mathbf{H}) \|_* + \lambda\|\mathcal{P}_{\Omega^{\bot}} \mathbf{H} \|_1 \\
&+ \left\langle \A^T\U\V^T-\lambda \mbox{sgn}(\S_0), \mathbf{H}\right\rangle.
 \end{split}
\end{equation*}
Given the assumption (\ref{kkt}), for $\beta = \max \left( \|\W\|, \| \mathbf{F}\| \right) <1$, we have
\begin{equation*}
\begin{split}
& \left| \left\langle \A^T\U\V^T  - \lambda \mbox{sgn}(\S_0), \mathbf{H}\right\rangle \right|   \leq \left| \left\langle \mathbf{W}, \A\mathbf{H} \right\rangle \right|  + \lambda \left| \left\langle \mathbf{F}, \mathbf{H} \right\rangle \right|  \\
&\qquad \leq \beta \left( \| \mathcal{P}_{T^{\bot}}{(\A\mathbf{H})}\|_* + \lambda \|\mathcal{P}_{\Omega^{\bot}}{\mathbf{H}}\|_1\right)
\end{split}
\end{equation*}
Therefore, we have
\begin{equation*}
\begin{split}
\|\A\X_0+ \A\mathbf{H}\|_* & + \lambda \|\S_0 - \mathbf{H}\|_1 \\
& \geq \|\A\X_0\|_* +\lambda \|\S_0\|_1  \\
& +\left( 1-\beta\right) \left( \| \mathcal{P}_{T^{\bot}}{(\A\mathbf{H})}\|_* + \lambda \|\mathcal{P}_{\Omega^{\bot}}{\mathbf{H}}\|_1\right).
\end{split}
\end{equation*}
Note that $ \Omega \bigcap T = \left\lbrace 0 \right\rbrace$, unless $ \mathbf{H}=0$, in which case $ \| \mathcal{P}_{T^{\bot}}{(\A\mathbf{H})}\|_* + \lambda \|\mathcal{P}_{\Omega^{\bot}}{\mathbf{H}}\|_1 > 0$. The proof is completed.
\end{proof}

\subsubsection{Proof of Lemma \ref{Dual_certifications_Relaxtion_dip}}
\begin{proof}
Following the proof process of Lemma \ref{Dual_certifications_dip}, we can first get
\begin{equation*}
\begin{split}
 \|\A & \X_0 + \A\mathbf{H}\|_* + \lambda \|\S_0 - \mathbf{H}\|_1\geq \|\A\X_0\|_* +\lambda \|\S_0\|_1  \\
&  + \frac{1}{2}\left( \| \mathcal{P}_{T^{\bot}}{(\A\mathbf{H})}\|_* + \lambda \|\mathcal{P}_{\Omega^{\bot}}{\mathbf{H}}\|_1\right)  -\lambda/4 \left\langle \mathcal{P}_{\Omega}\D, \A\mathbf{H} \right\rangle \\
& \geq \|\A\X_0\|_* +\lambda \|\S_0\|_1  + \frac{1}{2} \| \mathcal{P}_{T^{\bot}}{(\A\mathbf{H})}\|_* \\
& + \frac{\lambda}{2} \|\mathcal{P}_{\Omega^{\bot}}{\mathbf{H}}\|_1 -\lambda/4\Vert  \mathcal{P}_{\Omega} (\A\mathbf{H})\Vert_F. 
\end{split}
\end{equation*}
For any matrices $\A$ and $\B$, according to the normed triangle inequality $\|\A\B\|_F\leq \| \A\| \| \B\|_F$, and the support set $ \Omega $ is randomly distributed, we can get the following inequality:
\begin{equation}
\label{smoothness}
\begin{split}
\|\mathcal{P}_{\Omega^{\bot}}{(\A\mathbf{H})} \|_F & \leq \Vert \A \Vert_F \|\mathcal{P}_{\Omega^{\bot}}{(\mathbf{H})} \|_F \\
& = \|\mathcal{P}_{\Omega^{\bot}}{(\mathbf{H})} \|_F,
\end{split}
\end{equation}
{\color{red}where the last equality holds because of Assumption 3 about the transformation operator in the main text.}

Based on the following facts,
\begin{equation*}
\begin{split}
& \|\mathcal{P}_\Omega{(\A\mathbf{H})} \|_F \leq \|\mathcal{P}_\Omega{\left( \mathcal{P}_{T}+\mathcal{P}_{T^{\bot}} \right) } {(\A\mathbf{H})}\|_F \\
& \quad \leq \| \mathcal{P}_\Omega \mathcal{P}_{T} {(\A\mathbf{H})}\|_F + \| \mathcal{P}_\Omega \mathcal{P}_{T^{\bot}} {(\A\mathbf{H})}\|_F \\
&\quad \leq \frac{1}{2} \|\A\mathbf{H}\|_F + \|\mathcal{P}_{T^{\bot}} {(\A\mathbf{H})} \|_F \\
& \quad\leq \frac{1}{2} \|\mathcal{P}_\Omega {(\A\mathbf{H})}\|_F + \frac{1}{2} \|\mathcal{P}_{\Omega^{\bot}} {(\A\mathbf{H})}\|_F+ \|\mathcal{P}_{T^{\bot}} {(\A\mathbf{H})} \|_F\\
 &\quad \leq \frac{1}{2} \|\mathcal{P}_\Omega {(\A\mathbf{H})}\|_F + \frac{1}{2} \|\mathcal{P}_{\Omega^{\bot}} {(\mathbf{H})}\|_F+ \|\mathcal{P}_{T^{\bot}} {(\A\mathbf{H})} \|_F,\\
\end{split}
\end{equation*}
we can obtain
\begin{equation*}
\|\mathcal{P}_\Omega{(\A\mathbf{H})} \|_F  \leq  \|\mathcal{P}_{\Omega^{\bot}}{(\mathbf{H})} \|_F + 2\|\mathcal{P}_{T^{\bot}}{(\A\mathbf{H})} \|_F.
\end{equation*}
Based on the above formula, $\Vert \X \Vert_F \leq \Vert \X \Vert_1$ and $\Vert \X \Vert_F \leq \Vert \X \Vert_*$ for any matrix $\X$,  we can get
\begin{equation*}
\begin{split}
\|\A\X_0+ \A\mathbf{H}\|_* &+ \lambda \|\S_0 - \mathbf{H}\|_1  \geq \|\A\X_0\|_* +\lambda \|\S_0\|_1 \\
&+ \frac{1-\lambda}{2} \| \mathcal{P}_{T^{\bot}}{(\A\mathbf{H})}\|_* + \frac{\lambda}{4} \|\mathcal{P}_{\Omega^{\bot}}{\mathbf{H}}\|_1. 
\end{split}
\end{equation*}
Therefore, we have $ \frac{1-\lambda}{2} \| \mathcal{P}_{T^{\bot}}{(\A\mathbf{H})}\|_* + \frac{\lambda}{4} \|\mathcal{P}_{\Omega^{\bot}}{\mathbf{H}}\|_1 $ is strictly positive, Unless $ \mathbf{H} \neq \mathbf{0} $.
\end{proof}

\section{The Proof of MNN-MC Theorem}
\label{mc_proof}

For ease of exposition, let us first provide the equivalent model of MNN-MC.
\subsection{The Equivalent Model of MNN-MC Model}
The variant of the MNN-MC model (\ref{mnn_mc}) can be rewritten as
\begin{equation}
\label{mc_v}
\min_{\X} \quad \|\A\X \|_*, \ \mbox{s.t.}\ \mathcal{P}_{\Omega}\mathbf{M} = \mathcal{P}_{\Omega}\mathbf{X}. \\
\end{equation}
The Lemma \ref{mc_condition} gives sufficient conditions for the uniqueness of the minimizer to Eq. (\ref{mc_v}).

\subsection{Dual Verification}
\begin{lemma}
\label{mc_condition}
Consider a matrix $\X_0=\U\mathbf{\Sigma}\V^T$ of rank $r$ which is feasible for the problem (\ref{mc_v}), and suppose that the following two conditions hold:
\begin{enumerate}
\item [1.] There exist a dual variable $\lambda$ such that $\Y=\mathcal{R}_{\Omega}^* \lambda$ obeys
\begin{equation}
\mathcal{P}_{T}{\Y} = \U\V^T, \Vert \mathcal{P}_{T^{\bot}}{\Y} \Vert <1.
\end{equation}
\item [2.] The sampling operator $\mathcal{R}_{\Omega}$ restricted to elements in $T$ is injective.
\end{enumerate}
Then $\X_0$ is the unique minimizer.
\end{lemma}

The proof of Lemma \ref{mc_condition} uses a standard fact which states that the nuclear norm and the spectral norm are dual to one another.

\begin{lemma}
\label{nu_spe}
For each pair $\W$ and $\mathbf{H}$, we have
\begin{equation}
\left\langle \W, \mathbf{H}\right\rangle \leq \Vert \W \Vert \Vert \mathbf{H} \Vert_*. 
\end{equation}
In addition, for each $\mathbf{H}$, there is a $\W$ obeying $\Vert \W \Vert=1$ which achieves the quality.
\end{lemma}
A variety of proofs are available for Lemma \ref{nu_spe}, and an elementary argument is sketched in \cite{recht2010guaranteed}. We now turn to the proof of Lemma \ref{mc_condition}.

\begin{proof}
Consider any perturbation $\X_0 +\mathbf{H}$, where $\mathcal{R}_{\Omega}(\mathbf{H}) =0$. Then for any $\W^0$ obeying $\W^0\in T^{\bot}$ and $\Vert \W \Vert \leq 1 $, $\U\V^T+\W^0$ is a subgradient of the nuclear norm at $\X_0$, therefore,
\begin{equation}
\begin{split}
\Vert \A\X_0 + \A\mathbf{H} \Vert_* & \geq \Vert \A\X_0 \Vert_* \\
& + \left\langle \A^T\U\V^T+\A^T\W^0, \mathbf{H} \right\rangle.
\end{split}
\end{equation}
Letting $\W=\mathcal{P}_{T^{\bot}}(\Y)$, we may write $\U\V^T=\mathcal{R}_{\Omega}^*\lambda-\W$. Since $\Vert \W \Vert <1 $ and $ \mathcal{R}_{\Omega}(\mathbf{H})=0$, it then follows that 
\begin{equation}
\begin{split}
\Vert \A\X_0 + \A\mathbf{H} \Vert_* & \geq \Vert \A\X_0 \Vert_* +\left\langle \A^T\W^0-\A^T\W, \mathbf{H} \right\rangle \\ 
& \geq \Vert \A\X_0 \Vert_* +\left\langle \W^0-\W, \A\mathbf{H} \right\rangle.
\end{split}
\end{equation}
Now, by construction,
\begin{equation}
\begin{split}
\left\langle \W^0-\W, \A\mathbf{H}\right\rangle & \geq \left\langle \mathcal{P}_{T^{\bot}}\left( \W^0-\W \right), \A\mathbf{H}\right\rangle \\
& = \left\langle  \W^0-\W, \mathcal{P}_{T^{\bot}}(\A\mathbf{H})\right\rangle.
\end{split}
\end{equation}
We use Lemma \ref{nu_spe} and set $\W^0=\mathcal{P}_{T^{\bot}}(\Z)$ where $\Z$ is any matrix obeying $\Vert \Z \Vert \leq 1$ and $\left\langle \Z, \mathcal{P}_{T^{\bot}}(\A\mathbf{H}) \right\rangle=\Vert \mathcal{P}_{T^{\bot}}(\A\mathbf{H}) \Vert_*$. Then $\W^0 \in \bot^T$, $\Vert \W^0 \Vert \leq 1$, and 
\begin{equation}
\left\langle \W^0-\W, \A\mathbf{H}\right\rangle \geq \left( 1-\Vert \W\Vert\right) \Vert \mathcal{P}_{T^{\bot}}(\A\mathbf{H}) \Vert_*,
\end{equation}
which by assumption is strictly positive unless $ \mathcal{P}_{T^{\bot}}(\mathbf{H})=0$. In other words, $ \Vert \A\X_0 + \A\mathbf{H} \Vert_* \geq \Vert \A\X_0 \Vert_* $ unless $ \mathcal{P}_{T^{\bot}}(\mathbf{H})=0$. Assume then that $ \mathcal{P}_{T^{\bot}}(\mathbf{H})=0$ or equivalently that $ \mathbf{H}\in T $. Then $ \mathcal{R}_{\Omega}(\mathbf{H})=0 $ implies that $\mathbf{H}=0$ by the injectivity assumption. In conclusion, $ \Vert \A\X_0 + \A\mathbf{H} \Vert_* \geq \Vert \A\X_0 \Vert_* $ unless $\mathbf{H}=0$.  
\end{proof}

\subsection{Construct dual variables}
{\color{red}The remaining task is to construct dual variable that satisfy the corresponding conditions mentioned in Lemma \ref{mc_condition} (and show the injectivity of the sampling operator restricted to matrices in $T$ along the way).} Set $\mathcal{P}_{\Omega}$ to be the orthogonal projector onto the indices in $\Omega$ so that the $(i,j)$-th component of $\mathcal{P}_{\Omega}(\X)$ is equal to $X_{ij}$ if $(i,j)\in \Omega$ and zero otherwise.

Following the settings in \cite{candes2012exact}, we introduce the  operator $\mathcal{A}_{\Omega T}$ defined by 
\begin{equation}
\mathcal{A}_{\Omega T} (\X) = \mathcal{P}_{\Omega}\mathcal{P}_{T}(\X).
\end{equation}
Then if $ \mathcal{A}_{\Omega T}^*\mathcal{A}_{\Omega T}= \mathcal{P}_{T}\mathcal{P}_{\Omega}\mathcal{P}_{T}$ has full rank when restricted to $T$, the dual variable $\Y$ can be given by 
\begin{equation}
\label{dual_mc}
\Y = \mathcal{A}_{\Omega T}(\mathcal{A}_{\Omega T}^*\mathcal{A}_{\Omega T})^{-1}(\U\V^T)
\end{equation}
We clarify the meaning of (\ref{dual_mc}) to avoid any confusion. $(\mathcal{A}_{\Omega T}^*\mathcal{A}_{\Omega T})^{-1}(\U\V^T)$ is meant to be that element $\mathbf{F}$ in $T$ obeying $(\mathcal{A}_{\Omega T}^*\mathcal{A}_{\Omega T})(\mathbf{F})=\U\V^T$. Then, according to the setup of $\Y$ in Eq. (\ref{dual_mc}), we can obtain
\begin{equation}
\begin{split}
\mathcal{P}_{T}(\Y) &=\mathcal{P}_{T}\mathcal{A}_{\Omega T}(\mathcal{A}_{\Omega T}^*\mathcal{A}_{\Omega T})^{-1}(\U\V^T)\\
& =(\mathcal{P}_{T}\mathcal{P}_{\Omega}\mathcal{P}_{T})(\mathcal{A}_{\Omega T}^*\mathcal{A}_{\Omega T})^{-1}(\U\V^T) \\ 
&=(\mathcal{P}_{T}\mathcal{P}_{\Omega}\mathcal{P}_{\Omega}\mathcal{P}_{T})(\mathcal{A}_{\Omega T}^*\mathcal{A}_{\Omega T})^{-1}(\U\V^T) \\
& =(\mathcal{A}_{\Omega T}^*\mathcal{A}_{\Omega T})(\mathcal{A}_{\Omega T}^*\mathcal{A}_{\Omega T})^{-1}(\U\V^T) \\
& = \U\V^T.\\
\end{split}
\end{equation}
To summarize the aims of our proof strategy,
\begin{enumerate}
\item[1)] We must first show that $\mathcal{A}_{\Omega T}^*\mathcal{A}_{\Omega T}=\mathcal{P}_{T}\mathcal{P}_{\Omega}\mathcal{P}_{T}$ is a one-to-one linear mapping from $T$ onto itself. In this case, $\mathcal{A}_{\Omega T}=\mathcal{P}_{\Omega}\mathcal{P}_{T}$ as mapping from $T$ to $R^{n_1\times n_2}$ is injective. This is the second sufficient condition of Lemma \ref{mc_condition}. Moreover, our anstatz for $\Y$ given by Eq. (\ref{dual_mc}) is well defined.
\item[2)] Having established that $\Y$ is well defined, we will show that 
\begin{equation}
\Vert \mathcal{P}_{T^{\bot}}(\Y) \Vert < 1,
\end{equation}
thus proving the first sufficient condition.
\end{enumerate}

\subsection{Proof of Dual Verification}
Next, we need to prove the injectivity of $\mathcal{A}_{\Omega T}$, and $\Vert \mathcal{P}_{T^{\bot}}(\Y) \Vert < 1$.
\subsubsection{The Injectivity Property} 
To prove this, we will show that the linear operator $p^{-1}\mathcal{P}_{T}(\mathcal{P}_{\Omega}-p\mathcal{I})\mathcal{P}_{T}$ has small operator norm, which we recall is $\mbox{sup}_{\Vert \X \Vert_F \leq 1} p^{-1}\Vert \mathcal{P}_{T}(\mathcal{P}_{\Omega}-p\mathcal{I})\mathcal{P}_{T}(\X) \Vert_F$. Proving the above conclusion requires the following two lemmas.
\begin{lemma}[Theorem 4.1 in \cite{candes2012exact}] Suppose $\Omega$ is sampled from $\mbox{Ber}(p)$ and put $n=\max\{n_1,n_2\}$. Suppose that the low-rank matrix $\X_0$ satisfies the incoherence condition with $\mu$. Then there is a numerical constant $C_R$ such that for all $\beta >1$,
\begin{equation}
\label{op_inject}
p^{-1} \Vert \mathcal{P}_{T}\mathcal{P}_{\Omega} \mathcal{P}_{T} -p \mathcal{P}_{T} \Vert \leq C_R \sqrt{\frac{\mu nr(\beta \log n)}{m}}
\end{equation}
with probability at least $1-3n^{-\beta}$ provided that $C_R \sqrt{\frac{\mu nr(\beta \log n)}{m}}<1$.
\label{op_inject_t}
\end{lemma}
\begin{lemma}[Theorem 4.2 in \cite{candes2012exact}] 
Let $\{\delta_{ab}\}$ be independent $0/1$ Bernoulli variables with $\mathbb{P}(\delta_{ab}=1)=p=\frac{m}{n_1n_2}$ and put $n=\max\{ n_1,n_2\}$. Suppose that $\Vert \mathcal{P}_{T}(\mathbf{e}_a\mathbf{e}_b^*)\Vert_F^2\leq 2\mu_0r/n$. Set
\begin{equation}
\begin{split}
\Z &\equiv p^{-1} \left\Vert \sum_{ab}(\delta_{ab}-p) \mathcal{P}_{T}(\mathbf{e}_a\mathbf{e}_b^*)\otimes  \mathcal{P}_{T}(\mathbf{e}_a\mathbf{e}_b^*) \right\Vert \\
& =p^{-1} \left\Vert \mathcal{P}_{T}\mathcal{P}_{\Omega} \mathcal{P}_{T}- p \mathcal{P}_{T}\right\Vert.
\end{split}
\end{equation}
1. There exists a constant $C_R^{'}$ such that 
\begin{equation}
\mathbb{E}(\Z) \leq C_R^{'}\sqrt{\frac{\mu n r \log n}{m}}
\end{equation}
provided that the right-band side is smaller than 1.

2. Suppose $ \mathbb{E}(\Z) \leq 1 $. Then for each $\lambda >0$, we have
\begin{equation}
\begin{split}
\mathbb{P}  &\left( \vert \Z- \mathbb{E}(\Z)\vert  > \lambda \sqrt{\frac{\mu nr\log n}{m}} \right) \\
& \leq 3 \exp \left( -\gamma_0^{'}\min \left\{\lambda^2\log n, \lambda\frac{m\log n}{\mu n r} \right\} \right)
\end{split}
\end{equation}
for some positive constant $\gamma_0^{'}$.
\end{lemma}

Take $ m $ large enough so that $C_R \sqrt{\mu (nr/m)\log n}\leq 1/2$. Then it follows from (\ref{op_inject}) that 
\begin{equation}
\label{op_inequ}
\begin{split}
\frac{p}{2} \left\Vert \mathcal{P}_{T}(\X)\right\Vert_F & \leq \left\Vert (\mathcal{P}_{T}\mathcal{P}_{\Omega}\mathcal{P}_{T})(\X)\right\Vert_F \\
& \leq \frac{3p}{2} \left\Vert \mathcal{P}_{T}(\X)\right\Vert_F 
\end{split}
\end{equation}
for all $ \X $ with large probability. In particular, the operator $ \mathcal{A}_{\Omega T}^*\mathcal{A}_{\Omega T}= \mathcal{P}_{T}\mathcal{P}_{\Omega}\mathcal{P}_{T}$ mapping $T$ onto itself is well conditioned, and hence invertible. An immediate consequence is the following corollary.

\begin{corollary}
Assume that $C_R \sqrt{\mu (nr/m)\log n}\leq 1/2$. With the same probability as in Lemma \ref{op_inject_t}, we have 
\begin{equation}
\label{op_col}
\Vert \mathcal{P}_{\Omega}\mathcal{P}_{T}(\X) \Vert_F \leq \sqrt{3p/2}\Vert \mathcal{P}_{T}(\X) \Vert_F.
\end{equation}
\end{corollary}
\begin{proof}
We have $\Vert \mathcal{P}_{\Omega}\mathcal{P}_{T}(\X) \Vert_F^2 = \left\langle  \X, (\mathcal{P}_{\Omega}\mathcal{P}_{T})^{*}\mathcal{P}_{\Omega}\mathcal{P}_{T}\X \right\rangle = \left\langle \X, (\mathcal{P}_{T}\mathcal{P}_{\Omega}\mathcal{P}_{T})\X \right\rangle$, and thus
\begin{equation}
\begin{split}
\Vert \mathcal{P}_{\Omega}\mathcal{P}_{T}(\X) \Vert_F^2 & = \left\langle \mathcal{P}_{T}(\X), (\mathcal{P}_{T}\mathcal{P}_{\Omega}\mathcal{P}_{T})\X \right\rangle \\
& \leq \Vert \mathcal{P}_{T}(\X) \Vert_F \Vert (\mathcal{P}_{T}\mathcal{P}_{\Omega}\mathcal{P}_{T})\X\Vert_F,
\end{split}
\end{equation}
where the inequality is due to Cauchy-Schwarz. The conclusion (\ref{op_col}) follows from (\ref{op_inequ}). This completes the proof.
\end{proof}
\subsubsection{The Size Property} 
Next, we prove that $\Vert \mathcal{P}_{T^{\bot}}(\Y) \Vert < 1$. Before proving this conclusion, we need to introduce five lemmas.
\begin{lemma}[Lemma 4.4 in \cite{candes2012exact}] Fix $\beta \geq 2$ and $\lambda \geq 1$. There is a numerical constant $C_0$ such that if $ m \geq \lambda \mu^2 n r \beta \log n $, then 
\begin{equation}
p^{-1}\Vert (\mathcal{P}_{T^{\bot}}\mathcal{P}_{\Omega}\mathcal{P}_{T})(\U\V^T)\Vert \leq C_0 \lambda^{-1/2}
\end{equation}
with probability at least $1-n^{-\beta}$.
\label{lem_mc_norm_1}
\end{lemma}

\begin{lemma}[Lemma 4.5 in \cite{candes2012exact}] Fix $\beta \geq 2$ and $\lambda \geq 1$. There are numerical constants $C_1$ and $c_1$ such that if $m \geq \lambda \mu \max (\sqrt{\mu},\mu) n r \beta \log n$, then 
\begin{equation}
p^{-1}\Vert (\mathcal{P}_{T^{\bot}}\mathcal{P}_{\Omega}\mathcal{P}_{T})\mathcal{H}(\U\V^T)\Vert \leq C_1 \lambda^{-1}
\end{equation}
with probability at least $1-c_1n^{-\beta}$, where $\mathcal{H}\equiv \mathcal{P}_{\Omega}-p^{-1}\mathcal{P}_{T}\mathcal{P}_{\Omega}\mathcal{P}_{T}$.
\label{lem_mc_norm_2}
\end{lemma}

\begin{lemma}[Lemma 4.6 in \cite{candes2012exact}] Fix $\beta \geq 2$ and $\lambda \geq 1$. There are numerical constants $C_2$ and $c_2$ such that if $m \geq \lambda \mu^{4/3} n r^{4/3} \beta \log n$, then 
\begin{equation}
p^{-1}\Vert (\mathcal{P}_{T^{\bot}}\mathcal{P}_{\Omega}\mathcal{P}_{T})\mathcal{H}^2(\U\V^T)\Vert \leq C_2 \lambda^{-3/2}
\end{equation}
with probability at least $1-c_2n^{-\beta}$.
\label{lem_mc_norm_3}
\end{lemma}

\begin{lemma}[Lemma 4.7 in \cite{candes2012exact}] Fix $\beta \geq 2$ and $\lambda \geq 1$. There are numerical constants $C_3$ and $c_3$ such that if $m \geq \lambda \mu^2 n r^2 \beta \log n$, then 
\begin{equation}
p^{-1}\Vert (\mathcal{P}_{T^{\bot}}\mathcal{P}_{\Omega}\mathcal{P}_{T})\mathcal{H}^3(\U\V^T)\Vert \leq C_3 \lambda^{-1/2}
\end{equation}
with probability at least $1-c_3n^{-\beta}$.
\label{lem_mc_norm_4}
\end{lemma}

\begin{lemma}[Lemma 4.8 in \cite{candes2012exact}] Under the assumptions of Lemma \ref{op_inject_t}, there is a numerical constant $C_{k_0}$ such that if $m \geq (2C_R)^2\mu n r \beta \log n$, then
\begin{equation}
\begin{split}
p^{-1} & \left\Vert (\mathcal{P}_{T^{\bot}}\mathcal{P}_{\Omega})\mathcal{P}_{T}\sum_{k\geq k_0}\mathcal{H}^k(\U\V^T)\right\Vert \leq C_{k_0} (\frac{n^2r}{m})^{1/2} \\
& \times (\frac{mu n r \beta \log n}{m})^{k_0/2}
\end{split}
\end{equation}
with probability at least $1-3n^{-\beta}$.
\label{lem_mc_norm_5}
\end{lemma}

Next, we give the proof of $\Vert \mathcal{P}_{T^{\bot}}(\Y) \Vert < 1$.
\begin{proof}
Introducing
\begin{equation}
\mathcal{H} \equiv \mathcal{P}_{T} - p^{-1} \mathcal{P}_{T}\mathcal{P}_{\Omega} \mathcal{P}_{T},
\end{equation}
where $\Vert \mathcal{H}(\X)\Vert_F\leq C_R\sqrt{\mu_0(nr/m)\beta\log n} \Vert \mathcal{P}_{T}(\X)\Vert_F$ holds with large probability because of Lemma \ref{op_inject_t}. For any matrix $\X \in T$, $(\mathcal{P}_{T}\mathcal{P}_{\Omega}\mathcal{P}_{T})^{-1}(\X)$ can be expressed in terms of power series
\begin{equation}
(\mathcal{P}_{T}\mathcal{P}_{\Omega} \mathcal{P}_{T})^{-1}(\X) = p^{-1}(\X +\mathcal{H}(\X)+\mathcal{H}^2(\X)+\cdots)
\end{equation}
for $\mathcal{H}$ is a contraction when $m$ is sufficiently large. Since $\Y=\mathcal{P}_{\Omega}\mathcal{P}_{T}( \mathcal{P}_{T}\mathcal{P}_{\Omega} \mathcal{P}_{T})^{-1}\times (\U\V^T)$, $ \mathcal{P}_{T^{\bot}}(\Y)$ may be decomposed as
\begin{equation}
\begin{split}
\mathcal{P}_{T^{\bot}}(\Y) & = p^{-1} (\mathcal{P}_{T^{\bot}}\mathcal{P}_{\Omega}\mathcal{P}_{T})(\U\V^T+\mathcal{H}(\U\V^T) \\
& +\mathcal{H}^2(\U\V^T)+\cdots).
\end{split}
\end{equation}
Therefore, to bound the norm of the left-hand side, it is of course sufficient to bound the norm of the summands in the right-hand side.

Under all of the assumptions of Theorem 2 in the main text, with the guidance of Lemmas \ref{lem_mc_norm_1}, \ref{lem_mc_norm_2}, \ref{lem_mc_norm_3}, and \ref{lem_mc_norm_5}, the latter applied with $k_0=3$. Together they imply that there are numerical constants $c$ and $c_0$ such that $\Vert \mathcal{P}_{T^{\bot}}(\Y)\Vert < 1$ with probability at least $1-cn^{-\beta}$ provided that the number of samples obeys 
\begin{equation}
m\geq c_0\max \left( \mu^2, \mu^{4/3}r, \mu n^{1/4} \right) nr\beta \log(n).
\end{equation}
This completes the proof.
\end{proof}

\section{Proof of the Optimal Solution to the Objective Function}
\label{opt_proof}

Based on Theorems 1 and 2 in the main text, we deduce the following corollary: 

\begin{corollary}
\label{coro_mnn}
Suppose $\X_0$ and $\S_0$ satisfy Assumptions \ref{assumption1} and \ref{assumption2}, and transformation operator ${\color{red}\mathcal{D}}(\cdot)$ satisfy Assumption \ref{assumption3}. Denote the objective functions of RPCA and MC models as 
\begin{equation}
\begin{split}
\mathcal{J}_1^{\mathcal{D}}(\X) := & \Vert \mathcal{D}(\X) \Vert_* +\lambda \Vert \M-\X \Vert_1, \\
\mathcal{J}_2^{\mathcal{D}}(\X) := & \Vert \mathcal{D}(\X) \Vert_* +\mu \Vert \mathcal{P}_{\Omega}(\M-\X) \Vert_F^2,\\
\end{split}
\end{equation}
respectively, where $\lambda =1/\sqrt{n_1}$ and $\mu=(\sqrt{n_1}+\sqrt{n_2})\sqrt{p}\sigma$ according to \cite{candes2010matrix}, and $n_1, n_2, \sigma ,p$ are the sizes of matrix, noise standard variance, and missing ratio. Then, for any feasible solution $\X$, we have:
\begin{equation}
\label{obj}
\mathcal{J}_1^{\mathcal{D}}(\X) \geq \mathcal{J}_1^{\mathcal{D}}(\X_0), \mathcal{J}_2^{\mathcal{D}}(\X) \geq \mathcal{J}_2^{\mathcal{D}}(\X_0).
\end{equation}
\end{corollary} 

Next, we give the proof of this corollary.

\begin{proof}
Since Theorem \ref{theorem_main3} holds under Assumptions \ref{assumption1}-\ref{assumption3}, then we can conclude that the point $(\X_0,\S_0)$ minimizes the objective function value of the MNN-RPCA model (\ref{mnn_rpca}), i.e., $\Vert \mathcal{D}(\X_0)\Vert_*+\lambda \Vert \S_0 \Vert_* = \Vert \mathcal{D}(\X_0)\Vert_*+\lambda \Vert \M - \X_0 \Vert_1$ is the minimum value among all feasible solutions. Therefore, we can obtain:
\begin{equation}
\mathcal{J}_1^{\mathcal{D}}(\X) \geq \mathcal{J}_1^{\mathcal{D}}(\X_0).
\end{equation}
Similarly, for the MNN-MC model (\ref{mnn_mc}), we have:
\begin{equation}
\mathcal{J}_2^{\mathcal{D}}(\X) \geq \mathcal{J}_2^{\mathcal{D}}(\X_0).
\end{equation} 
This completes the proof.
\end{proof}

\section{More Results}

In the section, we only provide the mean values of all methods on four categories of data. From the table in the main text, we have seen that the MNN-based models achieve the best restoration performance in the vast majority of cases. Here, taking PSNR as an example, we further present the performance of all methods on each dataset.

\begin{figure*}[!h]
\renewcommand{\arraystretch}{0.5}
\setlength\tabcolsep{0.5pt}
\centering
\begin{tabular}{c c c c}
\centering
\includegraphics[width=38mm, height = 30mm]{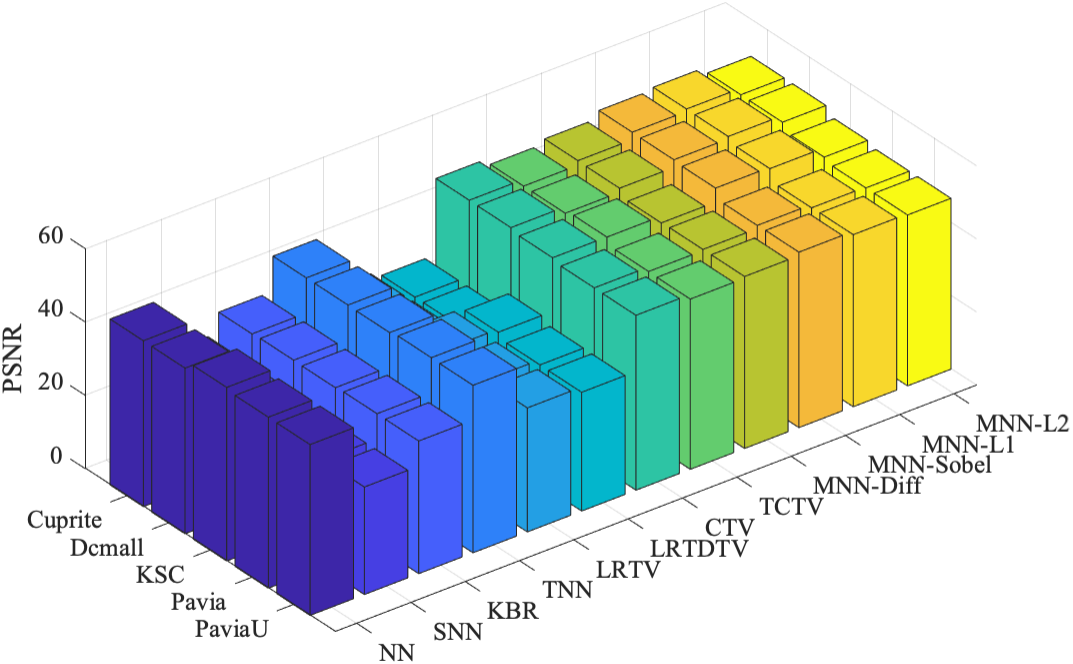}&
\includegraphics[width=38mm, height = 30mm]{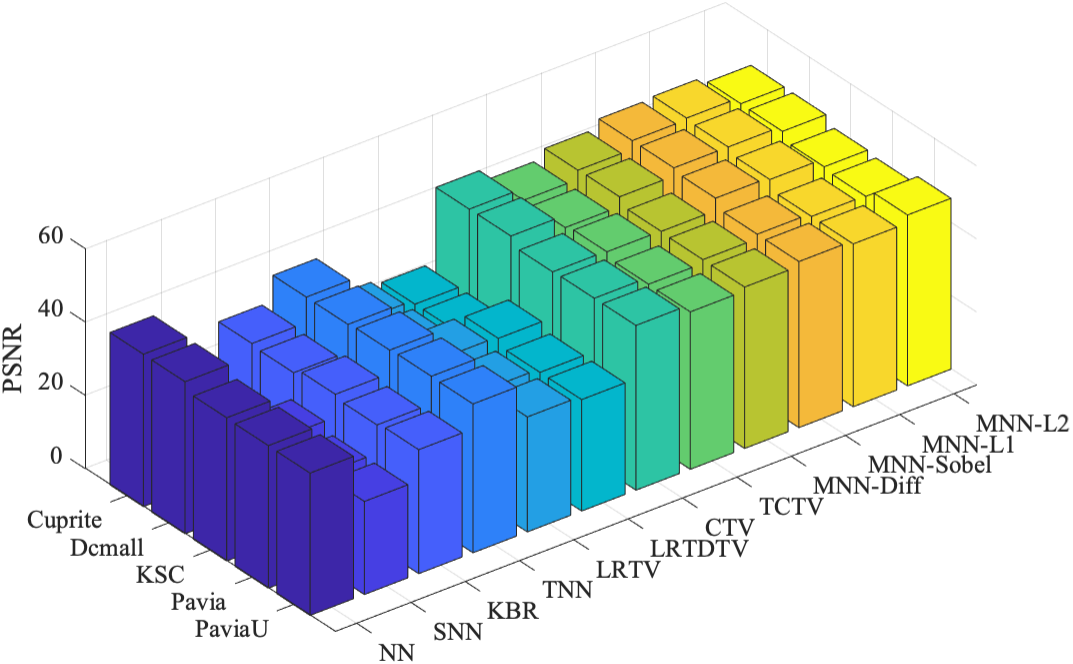}&
\includegraphics[width=38mm, height = 30mm]{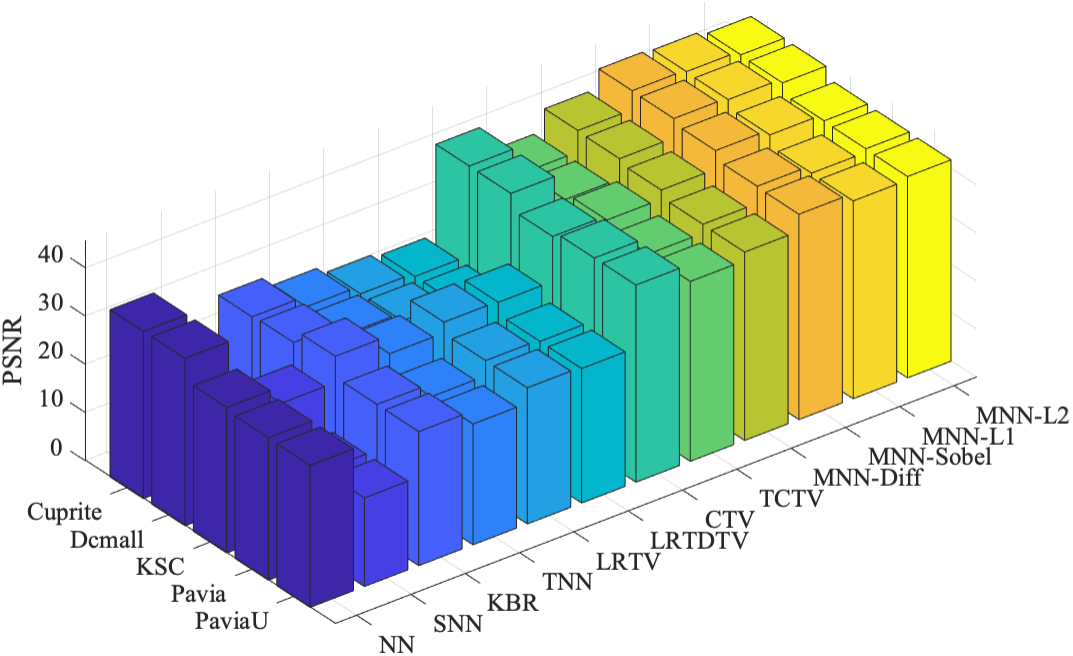}&
\includegraphics[width=38mm, height = 30mm]{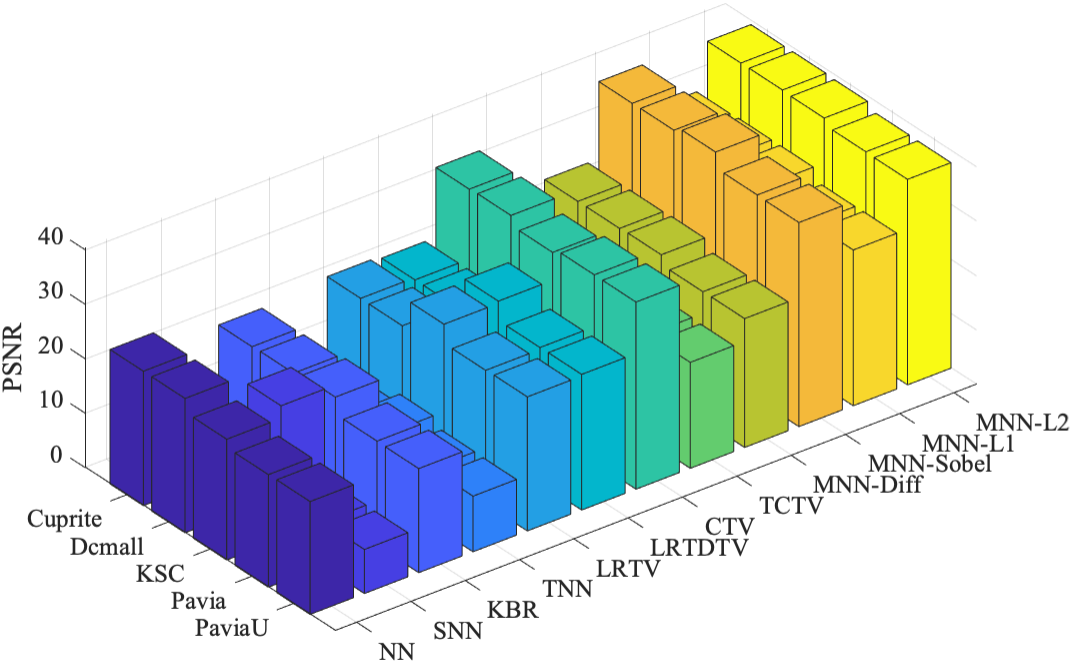}\\
\scriptsize{RPCA: s=0.1} & \scriptsize{RPCA: s=0.3} & \scriptsize{RPCA: s=0.5} & \scriptsize{RPCA: s=0.7} \\
\end{tabular}
\vspace{-2mm}
\caption{Performance comparison in terms of PSNR of recovered hyperspectral images obtained by all competing methods under RPCA tasks.}\label{rpca_hsi}
\vspace{-0.2cm}
\end{figure*}

\begin{figure*}[!h]
\renewcommand{\arraystretch}{0.5}
\setlength\tabcolsep{0.5pt}
\centering
\begin{tabular}{c c c c}
\centering
\includegraphics[width=38mm, height = 30mm]{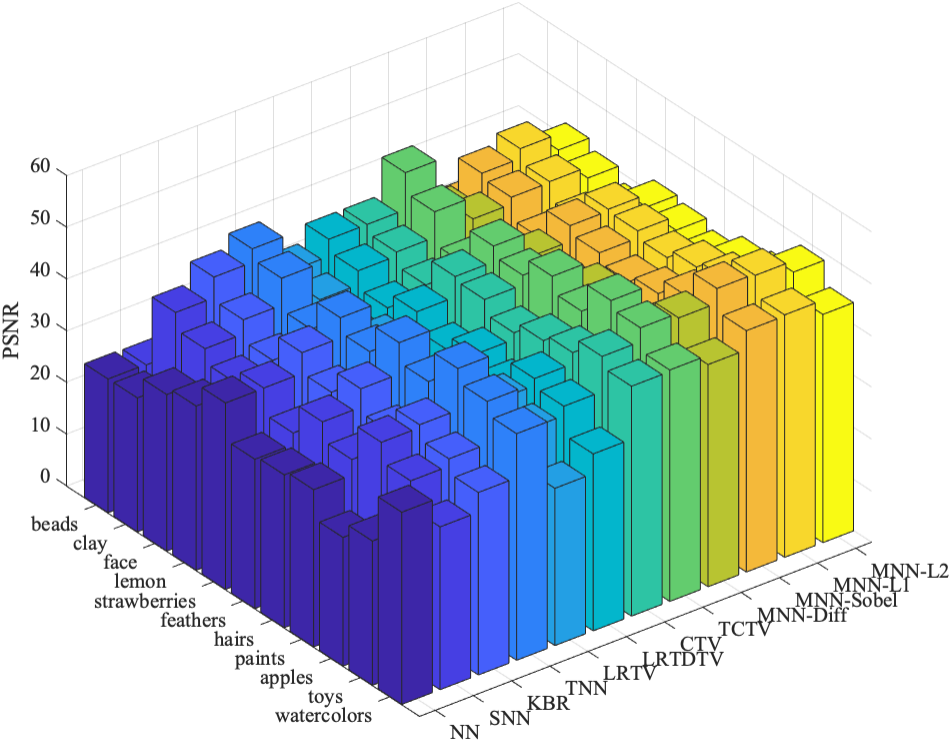}&
\includegraphics[width=38mm, height = 30mm]{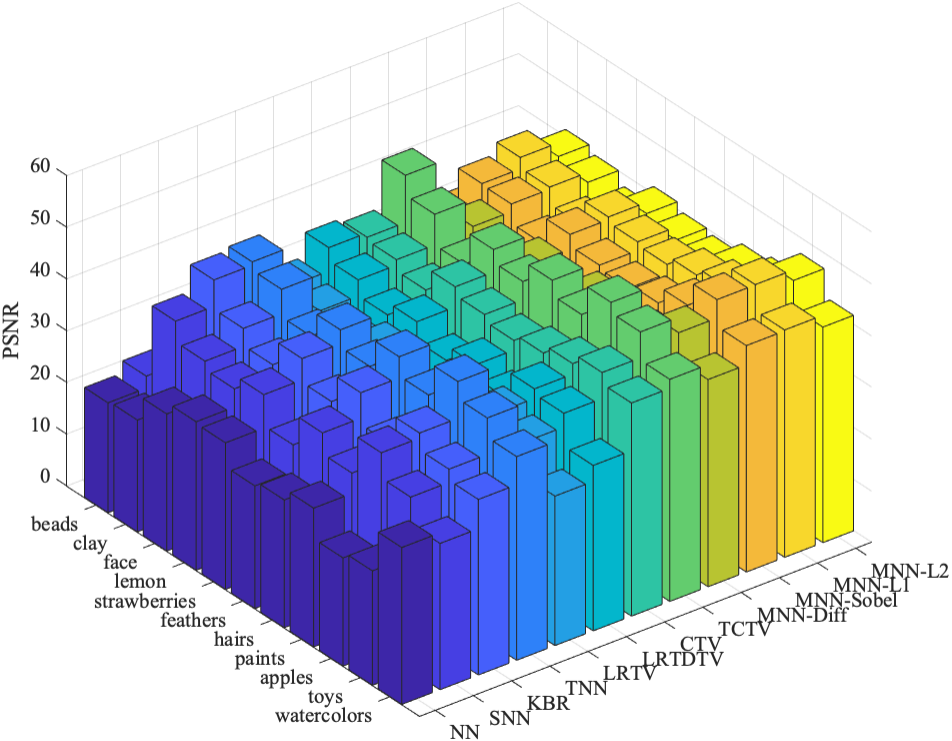}&
\includegraphics[width=38mm, height = 30mm]{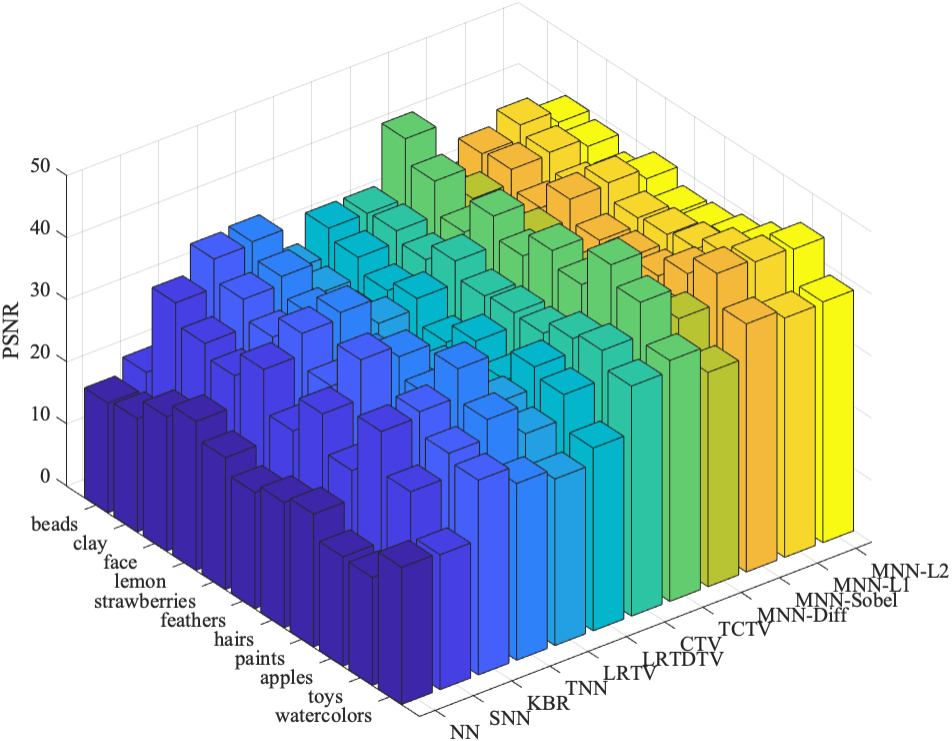}&
\includegraphics[width=38mm, height = 30mm]{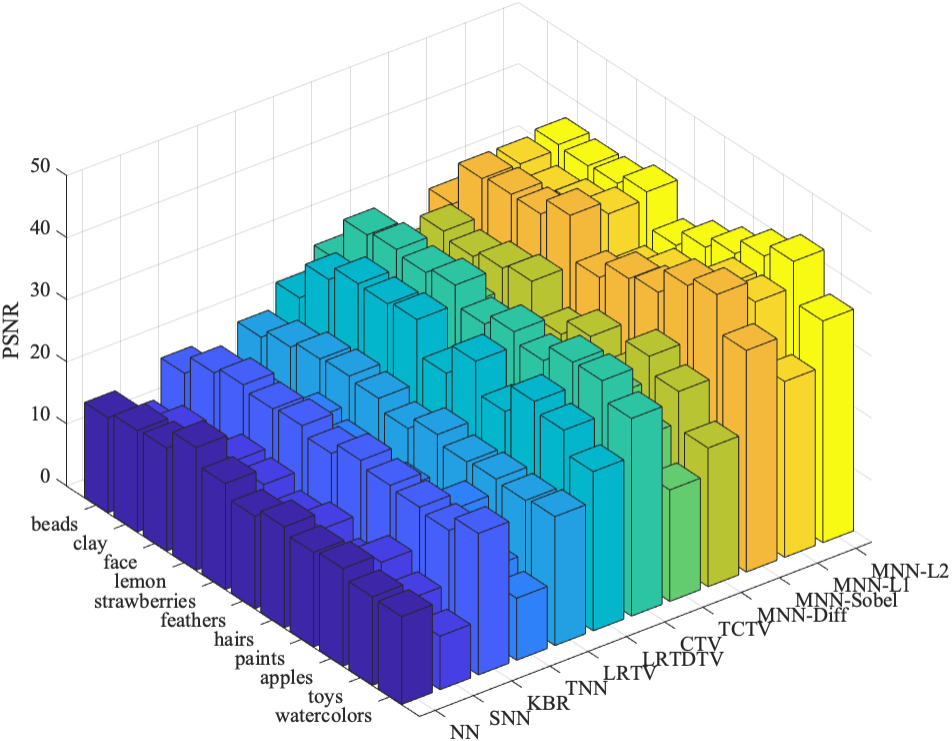}\\
\scriptsize{RPCA: s=0.1} & \scriptsize{RPCA: s=0.3} & \scriptsize{RPCA: s=0.5} & \scriptsize{RPCA: s=0.7} \\
\end{tabular}
\vspace{-2mm}
\caption{Performance comparison in terms of PSNR of recovered multispectral images obtained by all competing method under RPCA tasks.}\label{rpca_msi}
\vspace{-0.2cm}
\end{figure*}

\begin{figure*}[!h]
\renewcommand{\arraystretch}{0.5}
\setlength\tabcolsep{0.5pt}
\centering
\begin{tabular}{c c c c}
\centering
\includegraphics[width=38mm, height = 30mm]{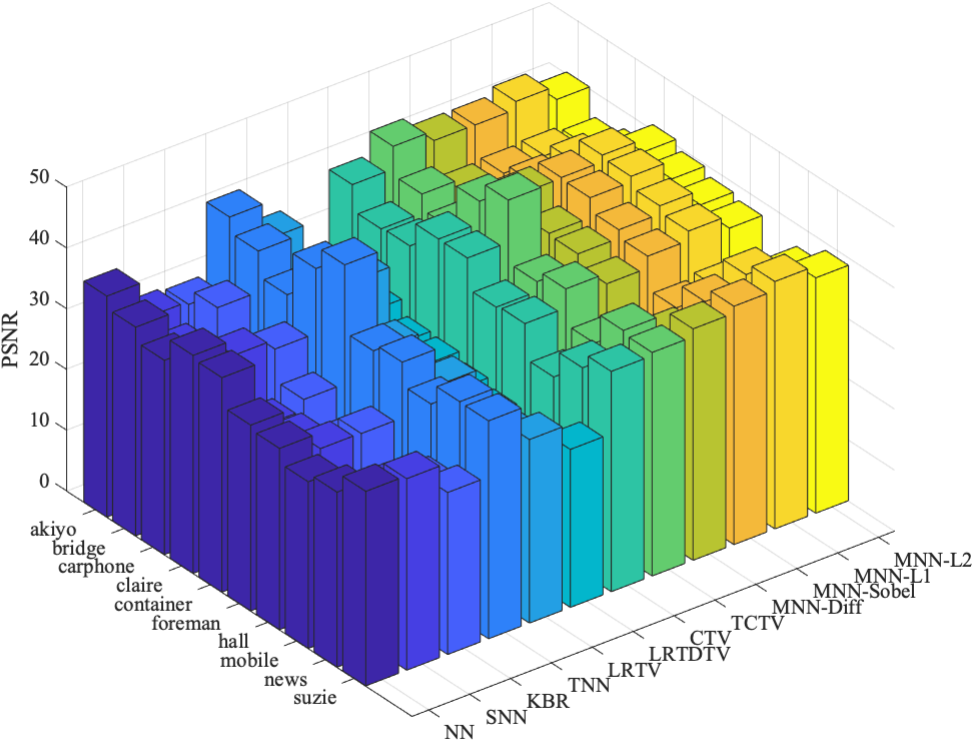}&
\includegraphics[width=38mm, height = 30mm]{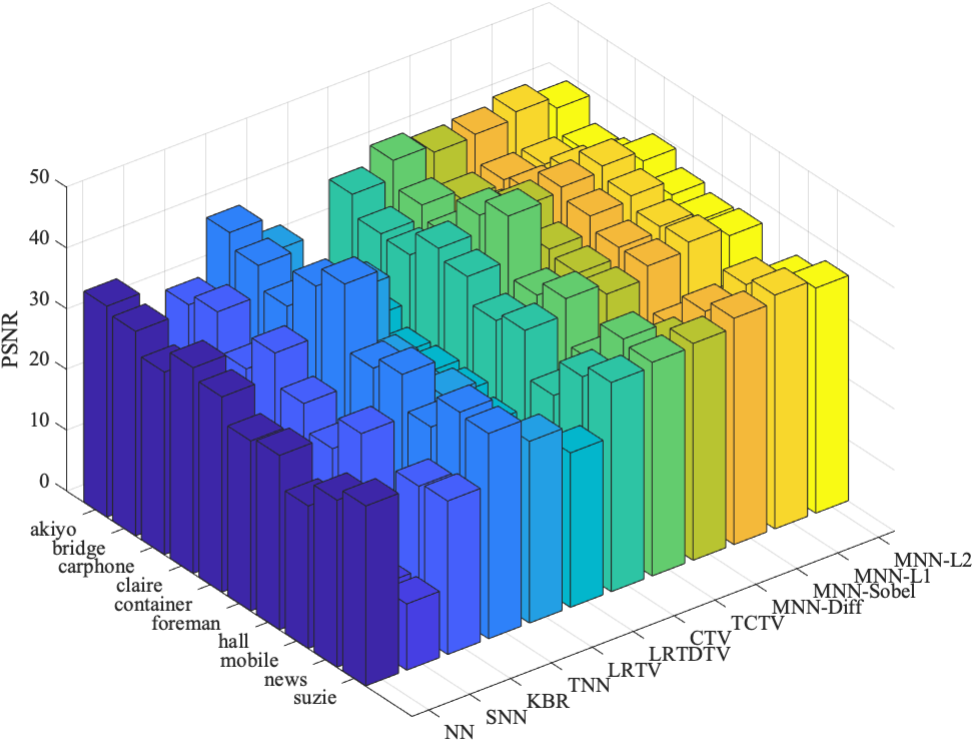}&
\includegraphics[width=38mm, height = 30mm]{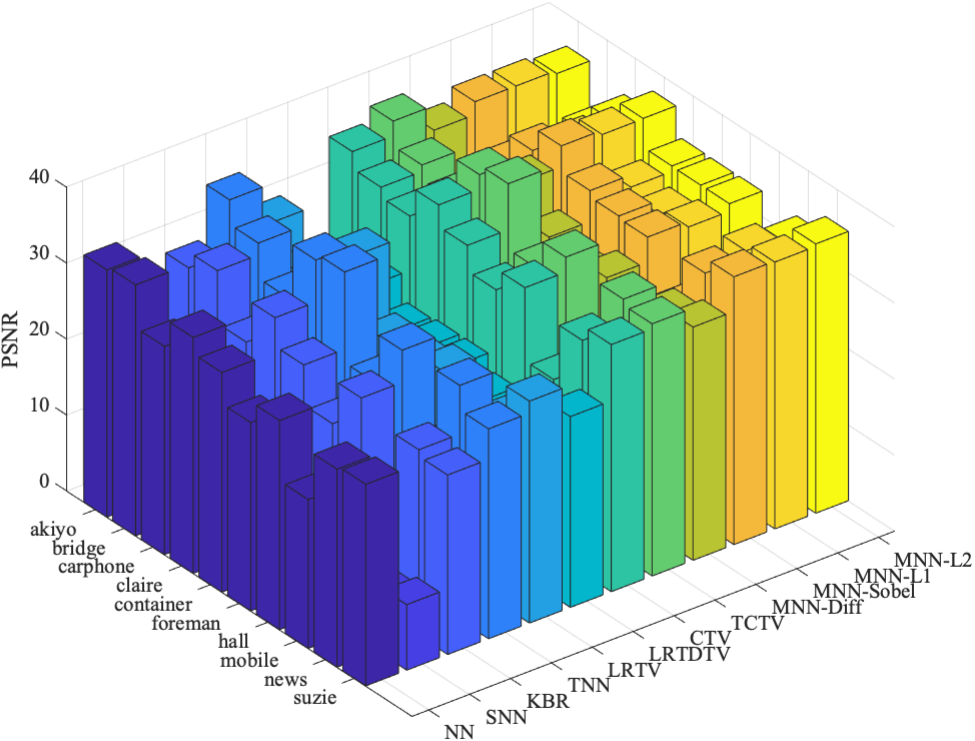}&
\includegraphics[width=38mm, height = 30mm]{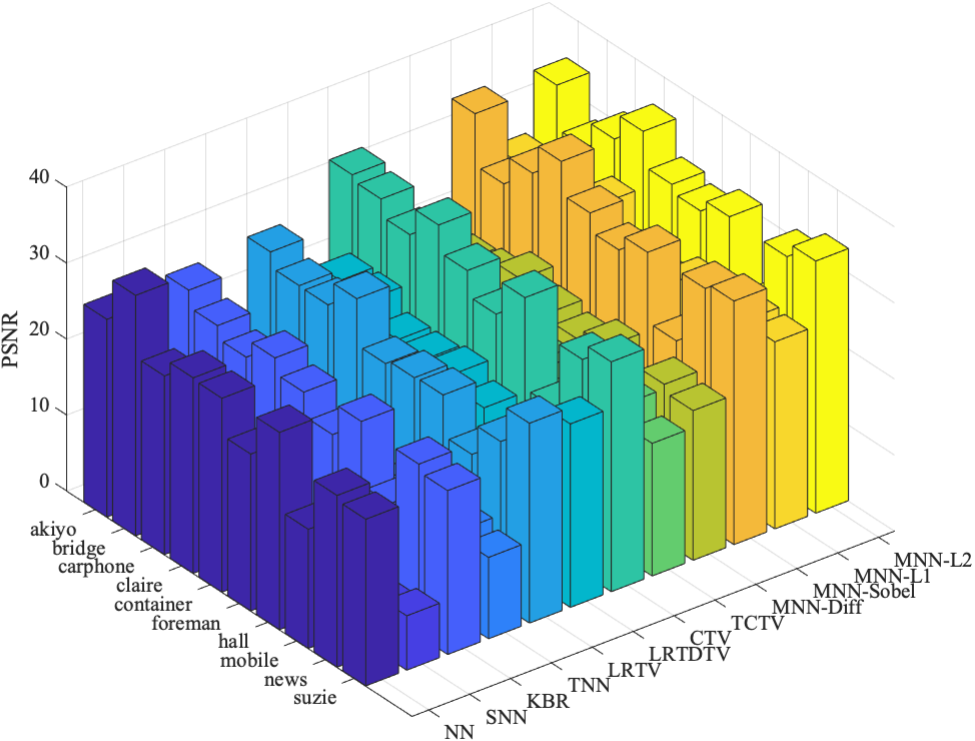}\\
\scriptsize{RPCA: s=0.1} & \scriptsize{RPCA: s=0.3} & \scriptsize{RPCA: s=0.5} & \scriptsize{RPCA: s=0.7} \\
\end{tabular}
\vspace{-2mm}
\caption{Performance comparison in terms of PSNR of recovered color videos obtained by all competing method under RPCA tasks.}\label{rpca_rgb}
\vspace{-0.2cm}
\end{figure*} 

\begin{figure*}[!h]
\renewcommand{\arraystretch}{0.5}
\setlength\tabcolsep{0.5pt}
\centering
\begin{tabular}{c c c c}
\centering
\includegraphics[width=38mm, height = 30mm]{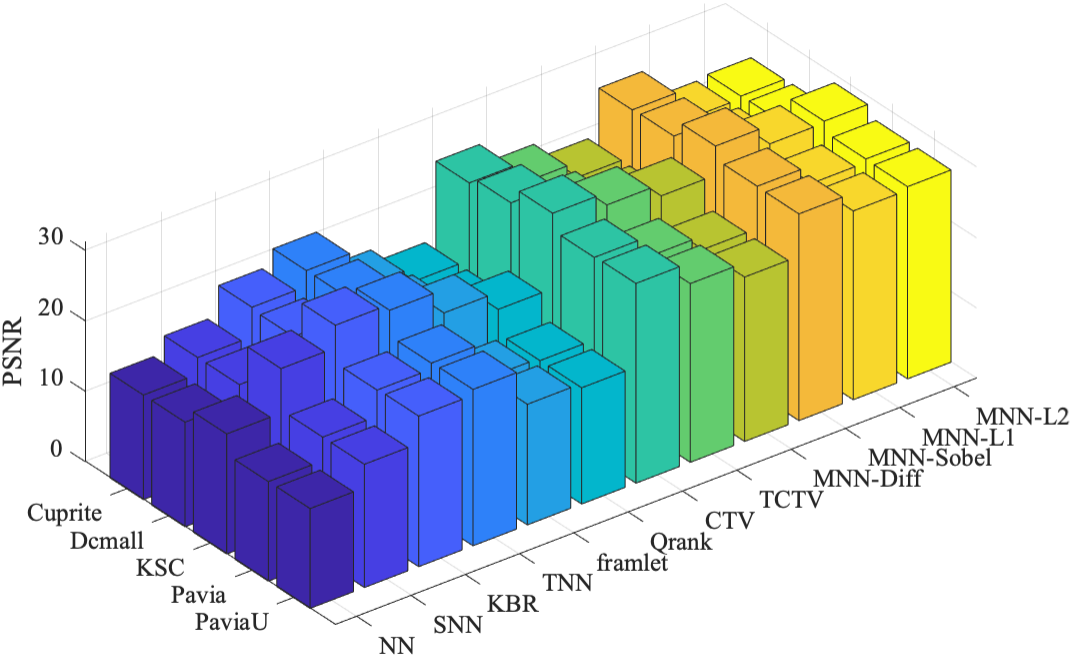}&
\includegraphics[width=38mm, height = 30mm]{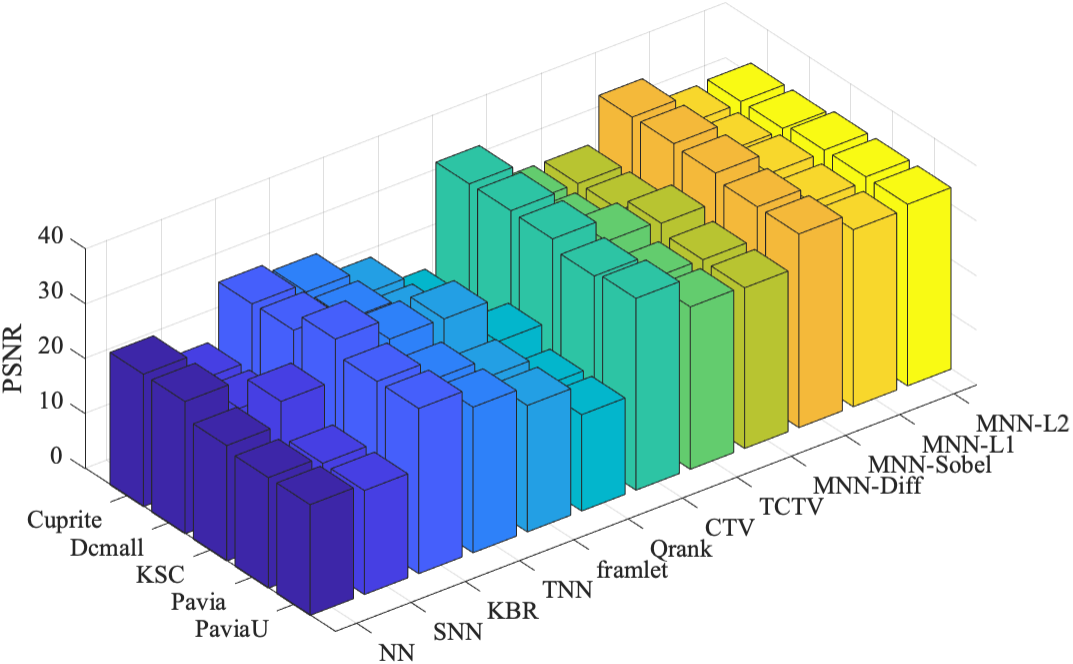}&
\includegraphics[width=38mm, height = 30mm]{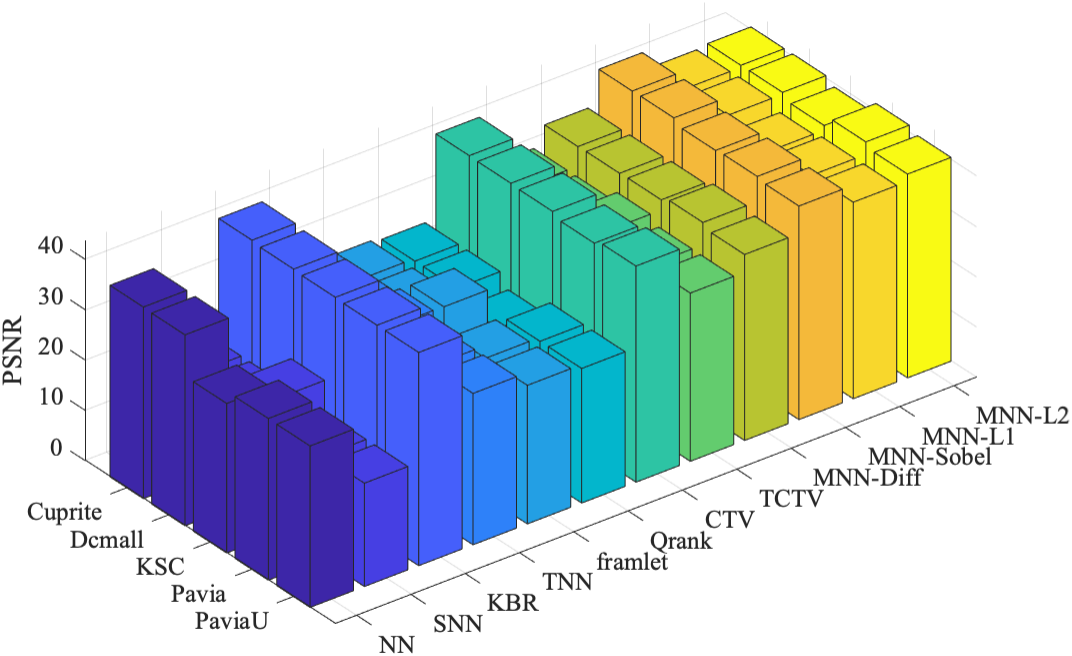}&
\includegraphics[width=38mm, height = 30mm]{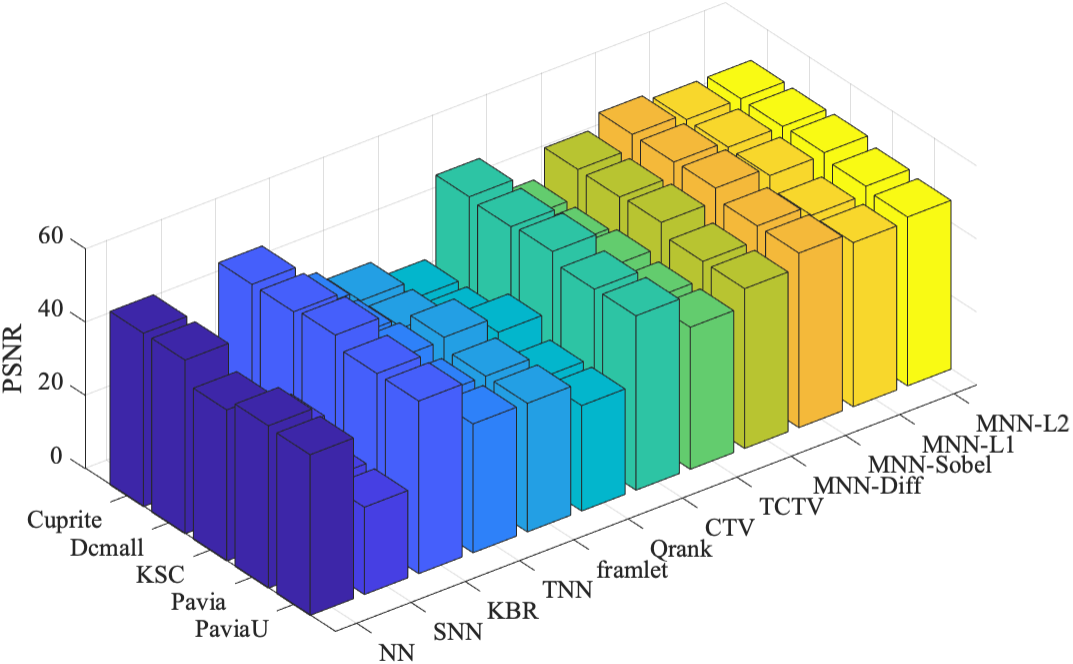}\\
\scriptsize{MC: sr=0.02} & \scriptsize{MC: sr=0.05} & \scriptsize{MC: sr=0.1} & \scriptsize{MC: sr=0.2} \\
\end{tabular}
\vspace{-2mm}
\caption{Performance comparison in terms of PSNR of recovered hyperspectral images obtained by all competing methods under MC tasks.}\label{mc_hsi}
\vspace{-0.2cm}
\end{figure*}

\begin{figure*}[!h]
\renewcommand{\arraystretch}{0.5}
\setlength\tabcolsep{0.5pt}
\centering
\begin{tabular}{c c c c}
\centering
\includegraphics[width=38mm, height = 30mm]{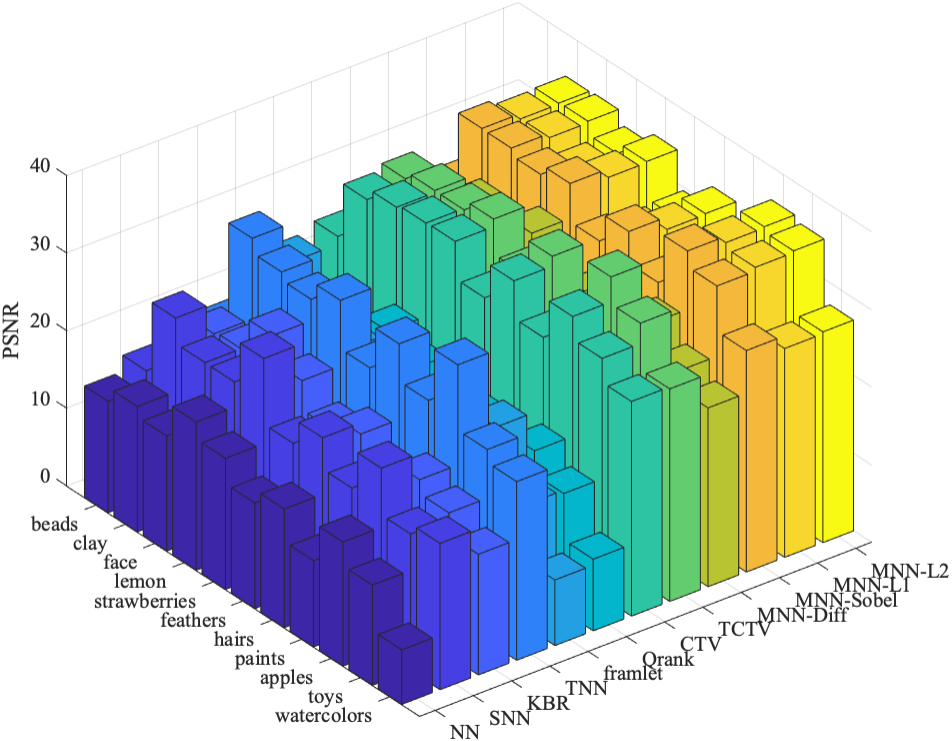}&
\includegraphics[width=38mm, height = 30mm]{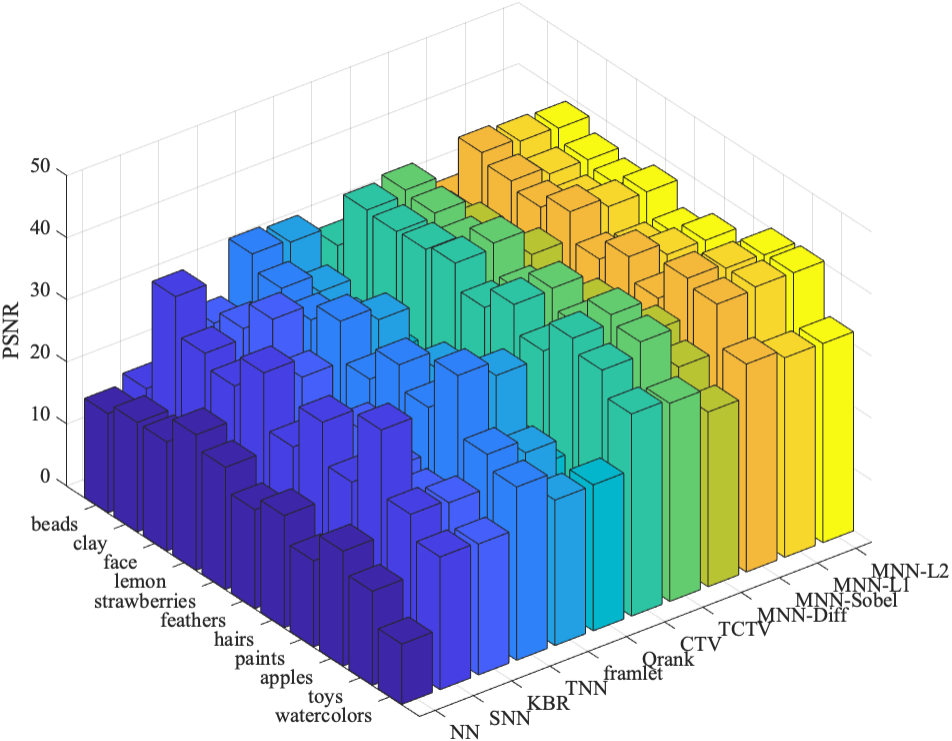}&
\includegraphics[width=38mm, height = 30mm]{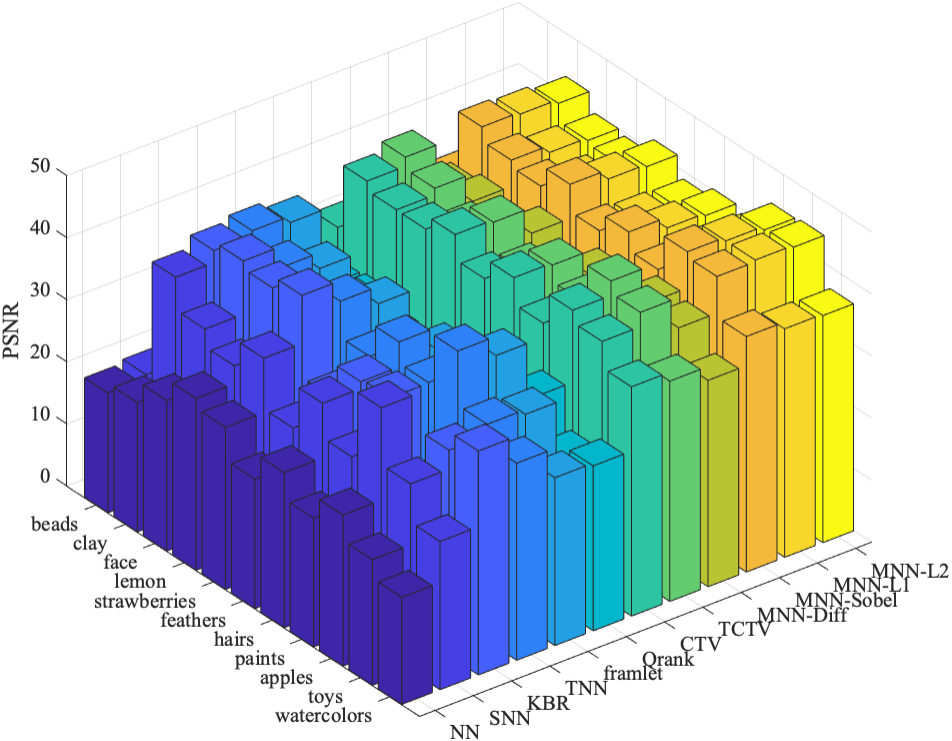}&
\includegraphics[width=38mm, height = 30mm]{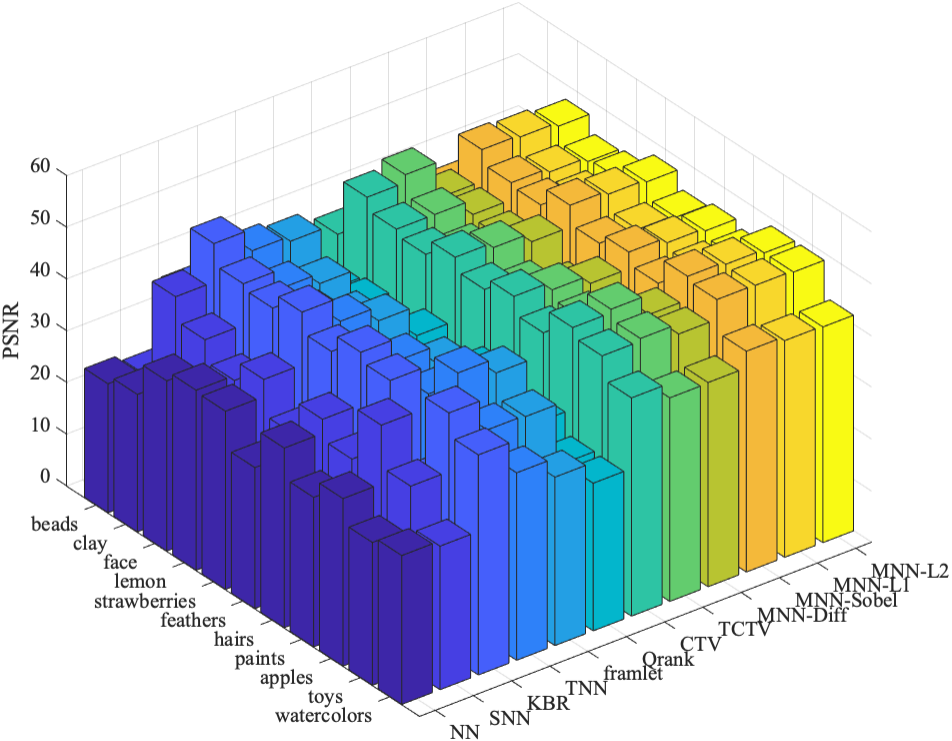}\\
\scriptsize{MC: sr=0.02} & \scriptsize{MC: sr=0.05} & \scriptsize{MC: sr=0.1} & \scriptsize{MC: sr=0.2} \\
\end{tabular}
\vspace{-2mm}
\caption{Performance comparison in terms of PSNR of recovered multispectral images obtained by all competing method under MC tasks.}\label{mc_msi}
\vspace{-0.2cm}
\end{figure*}

\begin{figure*}[!h]
\renewcommand{\arraystretch}{0.5}
\setlength\tabcolsep{0.5pt}
\centering
\begin{tabular}{c c c c}
\centering
\includegraphics[width=38mm, height = 30mm]{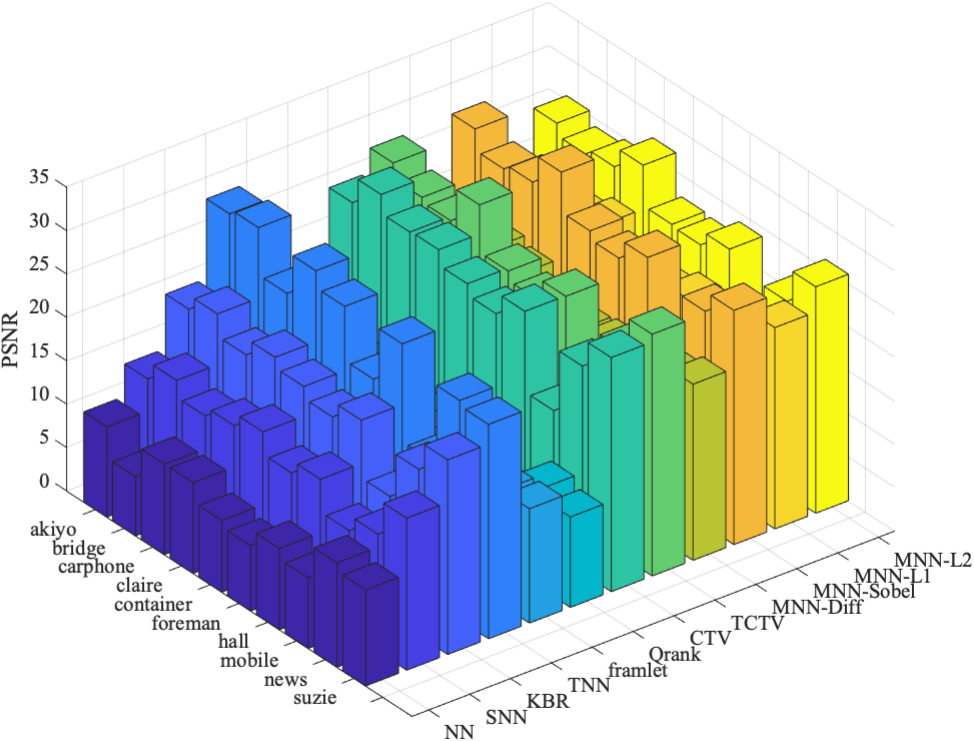}&
\includegraphics[width=38mm, height = 30mm]{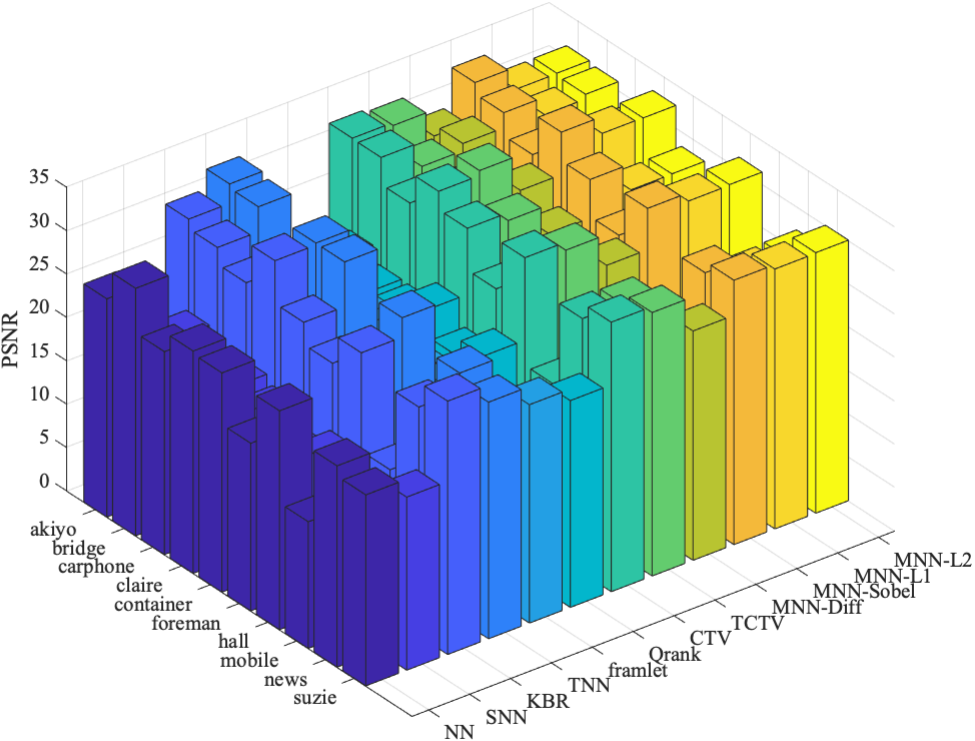}&
\includegraphics[width=38mm, height = 30mm]{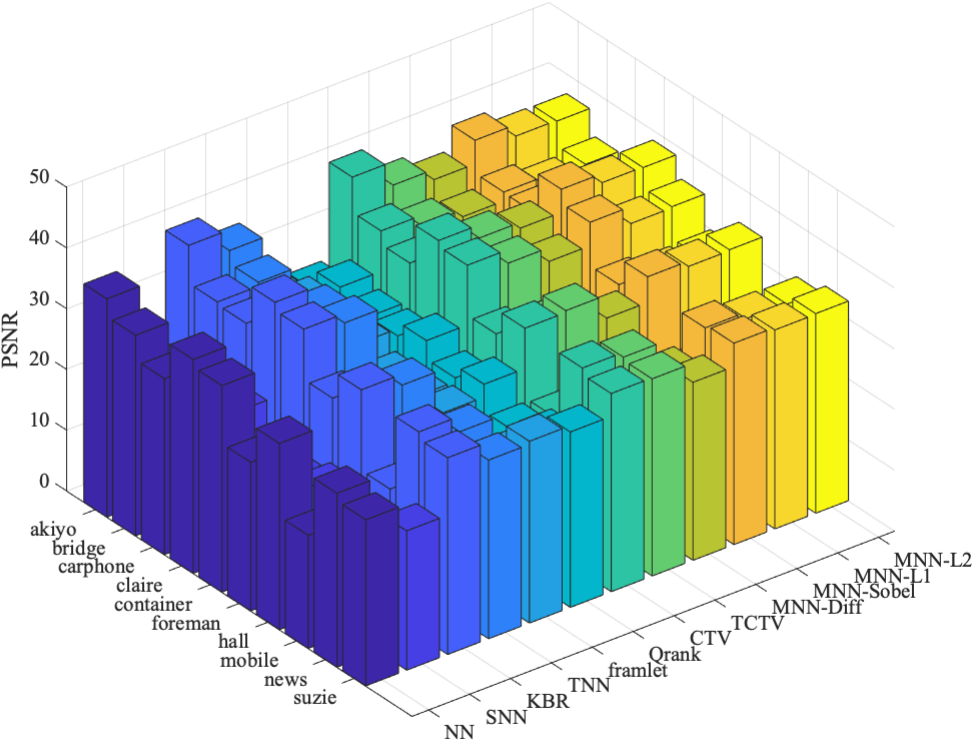}&
\includegraphics[width=38mm, height = 30mm]{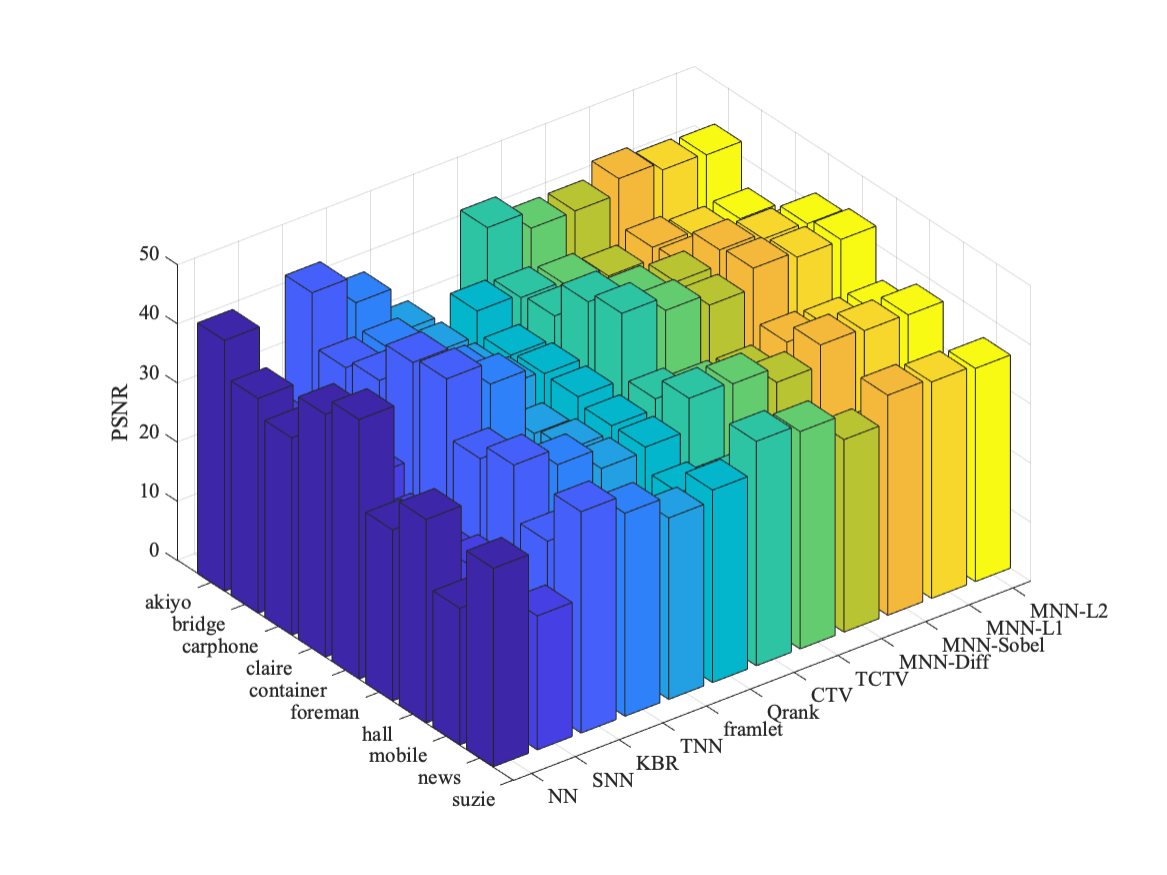}\\
\scriptsize{MC: sr=0.02} & \scriptsize{MC: sr=0.05} & \scriptsize{MC: sr=0.1} & \scriptsize{MC: sr=0.2} \\
\end{tabular}
\vspace{-2mm}
\caption{Performance comparison in terms of PSNR of recovered color videos obtained by all competing method under MC tasks.}\label{mc_rgb}
\vspace{-0.2cm}
\end{figure*}

\begin{figure*}[!h]
\renewcommand{\arraystretch}{0.5}
\setlength\tabcolsep{0.5pt}
\centering
\begin{tabular}{c c c c}
\centering
\includegraphics[width=38mm, height = 30mm]{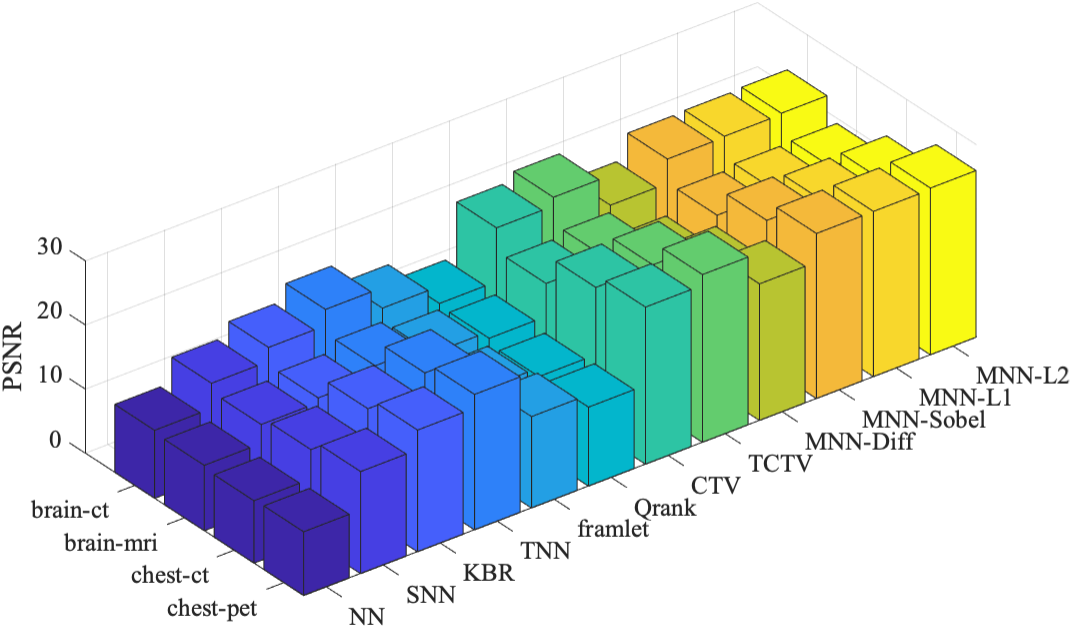}&
\includegraphics[width=38mm, height = 30mm]{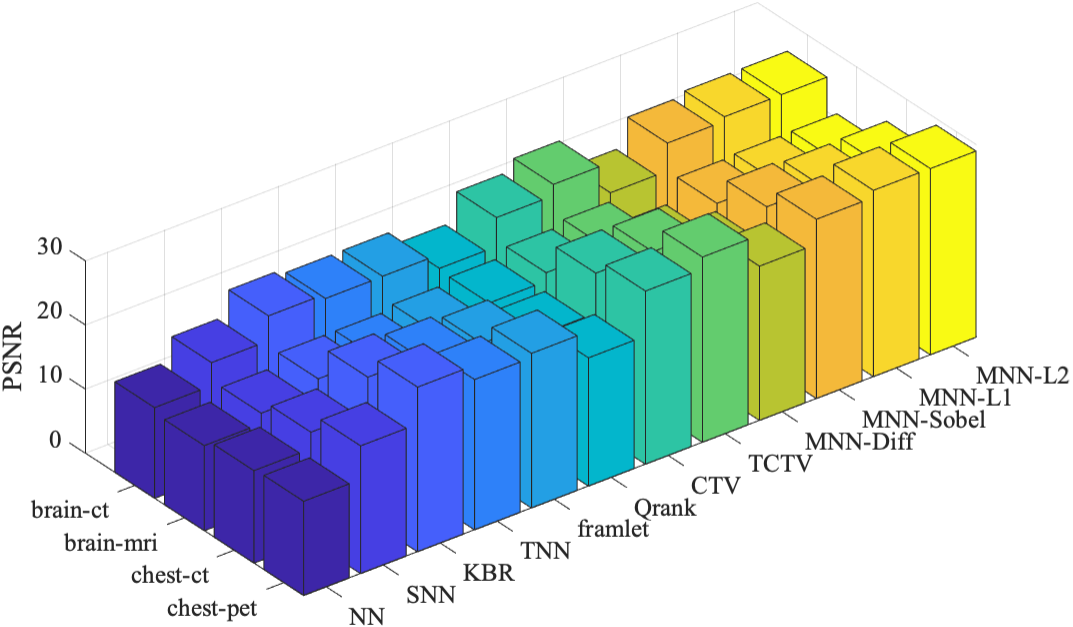}&
\includegraphics[width=38mm, height = 30mm]{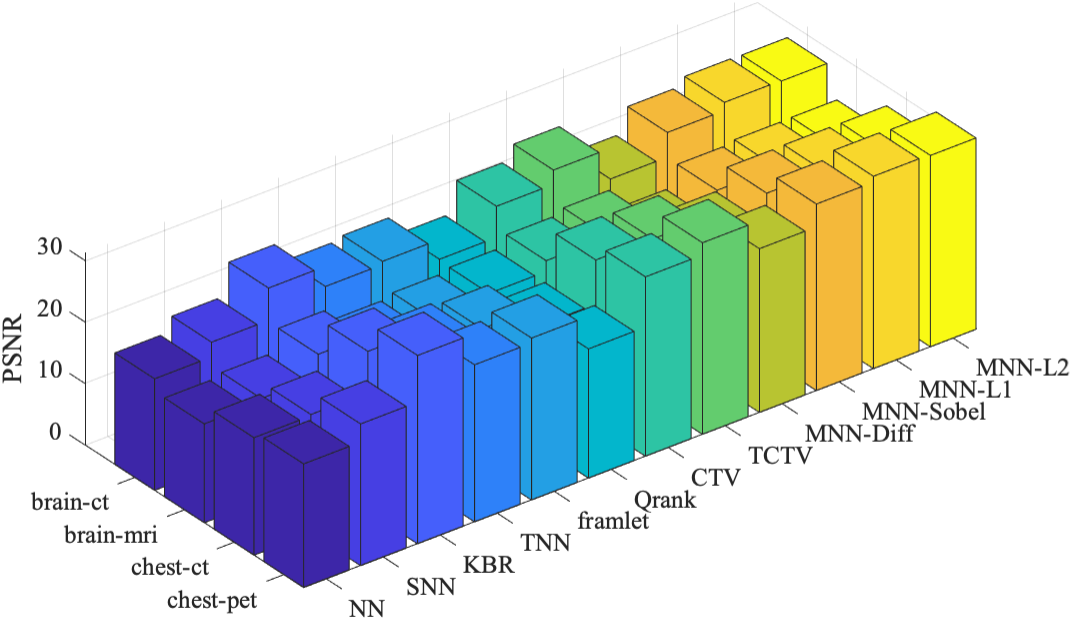}&
\includegraphics[width=38mm, height = 30mm]{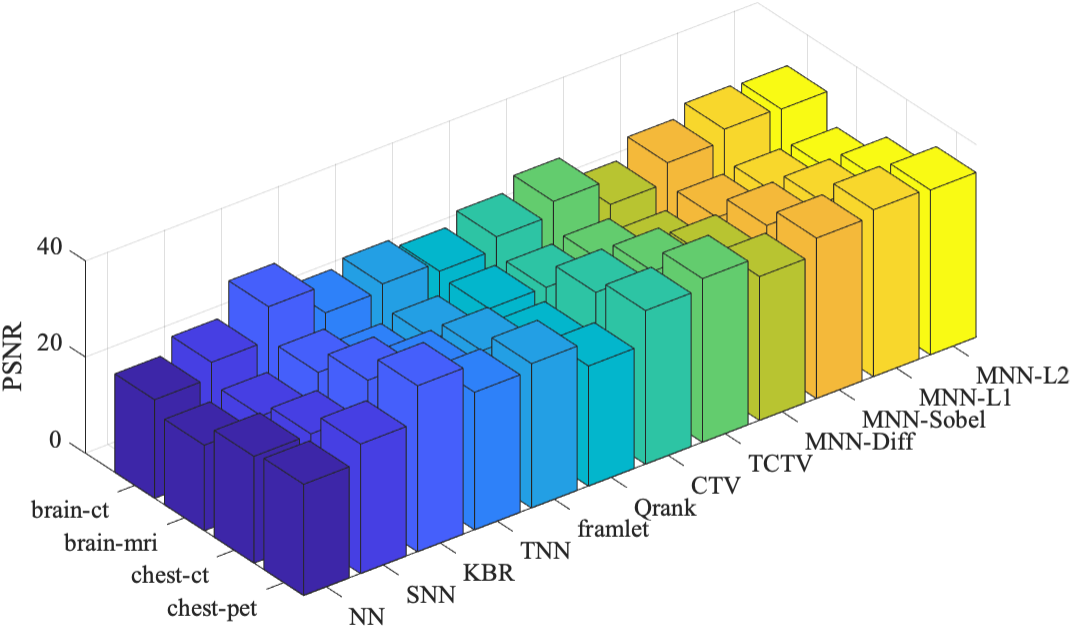}\\
\scriptsize{MC: sr=0.02} & \scriptsize{MC: sr=0.05} & \scriptsize{MC: sr=0.1} & \scriptsize{MC: sr=0.2} \\
\end{tabular}
\vspace{-2mm}
\caption{Performance comparison in terms of PSNR of recovered MRI and CT data obtained by all competing method under MC tasks.}\label{mc_ct}
\vspace{-0.2cm}
\end{figure*}

Figure \ref{rpca_hsi} to Figure \ref{rpca_rgb} sequentially display histograms of restoration metric distributions for all methods on multiple datasets for denoising tasks under three categories: HSI, MSI, and Color Video. From these histograms, it can be observed that the PSNR values of our four MNN-based models consistently remain high, especially MNN-Sobel and MNN-L2. Additionally, we also notice that for a few datasets and noise levels, the PSNR values of TCTV are higher than those of the four MNN-based models and significantly higher than CTV. This indicates the necessity of employing tensor tools for modeling tensor data. Therefore, in the future, we will consider further extending the MNN framework to tensor structures.

Figures \ref{mc_hsi} to \ref{mc_ct} sequentially display histograms of restoration metric distributions for all methods on multiple datasets for completing tasks under four categories: HSI, MSI, Color Video, and CT. Similar to the denoising results, the four MNN-based models perform excellently on the vast majority of datasets and missing rates, consistent with the mean PSNR results presented in Table 2 of the main text. Additionally, we also observe that when the sampling rate is relatively high, such as when the sampling rate (sr) is set to 0.2, KBR based on tensor decomposition exhibits outstanding PSNR performance, even surpassing MNN-Sobel, MNN-L1, and MNN-L2 on a few datasets. This indicates that tensor modeling still holds certain advantages.

}

\bibliographystyle{IEEEtran}
\bibliography{mybibfile}

\end{document}